\documentclass{article}

\usepackage{caption}
\usepackage{subcaption}
\usepackage[numbers]{natbib}
\usepackage[preprint]{neurips_2025}

\usepackage[utf8]{inputenc} % allow utf-8 input
\usepackage[T1]{fontenc}    % use 8-bit T1 fonts
\usepackage{hyperref}       % hyperlinks
\usepackage{graphicx}
\usepackage{subcaption}
\usepackage{mathtools}
\usepackage{url}            % simple URL typesetting
\usepackage{booktabs}       % professional-quality tables
\usepackage{amsfonts}       % blackboard math symbols
\usepackage{amsthm}
\usepackage{nicefrac}       % compact symbols for 1/2, etc.
\usepackage{microtype}      % microtypography
\usepackage{xcolor}         % colors
\usepackage{cite}
\usepackage[normalem]{ulem}
\usepackage{amsmath}
\usepackage{booktabs}
\title{Randomness Helps Rigor: A Probabilistic Learning Rate Scheduler Bridging Theory and Deep Learning Practice}
\newtheorem{theorem}{Theorem}
\newtheorem{remark}{Remark}
\newtheorem{definition}{Definition}
\newtheorem{assumption}{Assumptions}
\newtheorem{lemma}{Lemma} 
\newcommand\norm[1]{\left\lVert#1\right\rVert}

\newcommand\abs[1]{\left\lvert #1\right\rvert}
\newcommand{\BlackBox}{\rule{1.5ex}{1.5ex}}

\newenvironment{proofsketch}{\par\noindent{\bf Proof Sketch\ }}{\hfill\BlackBox\\[2mm]}

\author{%
  Dahlia Devapriya\\
  Department of Electrical Engineering\\
  IIT Madras\\
  Chennai, India.\\
  \texttt{ee22d003@smail.iitm.ac.in} \\
   \And
   Thulasi Tholeti\thanks{Authors contributed equally} \\
  Institute for Experiential AI\\
  Northeastern University\\
  Boston, USA.\\
  \texttt{t.tholeti@northeastern.edu} \\
   \And
   Janani Suresh$^{*}$\\
  Department of Electrical Engineering\\
  IIT Madras\\
  Chennai, India.\\
  \texttt{ee22s079@smail.iitm.ac.in} \\
   \And
   Sheetal Kalyani\\
  Department of Electrical Engineering\\
  IIT Madras\\
  Chennai, India.\\
  \texttt{skalyani@ee.iitm.ac.in} 
}

\begin{document}
% \title{Randomness Helps Rigor: A Probabilistic Learning Rate Scheduler Bridging Theory and Deep Learning Practice}
% \author{Dahlia Devapriya, Thulasi Tholeti$^*$, Janani Suresh$^*$, Sheetal Kalyani
% 		% \hspace{-0.5 cm}Department of Electrical Engineering,\\
% 		% \hspace{-0.8 cm} Indian Institute of Technology, Madras, \\
% 		% \hspace{-1cm} Chennai, India 600036.\\
% 		% \hspace{-1cm} \{ee17d022@smail,skalyani@ee\}.iitm.ac.in\\
% 	}
%  \def\thefootnote{*}\footnotetext{The authors contributed equally to this work. Dahlia Devapriya, Janani Suresh and Sheetal Kalyani are with the Department of Electrical Engineering, Indian Institute of Technology, Madras \big(\{ee22d003@smail, ee22s079@smail, skalyani@ee\}.iitm.ac.in\big). Thulasi Tholeti is with the Institute for Experiential AI, Northeastern University, Boston, USA (t.tholeti@northeastern.edu).}
\maketitle

\begin{abstract}
  Learning rate schedulers have shown great success in speeding up the convergence of learning algorithms in practice. However, their convergence to a minimum has not been proven theoretically. This difficulty mainly arises from the fact that, while traditional convergence analysis prescribes to monotonically decreasing (or constant) learning rates, schedulers opt for rates that often increase and decrease through the training epochs. In this work, we aim to bridge the gap by proposing a probabilistic learning rate scheduler (PLRS) that does not conform to the monotonically decreasing condition, with provable convergence guarantees. To cement the relevance and utility of our work in modern day applications, we show experimental results on deep neural network architectures such as ResNet, WRN, VGG, and DenseNet on  CIFAR-10, CIFAR-100, and Tiny ImageNet datasets. We show that PLRS performs as well as or better than existing state-of-the-art learning rate schedulers in terms of convergence as well as accuracy. For example, while training ResNet-110 on the CIFAR-100 dataset, we outperform the state-of-the-art knee scheduler by $1.56\%$ in terms of classification accuracy. Furthermore, on the Tiny ImageNet dataset using ResNet-50 architecture, we show a significantly more stable convergence than the cosine scheduler and a better classification accuracy than the existing schedulers.
  % To cement the relevance and utility of our work in modern day applications, we show experimental results where the proposed PLRS performs as well as (or better than) other state-of-the-art learning rate schedulers across a variety of datasets and \comm{deep neural network }architectures. 
\end{abstract}

\section{Introduction}
Over the last two decades, there has been an increased interest in analyzing the convergence of gradient descent-based algorithms. This can be majorly attributed to their extensive use in the training of neural networks and their numerous derivatives. Stochastic Gradient Descent (SGD) and their adaptive variants such as Adagrad  \citep{adagrad}, Adadelta  \citep{adadelta}, and Adam \citep{adam} have been the choice of optimization algorithms for most machine learning practitioners, primarily due to their ability to process enormous amounts of data in batches. Even with the introduction of adaptive optimization techniques that use a default learning rate, the use of stochastic gradient descent with a tuned learning rate was quite prevalent, mainly due to its generalization properties \citep{zhou2020towards}. However, tuning the learning rate of the network can be computationally intensive and time consuming.

Various methods to efficiently choose the learning rate without excessive tuning have been explored. One of the initial successes in this domain is the random search method \citep{bergstra2012random}; here, a learning rate is randomly selected from a specified interval across multiple trials, and the best performing learning rate is ultimately chosen. Following this, more advanced methods such as Sequential Model-Based Optimization (SMBO) \citep{bergstra2013making} for the choice of learning rate became prevalent in practice. SMBO represents a significant advancement over random search by tracking the effectiveness of learning rates from previous trials and using this information to build a model that suggests the next optimal learning rate. A tuning method for shallow neural networks based on theoretical computation of the Hessian Lipschitz constant was proposed by Tholeti et al. \citep{tholeti2020tune}.

     Several works on training deep neural networks prescribed the use of a decaying Learning Rate (LR)\footnote{We abbreviate learning rate only in the context of learning rate scheduler as LR scheduler.} scheduler \citep{resnet,robust, googlenet}. Recently, much attention has been paid to cyclically varying learning rates  \citep{cyclic}. By varying learning rates in a triangular schedule within a predetermined range of values, the authors hypothesize that the optimal learning rate lies within the chosen range, and the periodic high learning rate helps escape saddle points. 
    Although no theoretical backing has been provided, it was shown to be a valid hypothesis owing to the presence of many saddle points in a typical high dimensional learning task \citep{dauphin2014}. Many variants of the cyclic LR scheduler have henceforth been used in various machine learning tasks  \citep{ulm,dhillon2019baseline,andriushchenko2020understanding}. A cosine-based cyclic LR scheduler proposed by Loshchilov et al. \citep{sgdr} has also found several applications, including Transformers \citep{zamir2022restormer,caron2021emerging}. Following the success of the cyclic LR schedulers, a one-cycle LR scheduler proposed by Smith et al. \citep{smith2019super} has been observed to provide faster convergence empirically; this was attributed to the injection of `good noise' by higher learning rates which helps in convergence. Although empirical validation and intuitions were provided to support the working of these LR schedulers, a theoretical convergence guarantee has not been provided to the best of our knowledge.
    
There is extensive research on the convergence behavior of perturbed SGD methods, where noise is added to the gradient during updates.
% Another line of research that is aligned with our work in terms of the analysis framework is the convergence analysis of perturbed gradient descent and stochastic gradient descent methods. 
In Jin et al. \citep{jin2017escape}, the vanilla gradient descent is perturbed by samples from a ball whose radius is fixed using the optimization function-specific constants. They show escape from a saddle point by characterizing the distribution around a perturbed iterate as uniformly distributed over a perturbation ball along which the region corresponding to being stuck at a saddle point is shown to be very small. In Ge at al. \citep{ge2015escaping}, the saddle point escape for a perturbed stochastic gradient descent is proved using the second-order Taylor approximation of the optimization function, where the perturbation is applied from a unit ball to the stochastic gradient descent update. Following Ge at al. \citep{ge2015escaping}, several works prove the convergence of noisy stochastic gradient descent in the additive noise setting \citep{zhang2017hitting,jin2021,arjevani2023lower,yiming2025efficiently}. In contrast to the above works which operate in the \textit{additive} noise setting, our proposed LR scheduler results in \textit{multiplicative} noise. Analyzing the convergence behavior under the new multiplicative noise setting is fairly challenging and results in a non-trivial addition to the literature.
% which considerably alters the proof mechanism.} 

\subsection{Motivation}
Traditional convergence analysis of gradient descent algorithms and its variants requires the use of a constant or a decaying learning rate \citep{nesterov}. However, with the introduction of LR schedulers, the learning rates are no longer monotonically decreasing. Rather, their values heavily fluctuate, with the occasional use of very large learning rates. Although there are ample justifications provided for the success of such methods, there are no theoretical results which prove that stochastic gradient descent algorithms with fluctuating learning rates converge to a local minimum in a non-convex setting. With the increase of emphasis on trustworthy artificial intelligence, we believe that it is important to no longer treat optimization algorithms as black-box models, and instead provide provable convergence guarantees while deviating from the proven classical implementation of the descent algorithms. In this work, we aim to bridge the gap by providing rigorous mathematical proof for the convergence of our proposed probabilistic LR scheduler with SGD.
% , which does not adhere to the monotonically decreasing condition of the traditional convergence analysis.
\renewcommand*{\thefootnote}{\arabic{footnote}}
\subsection{Our contributions}\label{sec:contri}
\begin{enumerate}
    \item We propose a new Probabilistic Learning Rate Scheduler (PLRS) where we model the learning rate as an instance of a random noise distribution. 
    \item We provide convergence proofs to show that SGD with our proposed PLRS converges to a local minimum
    % \footnote{Note that for most deep learning applications, local minima are not significantly worse than global minima \citep{goodfellow2014qualitatively}.} 
    in Section \ref{sec:proof}. To the best of our knowledge, we are the first to theoretically prove convergence of SGD with a LR scheduler that does not conform to constant or monotonically decreasing rates. We show how our LR scheduler, in combination with inherent SGD noise, speeds up convergence by escaping saddle points.
    \item Our proposed probabilistic LR scheduler, while provably convergent, can be seamlessly ported into practice without the knowledge of theoretical constants (like gradient and Hessian-Lipschitz constants). We illustrate the efficacy of the PLRS through extensive experimental validation, where we compare the accuracies with state-of-the-art schedulers in Section \ref{sec:results}. We show that the proposed method outperforms popular schedulers such as cosine annealing \citep{sgdr}, one-cycle \citep{smith2019super}, knee \citep{knee} and the multi-step scheduler when used with ResNet-110 on CIFAR-100, VGG-16 on CIFAR-10 and ResNet-50 on Tiny ImageNet, while displaying comparable performances on other architectures like DenseNet-40-12 and WRN-28-10 when trained on CIFAR-10 and CIFAR-100 datasets respectively. We also observe lesser spikes in the training loss across epochs which leads to a faster and more stable convergence.
    % We provide our base code with all the hyperparameters for reproducibility in the supplemental material.
    We provide our base code with all the hyper-parameters for reproducibility\footnote{https://github.com/janani-suresh-97/uniform\_lr\_scheduler}.
\end{enumerate}
\section{Probabilistic learning rate scheduler}
Let $f: \mathbb{R}^d \rightarrow \mathbb{R}$ be the function to be minimized. The unconstrained optimization, $
    \min_{\textbf{x} \in \mathbb{R}^d} f(\bf x)$, can be solved iteratively using stochastic gradient descent whose update equation at time step $t$ is given by
\begin{equation}
   \mathbf{x}_{t+1} = \mathbf{x}_t -\eta_{t+1} g(\mathbf{x}_t).
\end{equation}
Here, $\eta_{t+1} \in \mathbb{R}$ is the learning rate and $g(\mathbf{x}_t)$ is the stochastic gradient of $f(\bf x)$ at time $t$. 
In this work, we propose a new LR scheduler, in which the learning rate $\eta_{t+1}$ is sampled from a uniform random variable, 
\begin{equation}
    \eta_{t+1} \sim \mathcal{U}[L_{min},L_{max}], \quad 0<L_{min}<L_{max}<1.
\end{equation}
Note that contrary to existing LR schedulers, which are deterministic functions, we propose that the learning rate at each time instant be a realization of a uniformly distributed random variable. Although the learning rate in our method is not scheduled, but is rather chosen as a random sample at every time step, we call our proposed method Probabilistic LR scheduler to keep in tune with the body of literature on LR schedulers. 
In order to represent our method in the conventional form of the stochastic gradient descent update, we split the learning rate $\eta_{t+1}$ into a constant learning rate $\eta_c$ and a random component, as $\eta_{t+1} = \eta_c + u_{t+1}$, where $u_{t+1}\sim \mathcal{U}[L_{min}-\eta_c,L_{max}-\eta_c]$. The stochastic gradient descent update using the proposed PLRS (referred to as SGD-PLRS) takes the form
\begin{equation}\label{update}
        \mathbf{x}_{t+1} = \mathbf{x}_{t} -(\eta_c +u_{t+1})g(\mathbf{x}_{t}) =\mathbf{x}_{t}-\eta_c\nabla f(\mathbf{x}_{t})-\mathbf{w}_{t},
    \end{equation}
    where we define $\mathbf{w}_{t}$ as
    \begin{equation}\label{noise_term}
        \mathbf{w}_{t}=\eta_c g(\mathbf{x}_{t}) - \eta_c\nabla f(\mathbf{x}_{t}) +  u_{t+1}g(\mathbf{x}_{t}).
    \end{equation}
    Here, $\nabla f(\mathbf{x}_t)$ refers to the true gradient, i.e., $\nabla f(\mathbf{x}_t) = \mathbb{E} [g(\mathbf{x}_{t})]$. Note that in \eqref{update}, the term $\mathbf{x}_{t}-\eta_c\nabla f(\mathbf{x}_{t})$ resembles the vanilla gradient descent update and $\mathbf{w}_{t}$ encompasses the noise in the update; the noise is inclusive of both the randomness due to the stochastic gradient as well as the randomness from the proposed LR scheduler. We set $\eta_c = \frac{L_{min}+L_{max}}{2}$ so that the noise $\mathbf{w}_{t}$ is zero mean, which we prove later in Lemma \ref{lemma:zeromean}.
\begin{remark}
   Note that a periodic LR scheduler such as triangular, or cosine annealing based scheduler can be considered as a single instance of our proposed PLRS. The range of values assigned to the learning rate $\eta_{t+1}$ is pre-determined in both cases. In fact, for any LR scheduler, the basic mechanism is to vary the learning rate between a low and a high value - the high learning rates help escape the saddle point by perturbing the iterate, whereas the low values help in convergence. This pattern of switching between high and low values can be achieved through both stochastic and deterministic mechanisms. While the current literature explores the deterministic route (without providing analysis), we propose and explore the stochastic variant here and also provide a detailed analysis.
\end{remark}

\section{Preliminaries and definitions}
We denote the Hessian of a function $f:\mathbb{R}^d\rightarrow\mathbb{R}$ at $\mathbf{x}\in\mathbb{R}^d$ as $\textbf{H}(\textbf{x})\coloneq \nabla^2f(\mathbf{x})$ and the minimum eigenvalue of the Hessian as $\lambda_{min} (\textbf{H}(\textbf{x}))\coloneq \lambda_{min}(\nabla^2f(\mathbf{x}))$ respectively.
\begin{definition}
    A function $f: \mathbb{R}^d \rightarrow \mathbb{R}$ is said to be $\beta$-smooth (also referred to as $\beta$-gradient Lipschitz) if, $\exists$ $\beta \geq 0$ such that, 
    \begin{equation}
     \norm{\nabla f(\mathbf{x}) - \nabla f(\mathbf{y})} \leq \beta \norm{\mathbf{x}-\mathbf{y}}, \quad\forall \mathbf{x},\mathbf{y} \in \mathbb{R}^d.
 \end{equation}
\end{definition}
\begin{definition}
     A function $f: \mathbb{R}^d \rightarrow \mathbb{R}$ is said to be $\rho$-Hessian Lipschitz if, $\exists$ $\rho \geq 0$ such that, 
     \begin{equation}
    \norm{\textbf{H}(\textbf{x}) - \textbf{H}(\textbf{y})} \leq \rho \norm{\mathbf{x}-\mathbf{y}},\quad \forall \mathbf{x},\mathbf{y} \in \mathbb{R}^d.
\end{equation}
 \end{definition}
Informally, a function is said to be gradient/Hessian Lipschitz, if the rate of change of the gradient/Hessian with respect to its input is bounded by a constant, i.e., the gradient/Hessian will not change rapidly.
 We now proceed to define approximate first and second-order stationary points of a given function $f$. 
 \begin{definition} \label{defn:fosp}
    For a function $f: \mathbb{R}^d \rightarrow \mathbb{R}$ that is differentiable, we say $\mathbf{x} \in \mathbb{R}^d$ is a $\nu$- first-order stationary point ($\nu$-FOSP), if for a small positive value of $\nu$, $\norm{\nabla f(\mathbf{x})} \leq \nu.$
% \begin{equation}
%     \norm{\nabla f(\mathbf{x})} \leq \epsilon.
% \end{equation}
\end{definition}
Before we define an $\epsilon$-second order stationary point, we define a saddle point.
\begin{definition}
    For a $\rho$-Hessian Lipschitz function $f: \mathbb{R}^d \rightarrow \mathbb{R}$ that is twice differentiable, we say $\textbf{x}\in\mathbb{R}^d$ is a saddle point if, 
    % $\norm{\nabla f(\mathbf{x}) } \leq \epsilon \quad \text{and}  \quad \lambda_{min} (\textbf{H}(\textbf{x})) \leq -\sqrt{\rho \epsilon}.$
    \begin{equation*}
        \norm{\nabla f(\mathbf{x}) } \leq \nu \quad \text{and}  \quad \lambda_{min} (\textbf{H}(\textbf{x})) \leq -\gamma,
    \end{equation*}
   where $\nu,\gamma>0$ are arbitrary constants.
\end{definition}
For a convex function, it is sufficient if the algorithm is shown to converge to the $\nu$-FOSP as it would be the global minimum. However, in the case of a non-convex function, a point satisfying the condition for a $\nu$-FOSP may not necessarily be a local minimum, but could be a saddle point or a local maximum. Hence, the Hessian of the function is required to classify it as a second-order stationary point, as defined below. Note that, in our analysis, we prove convergence of SGD-PLRS to the approximate second-order stationary point.
\begin{definition} \label{defn:sosp}
    For a $\rho$-Hessian Lipschitz function $f: \mathbb{R}^d \rightarrow \mathbb{R}$ that is twice differentiable, we say $\mathbf{x} \in \mathbb{R}^d$ is a $\nu$-second-order stationary point ($\nu$-SOSP) if, 
\begin{equation}
 \norm{\nabla f(\mathbf{x}) } \leq \nu \quad \text{and}  \quad \lambda_{min} (\textbf{H}(\textbf{x})) \geq -\gamma,
\end{equation}
where $\nu,\gamma>0$ are arbitrary constants.
\end{definition}
\begin{definition}\label{defn:saddle}
    A function $f: \mathbb{R}^d \rightarrow \mathbb{R}$ is said to possess the strict saddle property at all $\textbf{x}\in\mathbb{R}^d$ if $\textbf{x}$ fulfills any one of the following conditions: (i) $\norm{\nabla f(\mathbf{x})}\geq \nu$, (ii) $\lambda_{min} (\textbf{H}(\textbf{x})) \leq -\gamma,$ (iii) $\mathbf{x}$ is close to a local minimum.
    % \begin{enumerate}
    %     \item $\norm{\nabla f(\mathbf{x})}\geq \epsilon.$
    %     \item $\lambda_{min} (\textbf{H}(\textbf{x})) \leq -\sqrt{\rho \epsilon}.$
    %     \item $\mathbf{x}$ is close to a local minimum.
    % \end{enumerate}
\end{definition}
The strict saddle property ensures that an iterate stuck at a saddle point has a direction of escape.
% Next, we define the $\alpha$-strongly convex property of a function.
\begin{definition}
    A function $f: \mathbb{R}^d \rightarrow \mathbb{R}$ is $\alpha-$strongly convex if $\lambda_{min}(\textbf{H}(\textbf{x}))\geq \alpha\quad \forall \mathbf{x}\in\mathbb{R}^d$.
\end{definition}

We now provide the formal definitions of two common terms in time complexity.
\begin{definition}
A function $f(s)$ is said to be $O(g(s))$ if $\exists$ a constant $c>0$ such that $|f(s)|\leq c|g(s)|$.  Here $s\in S$ which is the domain of the functions $f$ and $g$.
\end{definition}
\begin{definition}
    A function $f(s)$ is said to be $\Omega(g(s))$ if $\exists$ a constant $c>0$ such that  $|f(s)|\geq c|g(s)|$.
\end{definition}
In our analysis, we introduce the notations $\tilde{O}(.)$ and $\tilde{\Omega}(.)$ which hide all factors (including $\beta$, $\rho$, $d$, and $\alpha$) except $\eta_c$, $L_{min}$ and $L_{max}$ in $O$ and $\Omega$ respectively.

\section{Proof of convergence}\label{sec:proof}
We present our convergence proofs to theoretically show that the proposed PLRS method converges to a $\nu$-SOSP in finite time. We first state the assumptions that are instrumental for our proofs. 
\begin{assumption}
We now state the assumptions regarding the function $f:\mathbb{R}^d\rightarrow\mathbb{R}$ that we require for proving the theorems.
\begin{itemize}
    \item[\hypertarget{hyper:a1}{\textbf{A1}}] The function $f$ is $\beta$-smooth.
    \item[\hypertarget{hyper:a2}{\textbf{A2}}] The function $f$ is $\rho$-Hessian Lipschitz.
    \item[\hypertarget{hyper:a3}{\textbf{A3}}] The norm of the stochastic gradient noise is bounded i.e, $\norm{g(\mathbf{x}_t)-\nabla f(\mathbf{x}_t)}\leq Q\quad \forall t\geq 0$. Further, $\mathbb{E}[Q^2]\leq\sigma^2$.
    \item[\hypertarget{hyper:a4}{\textbf{A4}}] The function $f$ has strict saddle property.
        \item[\hypertarget{hyper:a5}{\textbf{A5}}] The function $f$ is bounded i.e., $|f(\mathbf{x})|\leq B,\; \forall \mathbf{x}\in \mathbb{R}^d$.
        \item[\hypertarget{hyper:a6}{\textbf{A6}}] The function $f$ is locally $\alpha-$strongly convex i.e, in the $\delta$-neighborhood of a locally optimal point $\mathbf{x}^{*}$ for some $\delta>0$.
\end{itemize}
\end{assumption}
\begin{remark}
     If $\nabla\tilde{f}(\tilde{\mathbf{x}}_t)$ and $\tilde{g}(\tilde{\mathbf{x}}_t)$ are the gradient and stochastic gradient of the second order Taylor approximation of $f$ about the iterate $\tilde{\mathbf{x}}_t$, from Assumption \hyperlink{hyper:a3}{\textbf{A3}}, it is implied that $\norm{\tilde{g}(\tilde{\mathbf{x}}_t)-\nabla \tilde{f}(\tilde{\mathbf{x}}_t)}\leq \tilde{Q}$. Further, $\mathbb{E}[\tilde{Q}^2]\leq\tilde{\sigma}^2$.
\end{remark}

Note that these assumptions are similar to those in the perturbed gradient literature \citep{ge2015escaping,jin2017escape,jin2021}. 
% \sout{The proof structure presented in this section is inspired by the convergence analysis provided for the noisy gradient descent algorithm proposed for escaping saddle points in the work by Ge et al. \citep{ge2015escaping}. Note that,  \eqref{update} varies from the update equation analyzed by Ge et al. \citep{ge2015escaping} due to the distinct characterization of the noise term in \eqref{noise_term}. Although the update may look similar to other perturbed gradient methods }
We call attention to two significant differences in our approach compared to other perturbed gradient methods such as \citep{jin2017escape,ge2015escaping,jin2021}: (i) In contrast to the isotropic \textit{additive} perturbation commonly added to the SGD update, we introduce randomness in our learning rate, manifested as \textit{multiplicative} noise in the update. This makes the characterization of the total noise dependent on the gradient, making the analysis challenging. 
% \comm{We do not have an additive noise mechanism but rather a multiplicative randomness in the SGD update, unlike the other works which prove convergence of noisy SGD.} 
    % \textcolor{orange}{This calls for considerable changes in the output of the analysis, although the structure may seem similar} 
    (ii) The magnitude of noise injected is computed through the smoothness constants in the work by Jin et al. \citep{jin2017escape,jin2021}; instead, we treat the parameters $L_{min}$ and $L_{max}$ as hyperparameters to be tuned. This enables our PLRS method to be easily applied to training deep neural networks where the computation of these smoothness constants could be infeasible due to sheer computational complexity.

\noindent We reiterate the update equations of the proposed SGD-PLRS.  
\begin{equation}\tag{\ref{update}}
        \mathbf{x}_{t+1} = \mathbf{x}_{t}-\eta_c\nabla f(\mathbf{x}_{t})-\mathbf{w}_{t}.
    \end{equation}
    \begin{equation}\tag{\ref{noise_term}}
        \mathbf{w}_{t}=\eta_c g(\mathbf{x}_{t}) - \eta_c\nabla f(\mathbf{x}_{t}) +  u_{t+1}g(\mathbf{x}_{t}).
    \end{equation}

Note that the term $\mathbf{w}_t$ has zero mean and we state this formally in the lemma below.
\begin{lemma}[Zero mean property] \label{lemma:zeromean}
        The mean of $\mathbf{w}_{t-1} \; \forall t\geq1$ is $0$.
        \end{lemma} 
\begin{proof}
    \begin{equation}
    \begin{aligned}
        \mathbb{E}[\mathbf{w}_{t-1}] &= \mathbb{E}\left[\eta_c g(\mathbf{x}_{t-1}) - \eta_c\nabla f(\mathbf{x}_{t-1})\right] + \mathbb{E}\left[u_t g(\mathbf{x}_{t-1})\right] \\
        &= \mathbf{0} \qquad \forall t \ge 1.
    \end{aligned}
    \end{equation}
    This follows as $\mathbb{E}[u_{t}] = \frac{L_{\min} + L_{\max} - 2\eta_c}{2} = 0$ and $\mathbb{E}\left[g(\mathbf{x}_{t-1})\right] = \nabla f(\mathbf{x}_{t-1})$. \qedhere
\end{proof}
For a function satisfying the Assumptions \hyperlink{hyper:a1}{\textbf{A1}}-\hyperlink{hyper:a6}{\textbf{A6}}, there are three possibilities for the iterate $\textbf{x}_t$ with respect to the function's gradient and Hessian, namely, \hypertarget{hyper:b1}{\textbf{B1:}} Gradient is large; \hypertarget{hyper:b2}{\textbf{B2:}} Gradient is small and iterate is around a saddle point; \hypertarget{hyper:b3}{\textbf{B3:}} Gradient is small and iterate is around a $\nu$-SOSP.
% \begin{itemize}
%     \item[\hypertarget{hyper:b1}{\textbf{B1}}] Gradient is large.
%     % : $\norm{\nabla f(\mathbf{x}_t)}\geq \epsilon$.
%     \item[\hypertarget{hyper:b2}{\textbf{B2}}] Gradient is small; iterate is around a saddle point.
%     % : $\norm{\nabla f(\mathbf{x}_t)}\leq \epsilon, \lambda_{min}(\textbf{H}(\mathbf{x}_t))\leq -\gamma$.
%     \item[\hypertarget{hyper:b3}{\textbf{B3}}] Gradient is small; iterate is around a $\nu$-SOSP.
%     % : $\norm{\nabla f(\mathbf{x}_t)}\leq \epsilon,\lambda_{min}(\textbf{H}(\mathbf{x}_t))\geq -\gamma$.
% \end{itemize}

% Here $\epsilon,\gamma>0$. 
We now present three theorems corresponding to each of these cases. Our first result pertains to the case \hyperlink{hyper:b1}{\textbf{B1}} where the gradient of the iterate is large. 
\begin{theorem} \label{th:1}
Under the assumptions \hyperlink{hyper:a1}{\textbf{A1}} and \hyperlink{hyper:a3}{\textbf{A3}} with $L_{max} < \frac{1}{\beta}$, for any point $\mathbf{x}_t$ with $\norm{\nabla f(\mathbf{x}_t)}\geq \sqrt{3\eta_c\beta\sigma^2}$ where $\sqrt{3\eta_c\beta\sigma^2}<\epsilon$, after one iteration, we have
\[
\mathbb{E}[f(\mathbf{x}_{t+1})] - f(\mathbf{x}_t) \leq - \tilde{\Omega}(L_{max}^2).
\]
\end{theorem}
This theorem suggests that, for any iterate $\mathbf{x}_t$ for which the gradient is large, the expected functional value of the subsequent iterate $f(\mathbf{x}_{t+1})$ decreases, and the corresponding decrease $\mathbb{E}[f(\mathbf{x}_{t+1})] - f(\mathbf{x}_t)$ is in the order of $\tilde{\Omega}(L_{max}^2)$. The formal proof for this theorem can be found in Appendix \ref{ap:th1proof}.

The next theorem corresponds to the case \hyperlink{hyper:b2}{\textbf{B2}} where the gradient is small and the Hessian is negative.
\begin{theorem} \label{th:2}
Consider $f$ satisfying Assumptions \hyperlink{hyper:a1}{\textbf{A1}} - \hyperlink{hyper:a5}{\textbf{A5}}. Let $\{\textbf{x}_t\}$ be the SGD iterates of the function $f$ using PLRS. Let $\norm{\nabla f(\textbf{x}_0)}\leq \sqrt{3\eta_c\beta\sigma^2}<\epsilon$ and $ \lambda_{min}(\textbf{H}(\textbf{x}_0)) \leq  -\gamma$ where $\epsilon,\gamma>0$. Then, there exists a $T=\tilde{O}\left(L_{max}^{-1/4}\right)$ such that with probability at least $1-\tilde{O}\left(L_{max}^{7/2}\right)$, 
\[\mathbb{E}[f(\textbf{x}_T) - f(\textbf{x}_0)] \leq - \tilde{\Omega}\left(L_{max}^{3/4}\right).\]
\end{theorem}
\noindent The formal proof of this theorem is provided in Appendix \ref{ap:th2proof}. The sketch of the proof is given below.\\
\begin{proofsketch}
    This theorem shows that the iterates obtained using PLRS escape from a saddle point $\textbf{x}_0$ (where the gradient is small, and the Hessian has atleast one negative eigenvalue), i.e, it shows the decrease in the expected value of the function $f$ after $T=\tilde{O}\left(L_{max}^{-1/4}\right)$ iterations. Note that for a $\rho-$Hessian smooth function,
\begin{equation}\label{sketch1}
    f(\textbf{x}_T)\leq f(\textbf{x}_0)+\nabla f(\textbf{x}_0)^T(\textbf{x}_T-\textbf{x}_0)+\frac{1}{2}(\textbf{x}_T-\textbf{x}_0)^T \textbf{H}(\textbf{x}_0)(\textbf{x}_T-\textbf{x}_0)+\frac{\rho}{6}\norm{\textbf{x}_T-\textbf{x}_0}^3.
\end{equation}
To evaluate $\mathbb{E}[f(\textbf{x}_T) - f(\textbf{x}_0)]$ from \eqref{sketch1}, we require an analytical expression for $\textbf{x}_T-\textbf{x}_0$, which is not tractable. Hence, we employ the second-order Taylor approximation of the function $f$, which we denote as $\tilde{f}$. We then apply SGD-PLRS on $\tilde{f}$ to obtain $\tilde{\textbf{x}}_T$. Following this, we write $\textbf{x}_T-\textbf{x}_0$ = $(\textbf{x}_T-\tilde{\mathbf{x}}_T) + (\tilde{\mathbf{x}}_T -\textbf{x}_0$) and derive expressions for upper bounds on $\tilde{\textbf{x}}_T-\textbf{x}_0$ and $\textbf{x}_T-\tilde{\textbf{x}}_T$ which hold with high probability in Lemmas \ref{lemma:xt_x0} and \ref{lemma:xt_tildext}, respectively  (given in Appendix \ref{ap:lemma1} and \ref{ap:lemma2}). 
We split the quadratic term in \eqref{sketch1} into two parts corresponding to $\tilde{\textbf{x}}_T-\textbf{x}_0$ and $\textbf{x}_T-\tilde{\textbf{x}}_T$. We further decompose the term, say $\mathcal{Y}=(\tilde{\textbf{x}}_T-\textbf{x}_0)^T \textbf{H}(\textbf{x}_0)(\tilde{\textbf{x}}_T-\textbf{x}_0)$ into its eigenvalue components along each dimension with corresponding eigenvalues $\lambda_1,\hdots,\lambda_d$ of $\textbf{H}(\textbf{x}_0)$. Our main result in this theorem proves that the term $\mathcal{Y}$ dominates over all the other terms of \eqref{sketch1}, and that it is bounded by a negative value, thereby, proving $\mathbb{E}[f(\textbf{x}_T)]\leq f(\textbf{x}_0)$. This main result uses a two-pronged proof. Firstly, we use our assumption that the initial iterate $\textbf{x}_0$ is at a saddle point and hence at least one of $\lambda_i,\quad 1\leq i\leq d$ is negative. We formally show that the eigenvector corresponding to this eigenvalue points to the direction of escape.
% Further, under some assumptions upon the learning rate as in \cite{ge2015escaping}, we prove that for a particular order of iterations $T$, the dominating component of  $(\tilde{\textbf{x}}_T-\textbf{x}_0)^T \nabla^2f(\textbf{x}_0)(\tilde{\textbf{x}}_T-\textbf{x}_0)$ is negative, meaning that $\gamma_o$ is sufficient to push the summation to a negative value. 
Secondly, we use the second order statistics of our noise, to show that the magnitude of $\mathcal{Y}$ is large enough to dominate over the other terms of \eqref{sketch1}. Note that our noise term involves the stochasticity in the gradient and the probabilistic learning rate.
% Secondly, we observe that the quadratic term also has a multiplicative noise component as a result of expanding $\textbf{x}_T-\textbf{x}_0$. Computing its second order statistics, we find that its magnitude is large enough to enable the quadratic term to dominate over the other terms of \eqref{sketch1} in terms of $\tilde{O}$. 
Hence, we have shown that the negative eigenvalue of the Hessian at a saddle point and the unique characterization of the noise is sufficient to force a descent along the negative curvature safely out of the region of the saddle point within $T$ iterations. 
\end{proofsketch}
As each SGD-PLRS update is noisy, we need to ensure that once we escape a saddle point and move towards a local minimum (case \hyperlink{hyper:b3}{\textbf{B3}}), we do not overshoot the minimum but rather, stay in the $\delta-$neighborhood of an SOSP, with high probability.
% , i.e., the noise component must be controlled enough so that convergence can be guaranteed. 
We formalize this in Theorem \ref{th:3}.
\begin{theorem}\label{th:3}
    Consider $f$ satisfying the assumptions \hyperlink{hyper:a1}{\textbf{A1}}-\hyperlink{hyper:a6}{\textbf{A6}}. Let the initial iterate $\textbf{x}_0$ be $\delta$ close to a local minimum $\textbf{x}^{*}$ such that $\norm{\textbf{x}_0-\textbf{x}^{*}}\leq\tilde{O}(\sqrt{L_{max}})<\delta$. With probability at least $1-\xi$,  $\forall t\leq T$ where $T= \tilde{O}\left(\frac{1}{L_{max}^2}\log\frac{1}{\xi}\right)$,
    \begin{equation*}
        \norm{\textbf{x}_t-\textbf{x}^{*}}\leq \tilde{O}\left(\sqrt{L_{max}\log\frac{1}{L_{max}\xi}}\right)<\delta
    \end{equation*}
\end{theorem}
 This theorem deals with the case that the initial iterate $\textbf{x}_0$ is $\delta$-close to a local minimum $\textbf{x}^{*}$ (case \hyperlink{hyper:b3}{\textbf{B3}}). We prove that the subsequent iterates are also in the same neighbourhood, i.e., $\delta$ close to the local minimum, with high probability. In other words, we prove that the sequence $\{\norm{\textbf{x}_t-\textbf{x}^{*}}\}$ is bounded by $\delta$ for $t\leq T$. 
 In the neighbourhood of the local minimum, gradients are small and subsequently, the change in iterates, $\textbf{x}_t-\textbf{x}_{t-1}$ are minute. Therefore, the iterates stay near the local minimum with high probability. It is worth noting that the nature of the noise, which is comprised of stochastic gradients (whose stochasticity is bounded by $Q$) multiplied with a bounded uniform random variable (owing to PLRS), aids in proving our result. We provide the formal proof in Appendix \ref{ap:th3_proof}.
 
 \section{Empirical evaluation}\label{sec:results}
We provide results on CIFAR-10, CIFAR-100 \citep{cifar} and Tiny ImageNet \citep{le2015tiny} and compare with: (i) cosine annealing with warm restarts \citep{sgdr}, (ii) one-cycle scheduler \citep{smith2019super}, (iii) knee scheduler \citep{knee}, (iv) constant learning rate and (v) multi-step decay scheduler. We run experiments for 500 epochs for the CIFAR datasets and for 100 epochs for the Tiny ImageNet dataset using the SGD optimizer for all schedulers
\footnote{We provide results without momentum to be consistent with our theoretical framework. When we used the SGD optimizer with momentum for PLRS, we obtain results better than those reported without momentum.}.
% However, we have observed that the performance of PLRS remains consistent even with momentum.}. 
We also set all other regularization parameters, such as weight decay and dampening, to zero. We use a batch size of 64 for DenseNet-40-12, 50 for ResNet-50, and 128 for the others. We conduct all our experiments in a single NVIDIA GeForce RTX 2080 GPU card.\\
To determine the parameters $L_{min}$ and $L_{max}$ for PLRS, we perform a range test, where we observe the training loss for a range of learning rates as is done in state-of-the-art LR schedulers such as one-cycle \citep{smith2019super} and knee schedulers \citep{knee}. As the learning rate is gradually increased, we first observe a steady decrease in the training loss, then followed by a drastic increase. We note the learning rate at which there is an increase of training loss, say $\bar{L}$ and choose the maximum learning rate $L_{max}$ to be just below $\bar{L}$, where the loss is still decreasing.
% within $0.9 \bar{L}$. 
% This is done to ensure that $L_{max}$ is in that region of learning rates that does not cause an explosion in the training loss.
We then tune $L_{min}$ such that $0<L_{min}<L_{max}$.
% We then choose $L_{min}$ such that it is neither too close to $L_{max}$ and nor too small. Specifically, we choose $L_{min}$ such that it is not farther than $0.1$ from $L_{max}$. 
We choose the parameters for the baseline schedulers as suggested in the original papers (further details of parameters are provided in Appendix \ref{ap:parameters}). 
% \comm{In addition to providing results on neural networks, we also test the effectiveness of our algorithm in the context of online tensor decomposition. We provide results in that setting in Appendix \ref{ap:tensor}. }

% Having given theoretical results and empirical experiments in support of SGD with multiplicative noise, we also try to empirically compare against SGD with the additive noise update. 
While there are ample works which prove the convergence of SGD with additive noise as in \citep{ge2015escaping}, they cannot be ported into practice for deep neural networks. They require smoothness constants \citep{zhang2017hitting,jin2021,arjevani2023lower} or functional bounds on the norms of the function derivatives \citep{yiming2025efficiently} to be computed for the additive noise injection, which can not be obtained for the loss functions of neural networks or can only be approximated locally \citep{latorre2020lipschitz}. Further, the empirical convergence properties of noisy SGD are not demonstrated through examples in the majority of these analytical works which makes it hard to compare their convergence with PLRS. However, we compare our proposed PLRS against the noisy SGD mechanism proposed by Ge et al. \citep{ge2015escaping} providing convergence results on the online tensor decomposition problem using the code provided by the authors in Appendix \ref{ap:tensor}. 
% Further, we also provide results using the WRN-28-10 architecture on the CIFAR-10 dataset with the additive noise update which we demonstrate in Appendix \ref{ap:wrn_noisysgd}.}
\subsection{Results for CIFAR-10}
\begin{table}
\begin{center}
\begin{tabular}{cccc}
\hline

\textbf{Architecture} & \textbf{Scheduler} & \textbf{Max acc.} & \textbf{Mean acc. (S.D)}\\ \hline

VGG-16 & Cosine & 96.87 & 96.09 (0.78) \\
VGG-16 & {Knee} & 96.87 & \textbf{96.35} (0.45)\\ 
VGG-16 & {One-cycle} & {90.62} & 89.06 (1.56)\\ 
VGG-16 & Constant & 96.09 & 96.06 (0.05) \\
VGG-16 & Multi-step & 92.97 & 92.45 (0.90) \\
VGG-16 & PLRS (ours) & \bf{97.66} & 96.09 (1.56)\\ \hline
WRN-28-10 & Cosine & 92.03 & 91.90 (0.13) \\
WRN-28-10 & Knee & \bf{92.04} & 91.64 (0.63) \\
WRN-28-10 & {One-cycle} & 87.76 & 87.37 (0.35) \\
WRN-28-10 & Constant & \bf{92.04} & \textbf{92.00} (0.08) \\
WRN-28-10 & Multi-step & 88.94 & 88.80 (0.21) \\
WRN-28-10 & PLRS (ours) & 92.02 & 91.43 (0.54)\\ \hline

\end{tabular}
\caption{ Maximum and mean (with standard deviation) test accuracies over 3 runs for CIFAR-10.}
\label{table_cifar10}
\end{center}
\end{table}

We consider VGG-16 \citep{vgg} and WRN-28-10 \citep{wrn} and use $L_{min}=0.07$ and $L_{max}=0.1$ for both the networks. We record the maximum and mean test accuracies across different LR schedulers in Table \ref{table_cifar10}. The highest accuracy across schedulers is recorded in bold. For the VGG-16 network, we rank the highest in terms of maximum test accuracy. 
    % \textcolor{blue}{which is a $0.81\%$ increase}. 
    In terms of the mean test accuracy over $3$ runs, the knee scheduler outperforms the rest. Note that the second highest mean test accuracy is achieved by both PLRS and the cosine annealing schedulers. Unsurprisingly, the constant scheduler has the lowest standard deviation. In the WRN-28-10 network, note that the maximum test accuracies for the cosine, knee, constant and the PLRS schedulers are very similar (difference in the order of $10^{-2}$). Their similar performance is also reflected in the mean test accuracies although the constant learning rate edges out the other schedulers marginally. To study the convergence of the schedulers we also plot the training loss across epochs in Figure \ref{fig:cif10}. We observe that our proposed PLRS achieves one of the fastest rates of convergence in terms of the training loss compared across all the schedulers for both networks. Note that the cosine annealing scheduler records several spikes across the training. 
    \begin{figure}
    \centering
    \begin{subfigure}[b]{0.5\textwidth}
        \centering
        % \includegraphics[width=1\textwidth]
        % {Images/vgg_trainloss_unsamp_whitegrid.eps}% second figure itself
        \includegraphics[width=1\textwidth]{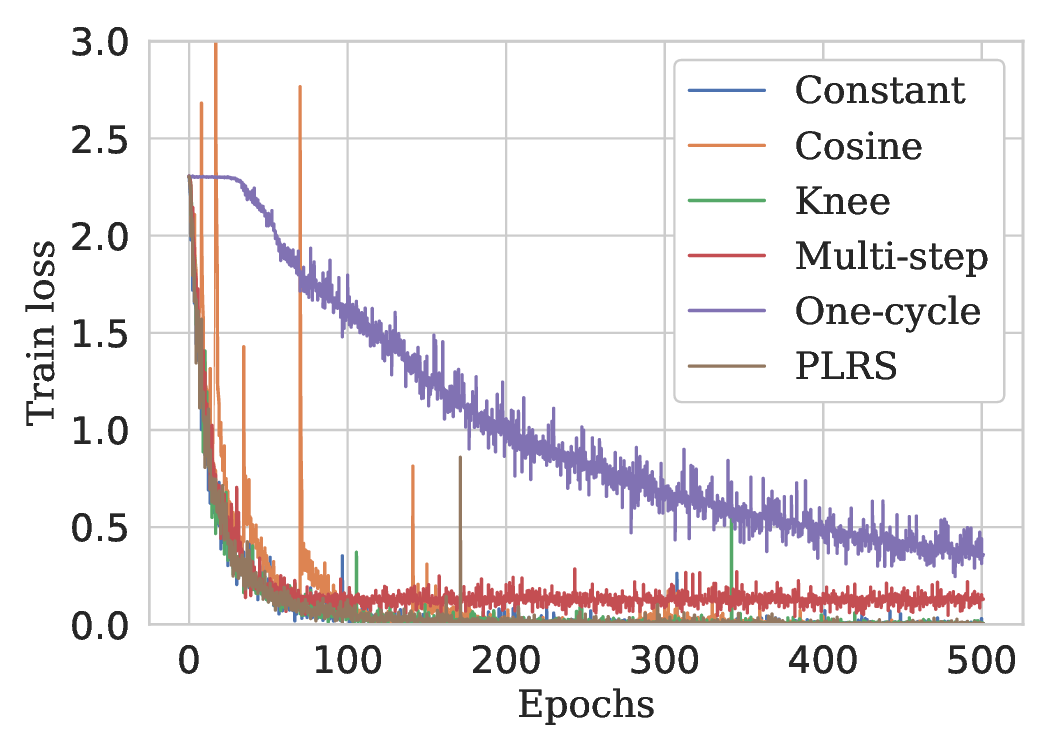}% second figure itself
        \caption{VGG-16}
    \end{subfigure}\hfill
    \begin{subfigure}[b]{0.5\textwidth}
        \centering
        \includegraphics[width=1\textwidth]{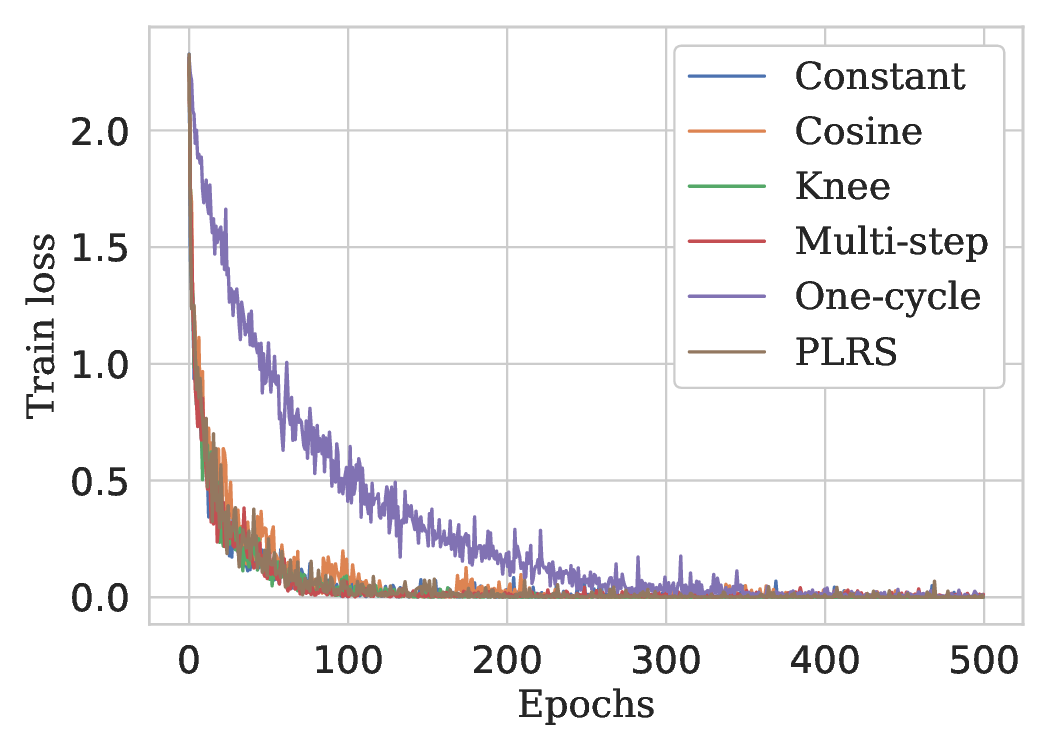}
        \caption{WRN-28-10}
    \end{subfigure}
    \caption{Training loss vs epochs for VGG-16 and WRN-28-10 for CIFAR-10.}
    \label{fig:cif10}
\end{figure}
\subsection{Results for CIFAR-100}
\begin{table}
\fontsize{9pt}{9pt}
\begin{center}
\begin{tabular}{cccc}
\hline

\textbf{Architecture} & \textbf{Scheduler} & \textbf{Max acc.} & \textbf{Mean acc.(S.D)}\\ \hline

ResNet-110 & Cosine & 74.22 & 72.66 (1.56) \\
ResNet-110 & {Knee} & 75.78 & 72.39 (2.96)\\ 
ResNet-110 & {One-cycle} & {71.09} & 70.05 (1.19)\\ 
ResNet-110 & Constant & 69.53 & 66.67 (2.51) \\
ResNet-110 & Multi-step & 63.28 & 61.20 (2.39) \\
ResNet-110 & PLRS (ours) & \bf{77.34} & \textbf{74.61} (2.95)\\ \hline
DenseNet-40-12 & Cosine   & {82.81} & 80.47 (2.07)\\
DenseNet-40-12 & {Knee}   & {82.81} & 80.73 (2.39)\\ 
DenseNet-40-12 & {One-cycle}  & 73.44 & 72.39 (0.90) \\ 
DenseNet-40-12 & Constant & 82.81 & 80.73 (2.39) \\
DenseNet-40-12 & Multi-step & \bf{87.50} & \textbf{84.89} (2.39) \\
DenseNet-40-12 & PLRS (ours) & 84.37 & 83.33 (0.90) \\ \hline

\end{tabular}
\caption{Maximum and mean (with standard deviation) test accuracies over 3 runs for CIFAR-100.}
\label{table:cifar100}
\end{center}
\end{table}
% We present results for experiments on the CIFAR-100 dataset using the ResNet-110 \citep{resnet} and DenseNet-40-12 \citep{densenet} networks in this section. 
We consider networks ResNet-110 \citep{resnet} and DenseNet-40-12 \citep{densenet}, and use $L_{min}=0.07$ and $L_{max}=0.1$ for the former, and $L_{min}=0.1$ and $L_{max}=0.2$ for the latter. The maximum and the mean test accuracies (with standard deviation) across $3$ runs are provided in Table \ref{table:cifar100}. For ResNet-110, PLRS performs best in terms of the maximum and the mean test accuracies. 
% \textcolor{blue}{Specifically, we record a $2.05\%$ higher maximum test accuracy.} 
This is closely followed by the other state-of-the-art LR schedulers such as knee and cosine schedulers. For the DenseNet-40-12 network, PLRS comes to a close second to the multi-step LR scheduler in terms of the maximum and mean test accuracies. However, it is important to note that the multi-step scheduler records the least test accuracy with the ResNet-110 network. Hence, its performance is not consistent across the networks, while PLRS is consistently one of the best performing schedulers.

We plot the training loss in Figure \ref{fig:cif100}. For ResNet-110, both PLRS and knee LR scheduler converge to a low training loss around $150$ epochs. While cosine annealing LR scheduler also seems to converge fast, it experiences sharp spikes along the curve during the restarts. For DenseNet-40-12, PLRS converges faster to a lower training loss compared to the other schedulers.
\begin{figure}
    \centering
    \begin{subfigure}[b]{0.5\textwidth}
        \centering
        % \includegraphics[width=1\textwidth]
        % {Images/resnet_train_loss_unsamp_whitegrid.eps}% second figure itself
        \includegraphics[width=1\textwidth]{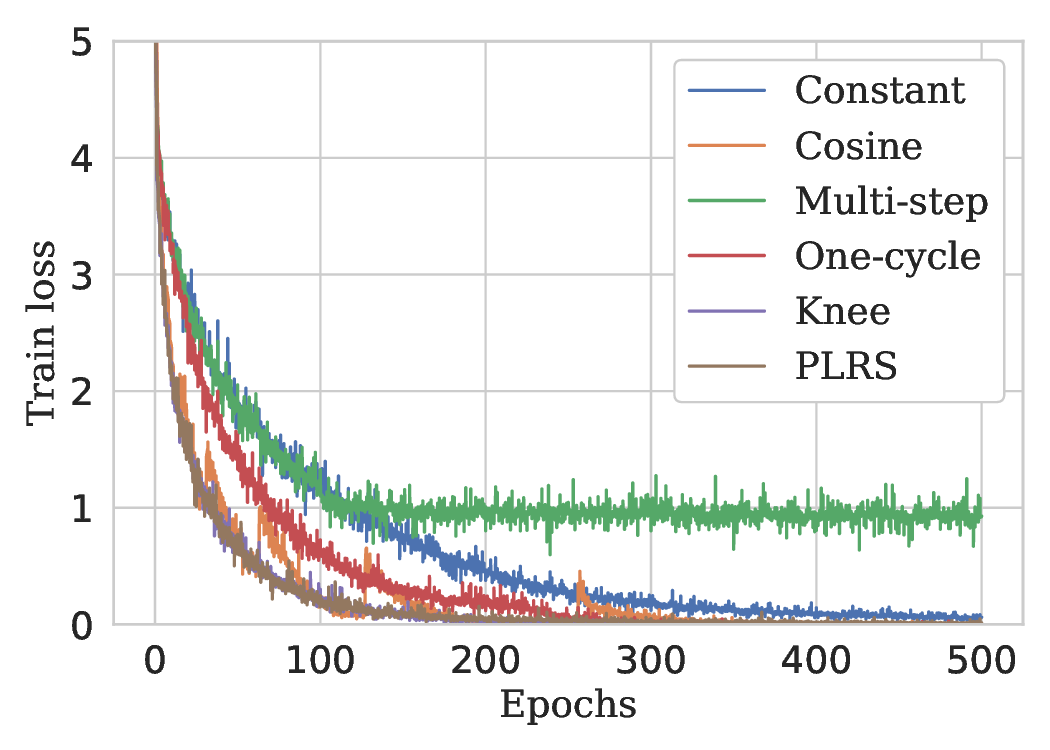}% second figure itself
        \caption{ResNet-110}
    \end{subfigure}%
    ~ 
    \begin{subfigure}[b]{0.5\textwidth}
        \centering
        \includegraphics[width=1\textwidth]{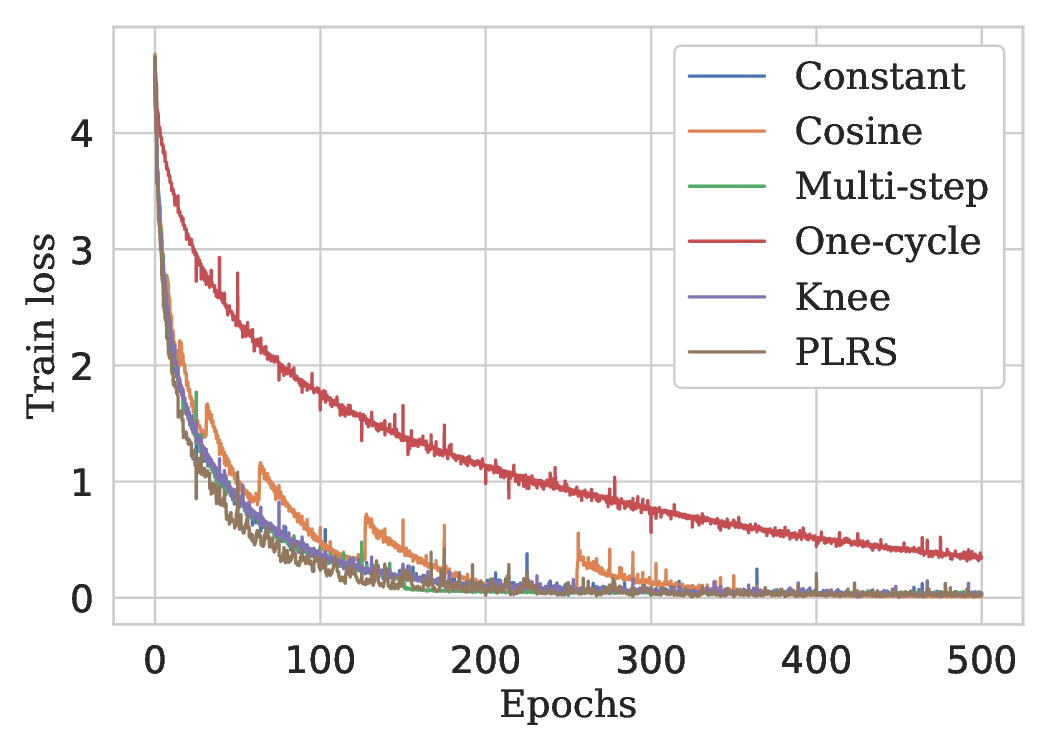}
        \caption{DenseNet-40-12}
    \end{subfigure}
    \caption{Training loss vs epochs for ResNet-110 and DenseNet-40-12 on CIFAR-100.}
    \label{fig:cif100}
\end{figure}
\subsection{Results for Tiny ImageNet}

We consider the Resnet-50 \citep{resnet} architecture and use $L_{min}=0.35$ and $L_{max}=0.4$. We present the maximum and mean test accuracies in Table \ref{table:tiny}. We provide the plot of training loss in Figure \ref{fig:tiny}. PLRS performs the best in terms of maximum test accuracy. In terms of mean test accuracy, it ranks second next to cosine annealing by a close margin. 
% Also, PLRS displays the best behaviour in terms of the convergence of training loss. 
It can be observed that PLRS achieves the fastest convergence to the lowest training loss compared to others. Moreover, it exhibits stable convergence, especially when compared cosine annealing, which experiences multiple spikes due to warm restarts.

\subsection{Limitations and broader impact}\label{sec:limit}
In line with all other works which focus on convergence proofs, our work too applies only to a restricted class of functions that meet the assumptions in Section \ref{sec:proof}.
% Our theoretical convergence proofs apply only to a restricted class of functions that meet the assumptions mentioned in Section \ref{sec:proof}. 
In contrast, our experiments are conducted on deep neural networks, which may not strictly satisfy these assumptions. While this is a limitation of our work, we note that many papers focused on theoretical convergence of SGD do not include empirical results, and many practice-oriented papers proposing new LR schedulers do not include convergence proofs. Another limitation is that our experiments are limited to benchmark image datasets, even though our proposed scheduler is general and can be applied to other domains.

Our work contributes to the relatively underexplored theoretical understanding of LR schedulers, an area where most prior research has focused on empirical or application-driven results. As discussed earlier, commonly used periodic schedulers, such as triangular or cosine annealing, can be viewed as special cases of our proposed PLRS. This generalization opens new avenues for theoretical investigation, including the analysis of convergence properties across a broader class of schedulers. In practice, PLRS demonstrates improved stability and enables faster convergence, reducing the number of training epochs required. This efficiency translates to lower GPU usage and energy consumption, supporting more sustainable and resource-conscious AI development.

\begin{figure}
\centering
% \fontsize{9pt}{9pt}
\begin{minipage}[c]{0.48\textwidth}
\vspace{0pt}
\centering
  \begin{tabular}{cccc}
  \\
  \\
  \\
\hline

 \textbf{Scheduler} & \textbf{Max acc.} & \textbf{Mean acc. (S.D)}\\ \hline

 Cosine & 62.13 & \textbf{62.03} (0.15) \\
{Knee} & 61.93 & {61.50} (0.42)\\ 
 {One-cycle} & {52.24} & 51.99 (0.22)\\ 
 Constant & 61.59 & 61.11 (0.42) \\
 Multi-step & 61.28 & 61.20 (0.08) \\
 PLRS (ours) & \bf{62.34} & {61.90} (0.73)\\ \hline
\\
\\
\\
\end{tabular}

    \captionof{table}{Maximum and mean (with standard deviation) test accuracies over 3 runs for Tiny ImageNet.}
    \label{table:tiny}
\end{minipage}
\hspace{2mm}
 \begin{minipage}[c]{0.48\textwidth}
    \centering
  \includegraphics[width=1\textwidth]{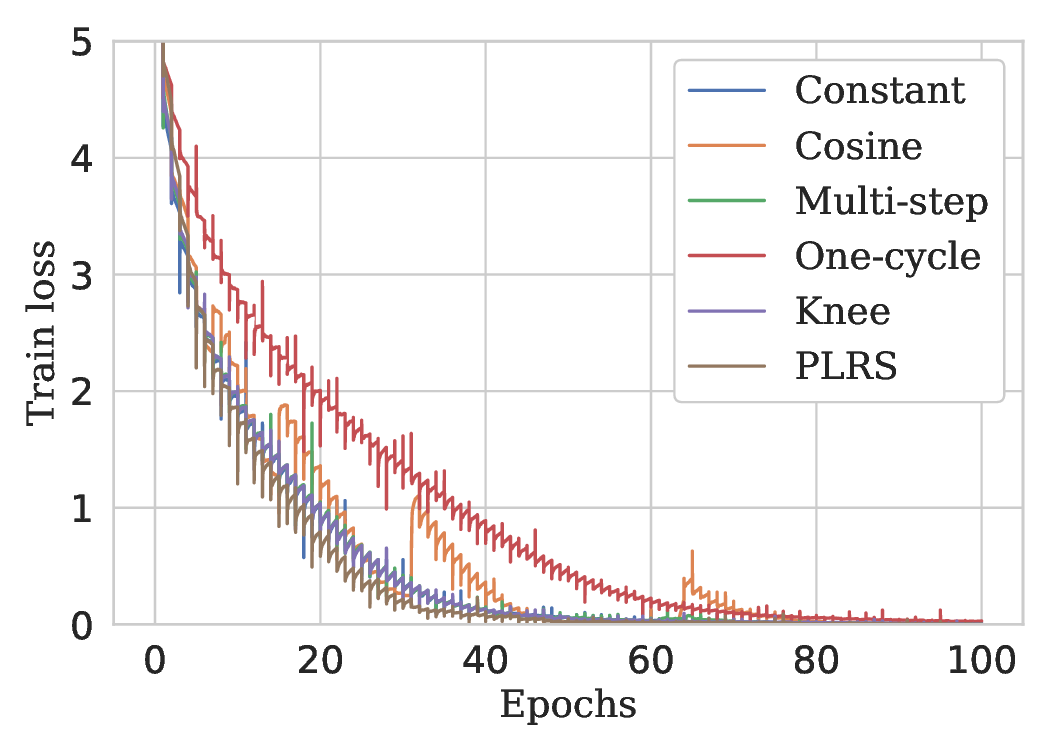}
  \captionof{figure}{Training loss vs epochs for ResNet-50 with Tiny ImageNet.}
  \label{fig:tiny}
  \end{minipage}

\end{figure}

% \begin{figure}
% \begin{floatrow}
% \fontsize{9pt}{9pt}
% \capbtabbox{%
%   \begin{tabular}{cccc}
% \hline

%  \textbf{Scheduler} & \textbf{Max acc.} & \textbf{Mean acc. (S.D)}\\ \hline

%  Cosine & 62.13 & \textbf{62.03} (0.15) \\
% {Knee} & 61.93 & {61.50} (0.42)\\ 
%  {One-cycle} & {52.24} & 51.99 (0.22)\\ 
%  Constant & 61.59 & 61.11 (0.42) \\
%  Multi-step & 61.28 & 61.20 (0.08) \\
%  PLRS (ours) & \bf{62.34} & {61.90} (0.73)\\ \hline
% \\
% \\
% \\
% \\
% \end{tabular}
% }{%
%   \caption{ Maximum and mean (with standard deviation) test accuracies over 3 runs for Tiny ImageNet.}
%   \label{table:tiny}%
% }
% \ffigbox{%
%   % \includegraphics[width=0.5\textwidth]{Images/tiny_loss.eps}%
%   \includegraphics[width=0.5\textwidth]{Figures/tiny.eps}%
% }{%
%   \caption{Training loss vs epochs for ResNet-50 with Tiny ImageNet.}%
%   \label{fig:tiny}
% }
% \end{floatrow}
% \end{figure}

\section{Concluding remarks}
We have proposed the novel idea of a probabilistic LR scheduler. The probabilistic nature of the scheduler helped us provide the first theoretical convergence proofs for SGD using LR schedulers.
% In this work, we have proposed a new probabilistic learning rate scheduler. We have provided theoretical proofs for the convergence of SGD with our proposed scheduler. The novelty of our proofs is that we do not assume a constant or a decaying learning rate for the proof of convergence. 
% As we know, the theory in machine learning is always behind implementation and 
In our opinion, this is a significant step in the right direction to bridge the gap between theory and practice in the LR scheduler domain. 
% We hope that this work inspires further theoretical exploration in justifying the use of LR schedulers. 
Our empirical results show that our proposed LR scheduler performs competitively with the state-of-the-art cyclic schedulers, if not better, on CIFAR-10, CIFAR-100, and Tiny ImageNet datasets for a wide variety of popular deep architectures. 
This leads us to hypothesize that the proposed probabilistic LR scheduler acts as a super-class of LR schedulers encompassing both probabilistic and deterministic schedulers. 
% We argue that the convergence proof provided for our probabilistic learning rate can be extended to include schedulers such as triangle learning rate scheduler and cosine annealing with the appropriate range of the noise perturbation. 
Future research directions include further exploration of this hypothesis.

% \section*{References}
\bibliographystyle{plain}
\bibliography{neurips_2025}

\begin{thebibliography}{10}

\bibitem{andriushchenko2020understanding}
Maksym Andriushchenko and Nicolas Flammarion.
\newblock Understanding and improving fast adversarial training.
\newblock In {\em Advances in Neural Information Processing Systems}, volume~33, pages 16048--16059, 2020.

\bibitem{arjevani2023lower}
Yossi Arjevani, Yair Carmon, John~C Duchi, Dylan~J Foster, Nathan Srebro, and Blake Woodworth.
\newblock Lower bounds for non-convex stochastic optimization.
\newblock {\em Mathematical Programming}, 199(1):165--214, 2023.

\bibitem{bergstra2012random}
James Bergstra and Yoshua Bengio.
\newblock Random search for hyper-parameter optimization.
\newblock {\em Journal of machine learning research}, 13(2):281--305, 2012.

\bibitem{bergstra2013making}
James Bergstra, Daniel Yamins, and David Cox.
\newblock Making a science of model search: Hyperparameter optimization in hundreds of dimensions for vision architectures.
\newblock In {\em International conference on machine learning}, pages 115--123, 2013.

\bibitem{caron2021emerging}
Mathilde Caron, Hugo Touvron, Ishan Misra, Herv{\'e} J{\'e}gou, Julien Mairal, Piotr Bojanowski, and Armand Joulin.
\newblock Emerging properties in self-supervised vision transformers.
\newblock In {\em Proceedings of the IEEE/CVF international conference on computer vision}, pages 9650--9660, 2021.

\bibitem{dauphin2014}
Yann~N Dauphin, Razvan Pascanu, Caglar Gulcehre, Kyunghyun Cho, Surya Ganguli, and Yoshua Bengio.
\newblock Identifying and attacking the saddle point problem in high-dimensional non-convex optimization.
\newblock In {\em Advances in neural information processing systems}, volume~27, page 2933–2941, 2014.

\bibitem{dhillon2019baseline}
Guneet~Singh Dhillon, Pratik Chaudhari, Avinash Ravichandran, and Stefano Soatto.
\newblock A baseline for few-shot image classification.
\newblock In {\em International Conference on Learning Representations}, 2020.

\bibitem{adagrad}
John Duchi, Elad Hazan, and Yoram Singer.
\newblock Adaptive subgradient methods for online learning and stochastic optimization.
\newblock {\em Journal of machine learning research}, 12(7):2121--2159, 2011.

\bibitem{ge2015escaping}
Rong Ge, Furong Huang, Chi Jin, and Yang Yuan.
\newblock Escaping from saddle points: Online stochastic gradient for tensor decomposition.
\newblock {\em Journal of Machine Learning Research}, 40:1--46, 2015.

\bibitem{resnet}
Kaiming He, Xiangyu Zhang, Shaoqing Ren, and Jian Sun.
\newblock Deep residual learning for image recognition.
\newblock In {\em Proceedings of the IEEE conference on computer vision and pattern recognition}, pages 770--778, 2016.

\bibitem{hoeffding1994probability}
Wassily Hoeffding.
\newblock Probability inequalities for sums of bounded random variables.
\newblock {\em The collected works of Wassily Hoeffding}, pages 409--426, 1994.

\bibitem{ulm}
Jeremy Howard and Sebastian Ruder.
\newblock Universal language model fine-tuning for text classification.
\newblock In {\em Proceedings of the 56th Annual Meeting of the Association for Computational Linguistics}, pages 328--339, Melbourne, Australia, 2018.

\bibitem{densenet}
Gao Huang, Zhuang Liu, Laurens Van Der~Maaten, and Kilian~Q Weinberger.
\newblock Densely connected convolutional networks.
\newblock In {\em Proceedings of the IEEE conference on computer vision and pattern recognition}, pages 4700--4708, 2017.

\bibitem{knee}
Nikhil Iyer, V~Thejas, Nipun Kwatra, Ramachandran Ramjee, and Muthian Sivathanu.
\newblock Wide-minima density hypothesis and the explore-exploit learning rate schedule.
\newblock {\em Journal of Machine Learning Research}, 24(65):1--37, 2023.

\bibitem{jin2017escape}
Chi Jin, Rong Ge, Praneeth Netrapalli, Sham~M Kakade, and Michael~I Jordan.
\newblock How to escape saddle points efficiently.
\newblock In {\em International conference on machine learning}, pages 1724--1732. PMLR, 2017.

\bibitem{jin2021}
Chi Jin, Praneeth Netrapalli, Rong Ge, Sham~M. Kakade, and Michael~I. Jordan.
\newblock On nonconvex optimization for machine learning: Gradients, stochasticity, and saddle points.
\newblock {\em Journal of the Association for Computing Machinery}, 68(2):1--29, 2021.

\bibitem{adam}
Diederik~P Kingma and Jimmy Ba.
\newblock Adam: A method for stochastic optimization.
\newblock {\em arXiv preprint arXiv:1412.6980}, 2014.

\bibitem{cifar}
Alex Krizhevsky, Geoffrey Hinton, et~al.
\newblock Learning multiple layers of features from tiny images.
\newblock Technical Report TR-2009, University of Toronto, Toronto, ON, Canada, 2009.

\bibitem{latorre2020lipschitz}
Fabian Latorre, Paul Thierry~Yves Rolland, and Volkan Cevher.
\newblock Lipschitz constant estimation for neural networks via sparse polynomial optimization.
\newblock In {\em 8th International Conference on Learning Representations}, 2020.

\bibitem{le2015tiny}
Ya~Le and Xuan Yang.
\newblock Tiny imagenet visual recognition challenge.
\newblock [Online]. Available: https://tinyimagenet.herokuapp.com, 2015.

\bibitem{sgdr}
Ilya Loshchilov and Frank Hutter.
\newblock {SGDR:} stochastic gradient descent with warm restarts.
\newblock In {\em 5th International Conference on Learning Representations}, 2017.

\bibitem{nesterov}
Yurii Nesterov.
\newblock {\em Introductory Lectures on Convex Optimization: A Basic Course}.
\newblock Springer Publishing Company, Incorporated, 2014.

\bibitem{vgg}
K~Simonyan and A~Zisserman.
\newblock Very deep convolutional networks for large-scale image recognition.
\newblock In {\em 3rd International Conference on Learning Representations}, 2015.

\bibitem{cyclic}
Leslie~N Smith.
\newblock Cyclical learning rates for training neural networks.
\newblock In {\em 2017 IEEE winter conference on applications of computer vision}, pages 464--472, 2017.

\bibitem{smith2019super}
Leslie~N Smith and Nicholay Topin.
\newblock Super-convergence: Very fast training of neural networks using large learning rates.
\newblock In {\em Artificial intelligence and machine learning for multi-domain operations applications}, volume 11006, pages 369--386, 2019.

\bibitem{googlenet}
Christian Szegedy, Wei Liu, Yangqing Jia, Pierre Sermanet, Scott Reed, Dragomir Anguelov, Dumitru Erhan, Vincent Vanhoucke, and Andrew Rabinovich.
\newblock Going deeper with convolutions.
\newblock In {\em Proceedings of the IEEE conference on computer vision and pattern recognition}, pages 1--9, 2015.

\bibitem{tholeti2020tune}
Thulasi Tholeti and Sheetal Kalyani.
\newblock Tune smarter not harder: A principled approach to tuning learning rates for shallow nets.
\newblock {\em IEEE Transactions on Signal Processing}, 68:5063--5078, 2020.

\bibitem{yiming2025efficiently}
Daniel Yiming~Cao, August~Y Chen, Karthik Sridharan, and Benjamin Tang.
\newblock Efficiently escaping saddle points under generalized smoothness via self-bounding regularity.
\newblock {\em arXiv e-prints}, pages arXiv--2503, 2025.

\bibitem{wrn}
Sergey Zagoruyko and Nikos Komodakis.
\newblock Wide residual networks.
\newblock In {\em British Machine Vision Conference 2016}. British Machine Vision Association, 2016.

\bibitem{zamir2022restormer}
Syed~Waqas Zamir, Aditya Arora, Salman Khan, Munawar Hayat, Fahad~Shahbaz Khan, and Ming-Hsuan Yang.
\newblock Restormer: Efficient transformer for high-resolution image restoration.
\newblock In {\em Proceedings of the IEEE/CVF conference on computer vision and pattern recognition}, pages 5728--5739, 2022.

\bibitem{adadelta}
Matthew~D Zeiler.
\newblock Adadelta: an adaptive learning rate method.
\newblock {\em arXiv preprint arXiv:1212.5701}, 2012.

\bibitem{robust}
Jingfeng Zhang, Bo~Han, Laura Wynter, Bryan Kian~Hsiang Low, and Mohan Kankanhalli.
\newblock Towards robust resnet: a small step but a giant leap.
\newblock In {\em Proceedings of the 28th International Joint Conference on Artificial Intelligence}, pages 4285--4291, 2019.

\bibitem{zhang2017hitting}
Yuchen Zhang, Percy Liang, and Moses Charikar.
\newblock A hitting time analysis of stochastic gradient langevin dynamics.
\newblock In {\em Conference on Learning Theory}, pages 1980--2022. PMLR, 2017.

\bibitem{zhou2020towards}
Pan Zhou, Jiashi Feng, Chao Ma, Caiming Xiong, Steven Hoi, and E.~Weinan.
\newblock Towards theoretically understanding why sgd generalizes better than adam in deep learning.
\newblock In {\em Advances in Neural Information Processing Systems}, pages 16048--16059, 2020.

\end{thebibliography}

\newpage

\appendix
\noindent\rule{\textwidth}{1pt} % top rule
\vspace{0.3cm}

\begin{center}
    {\LARGE \textbf{Appendix}}
\end{center}

\vspace{0.3cm}
\noindent\rule{\textwidth}{1pt} % bottom rule
\section{Proof of Theorem \ref{th:1}} \label{ap:th1proof}
\begin{theorem} [Theorem \ref{th:1} restated]\label{th:1restated}
Under the assumptions \hyperlink{hyper:a1}{\textbf{A1}} and \hyperlink{hyper:a3}{\textbf{A3}} with $L_{max} < \frac{1}{\beta}$, for any point $\mathbf{x}_t$ with $\norm{\nabla f(\mathbf{x}_t)}\geq \sqrt{3\eta_c\beta\sigma^2}$ where $\sqrt{3\eta_c\beta\sigma^2}<\epsilon$ (satisfying \hyperlink{hyper:b1}{\textbf{B1}}), after one iteration we have,
\[
\mathbb{E}[f(\mathbf{x}_{t+1})] - f(\mathbf{x}_t) \leq - \tilde{\Omega}(L_{max}^2).
\]
\end{theorem}
\begin{proof}
Using the second order Taylor series approximation for $f(\textbf{x}_{t+1})$ around $\textbf{x}_{t}$, where $\mathbf{x}_{t+1}=\mathbf{x}_{t}-\eta_c\nabla f(\mathbf{x}_t) - \mathbf{w}_t$, we have 
\begin{equation*}
    \begin{aligned}
        f(\mathbf{x}_{t+1}) - f(\mathbf{x}_t) &\le \nabla f(\mathbf{x}_t)^{T} \left( \mathbf{x}_{t+1} - \mathbf{x}_t \right) + \frac{\beta}{2} \norm{\mathbf{x}_{t+1} - \mathbf{x}_t }^2,
    \end{aligned}
\end{equation*}
following the result from \citep[Lemma 1.2.3]{nesterov}. Taking expectation w.r.t. $\mathbf{w}_{t}$, 
\allowdisplaybreaks
\begin{equation}
    \begin{aligned}\label{afterp1}
        \mathbb{E}[f(\mathbf{x}_{t+1})] - f(\mathbf{x}_t) &\leq \nabla f(\mathbf{x}_t)^T \mathbb{E}[\mathbf{x}_{t+1} - \mathbf{x}_t] +\frac{\beta}{2} \mathbb{E}[\norm{\mathbf{x}_{t+1} - \mathbf{x}_t}^2] 
        \\&=\nabla f(\mathbf{x}_t)^T \mathbb{E}[-\eta_c \nabla f(\mathbf{x}_t) - \mathbf{w}_t ]  + \frac{\beta}{2} \mathbb{E}[\norm{-\eta_c \nabla f(\mathbf{x}_t) - \mathbf{w}_t }^2]  \\ 
        % &=- \eta_c \norm{\nabla f(\textbf{x}_0)}^2-\mathbb{E}[\nabla f(\textbf{x}_0)^T\textbf{w}_0] +\frac{\beta}{2} \mathbb{E}[(\eta_c \nabla f(\textbf{x}_0) + \textbf{w}_0)^T (\eta_c \nabla f(\textbf{x}_0) + \textbf{w}_0) ] \\
        &=  - \eta_c \norm{\nabla f(\mathbf{x}_t)}^2 +\frac{\beta}{2} \mathbb{E}[\eta_c^2 \norm{\nabla f(\mathbf{x}_t)}^2  + \norm{\mathbf{w}_t}^2] ,
        \end{aligned}
\end{equation}
since $\mathbb{E}[\mathbf{w}_t]=0$ due to the zero mean property in Lemma \ref{lemma:zeromean}. We focus on the last term in the next steps. Expanding $\norm{\mathbf{w}_t}^2$,
\begin{equation*}
    \begin{aligned}
        \norm{\mathbf{w}_t}^2 &=(\eta_c g(\mathbf{x}_t) - \eta_c\nabla f(\mathbf{x}_t) +  u_{t+1}g(\mathbf{x}_t))^T (\eta_c g(\mathbf{x}_t) - \eta_c\nabla f(\mathbf{x}_t) +  u_{t+1}g(\mathbf{x}_t))\\
        &= \eta_c^2 \norm{g(\mathbf{x}_t)}^2-\eta_c^2 g(\mathbf{x}_t)^T \nabla f(\mathbf{x}_t) + \eta_c u_{t+1} \norm{g(\mathbf{x}_t)}^2 -\eta_c^2 \nabla f(\mathbf{x}_t)^Tg(\mathbf{x}_t)+\eta_c^2 \norm{\nabla f(\mathbf{x}_t}^2\nonumber\\
        &-\eta_c u_{t+1} \nabla f(\mathbf{x}_t)^Tg(\mathbf{x}_t) + \eta_c u_{t+1} \norm{g(\mathbf{x}_t)}^2 -\eta_c u_{t+1} g(\mathbf{x}_t)^T \nabla f(\mathbf{x}_t) + u_{t+1}^2 \norm{g(\mathbf{x}_t)}^2.
    \end{aligned}
\end{equation*}
Taking expectation with respect to $\mathbf{x}_t$ and noting that $\mathbb{E}[u_{t+1}]=0$ and $\mathbb{E}[g(\mathbf{x}_t)]=\nabla f(\mathbf{x}_t)$,\footnote{Note that there are two random variables in $\mathbf{w}_t$ which are the stochastic gradient $g(\mathbf{x}_t)$ and the uniformly distributed LR $u_{t+1}$ due to our proposed LR scheduler. Hence, the expectation is with respect to both these variables. Also note that $u_{t+1}$ and $g(\mathbf{x}_t)$ are independent of each other.}
\begin{equation}
    % \begin{align}
         \mathbb{E}[\norm{\mathbf{w}_t}^2] =\eta_c^2\mathbb{E}[ \norm{g(\mathbf{x}_t)}^2]-\eta_c^2\norm{\nabla f(\mathbf{x}_t)}^2+\mathbb{E}[ u_{t+1}^2]\mathbb{E}[ \norm{g(\mathbf{x}_t)}^2 ].\label{app3_1}
    % \end{align}
\end{equation}
Now, as per assumption \hyperlink{hyper:a3}{\textbf{A3}},
\begin{gather}
        \norm{g(\mathbf{x}_t)-\nabla f(\mathbf{x}_t)}^2 \leq Q^2\nonumber\\
      \norm{g(\mathbf{x}_t)}^2+\norm{\nabla f(\mathbf{x}_t)}^2-2g(\mathbf{x}_t)^T\nabla f(\mathbf{x}_t)\leq Q^2\nonumber\\
        \norm{g(\mathbf{x}_t)}^2\leq Q^2-\norm{\nabla f(\mathbf{x}_t)}^2+2g(\mathbf{x}_t)^T\nabla f(\mathbf{x}_t)\nonumber\\
        \mathbb{E}[\norm{g(\mathbf{x}_t)}^2]\leq \mathbb{E}[Q^2]-\norm{\nabla f(\mathbf{x}_t)}^2+2\norm{\nabla f(\mathbf{x}_t)}^2\leq\sigma^2+\norm{\nabla f(\mathbf{x}_t)}^2\label{app3_2},
\end{gather}
as $\mathbb{E}[Q^2]\leq\sigma^2$.
Applying \eqref{app3_2} to \eqref{app3_1},
\begin{equation}
    \begin{aligned}
         \mathbb{E}[\norm{\mathbf{w}_t}^2] &\leq \eta_c^2 \sigma^2+\eta_c^2\norm{\nabla f(\mathbf{x}_t)}^2-\eta_c^2\norm{\nabla f(\mathbf{x}_t)}^2 +\mathbb{E}[ u_{t+1}^2]\sigma^2+\mathbb{E}[ u_{t+1}^2]\norm{\nabla f(\mathbf{x}_t)}^2\\
         &=\eta_c^2 \sigma^2+\mathbb{E}[ u_{t+1}^2]\sigma^2+\mathbb{E}[ u_{t+1}^2]\norm{\nabla f(\mathbf{x}_t)}^2\\
         &=\eta_c^2 \sigma^2+\frac{(L_{max}-L_{min})^2\sigma^2}{12}+\frac{(L_{max}-L_{min})^2\norm{\nabla f(\textbf{x}_0)}^2}{12},\label{app3_3}
    \end{aligned}
\end{equation}
since the second moment of a uniformly distributed random variable in the interval $[L_{min}-\eta_c, L_{max}-\eta_c]$ is given by $\frac{(L_{max}-L_{min})^2}{12}$.
Using \eqref{app3_3} in \eqref{afterp1} and $\eta_c=\frac{L_{min}+L_{max}}{2}$,
\begin{equation*}
    \begin{aligned}
        \mathbb{E}[f(\mathbf{x}_{t+1})] - f(\mathbf{x}_t) &\leq - \eta_c \norm{\nabla f(\mathbf{x}_t)}^2 +\frac{\beta}{2} \eta_c^2 \norm{\nabla f(\mathbf{x}_t)}^2 +\frac{\beta\eta_c^2 \sigma^2}{2}+\frac{\beta(L_{max}-L_{min})^2\sigma^2}{24}\\
        &+\frac{\beta(L_{max}-L_{min})^2\norm{\nabla f(\mathbf{x}_t)}^2}{24} \\
        &\leq - \eta_c \norm{\nabla f(\mathbf{x}_t)}^2 +\frac{\beta}{2} \eta_c^2 \norm{\nabla f(\mathbf{x}_t)}^2 +\frac{\beta\eta_c^2 \sigma^2}{2}+\frac{\beta\eta_c^2\sigma^2}{6}+\frac{\beta\eta_c^2\norm{\nabla f(\textbf{x}_0)}^2}{6}\\
        &=-\norm{\nabla f(\mathbf{x}_t)}^2\left(\eta_c - \frac{2\beta\eta_c^2}{3}\right)+\frac{2\beta\eta_c^2 \sigma^2}{3}
        \end{aligned}
        \end{equation*}
Now, applying our initial assumption that $\norm{\nabla f(\mathbf{x}_t)}\geq \sqrt{3\eta_c\beta\sigma^2}$, we have,
        \begin{equation*}
        \mathbb{E}[f(\mathbf{x}_{t+1})] - f(\mathbf{x}_t) \leq -3\eta_c\beta\sigma^2\left(\eta_c - \frac{2\beta\eta_c^2}{3}\right)+\frac{2\beta\eta_c^2 \sigma^2}{3}=-3\eta_c^2\beta\sigma^2+\frac{6\beta^2\eta_c^3\sigma^2}{3}+\frac{2\beta\eta_c^2 \sigma^2}{3}\nonumber
        \end{equation*}
Since $L_{max}<\frac{1}{\beta}$ and $\eta_c=\frac{L_{min}+L_{max}}{2}$, we have $\eta_c\beta<L_{max}\beta<1$. Finally,
 \begin{equation*}
    \begin{aligned}       
        \mathbb{E}[f(\mathbf{x}_{t+1})] - f(\mathbf{x}_t)&\leq -3\eta_c^2\beta\sigma^2+\frac{6\beta\eta_c^2\sigma^2}{3}+\frac{2\beta\eta_c^2 \sigma^2}{3} =-\frac{\beta\eta_c^2 \sigma^2}{3}\nonumber\\
        &=-\tilde{\Omega}(\eta_c^2),
    \end{aligned}
\end{equation*}
which proves the theorem.
\end{proof}
\section{Additional results needed to prove Theorem \ref{th:2}}
Here, we state and prove two lemmas that are instrumental in the proof of Theorem \ref{th:2}.
\subsection{Proof of Lemma \ref{lemma:xt_x0}}
\label{ap:lemma1}
In the following Lemma, we prove that the gradients of a second order approximation of $f$ are probabilistically bounded for all $t\leq T$ and its iterates as we apply SGD-PLRS are also bounded when the initial iterate $\textbf{x}_0$ is a saddle point.
\begin{lemma}\label{lemma:xt_x0}
Let $f$ satisfy Assumptions \hyperlink{hyper:a1}{\textbf{A1}} - \hyperlink{hyper:a4}{\textbf{A4}}. Let $\tilde{f}$ be the second order Taylor approximation of $f$ and let $\tilde{\textbf{x}}_t$ be the  iterate at time step $t$ obtained using the SGD update equation as in \eqref{update} on $\tilde{f}$; let $\tilde{\textbf{x}}_0=\textbf{x}_0$, $\norm{\nabla f(\textbf{x}_0)}\leq \epsilon$ and the minimum eigenvalue of the Hessian of $f$ at $\textbf{x}_0$ be $ \lambda_{min}(\textbf{H}(\textbf{x}_0)) = -\gamma_{o}$ where $\gamma_o>0$. With probability at least $1-\tilde{O}(L_{max}^{15/4})$, we have
\[  \norm{\nabla \tilde{f}(\tilde{\textbf{x}}_t)}\leq \tilde{O}\left(\frac{1}{L_{max}^{0.5}}\right), \quad \norm{\tilde{\textbf{x}}_t - \textbf{x}_0}\leq \tilde{O}\left(L_{max}^{3/8}\log\left(\frac{1}{L_{max}}\right)\right)\quad \forall t\leq T=\tilde{O}\left(L_{max}^{-1/4}\right).\]
\end{lemma}
\begin{proof}  
As $\tilde{f}$ is the second order Taylor series approximation of $f$, we have
\begin{equation*}
\tilde{f}(\tilde{\textbf{x}})=f(\textbf{x}_0)+\nabla f(\textbf{x}_0)^T (\tilde{\textbf{x}}-\textbf{x}_0) + \frac{1}{2}(\tilde{\textbf{x}}-\textbf{x}_0)^T \textbf{H}(\textbf{x}_0) (\tilde{\textbf{x}}-\textbf{x}_0).
\end{equation*}
Taking derivative w.r.t. $\tilde{\textbf{x}}$, we have
% \begin{equation*}
    $\nabla \tilde{f}(\tilde{\textbf{x}}) = \nabla f(\textbf{x}_0)+\textbf{H}(\textbf{x}_0) (\tilde{\textbf{x}}-\textbf{x}_0).$
% \end{equation*}
Now, note that $\nabla \tilde{f}(\tilde{\textbf{x}}_{t-1})=\nabla f(\textbf{x}_0) + \textbf{H}(\textbf{x}_0)(\tilde{\textbf{x}}_{t-1} - \textbf{x}_0) = K(\textbf{x}_0) + \textbf{H}(\textbf{x}_0)\tilde{\textbf{x}}_{t-1}$, where $K(\textbf{x}_0) = \nabla f(\textbf{x}_0) - \textbf{H}(\textbf{x}_0)\textbf{x}_{0} = \nabla \tilde{f}(\tilde{\textbf{x}}_{t-1}) - \textbf{H}(\textbf{x}_0)\tilde{\textbf{x}}_{t-1}$. Therefore, 
\begin{equation}\label{quadratic}
    \begin{aligned}
        \nabla \tilde{f}(\tilde{\textbf{x}}_t) = K(\textbf{x}_0) + \textbf{H}(\textbf{x}_0) \tilde{\textbf{x}}_t &= \nabla \tilde{f}(\tilde{\textbf{x}}_{t-1}) - \textbf{H}(\textbf{x}_0)\tilde{\textbf{x}}_{t-1} + \textbf{H}(\textbf{x}_0) \tilde{\textbf{x}}_t  \\
        &= \nabla\tilde{f}(\tilde{\textbf{x}}_{t-1}) + \textbf{H}(\textbf{x}_0)(\tilde{\textbf{x}}_t - \tilde{\textbf{x}}_{t-1}).
    \end{aligned}
\end{equation}
% Now we obtain an expression for the term $\nabla\tilde{f}(\tilde{\textbf{x}}_t)$, which is to be bounded in order to prove the lemma. 

\noindent Next, using the SGD-PLRS update and rearranging,
\allowdisplaybreaks
\begin{equation} \label{tildefeq}
    \begin{aligned}
    \nabla\tilde{f}(\tilde{\textbf{x}}_t)  
    &= \nabla\tilde{f}(\tilde{\textbf{x}}_{t-1}) - \textbf{H}(\textbf{x}_0)(\eta_c\nabla\tilde{f}(\tilde{\textbf{x}}_{t-1}) + \tilde{\textbf{w}}_{t-1}) \\
    &=(I-\eta_c\textbf{H}(\textbf{x}_0))\nabla\tilde{f}(\tilde{\textbf{x}}_{t-1}) -\textbf{H}(\textbf{x}_0)\tilde{\textbf{w}}_{t-1},  \\
     % &\stackrel{(b)}{=} (1-\eta_cH(x_0))(\nabla\tilde{f}(\tilde{x}_{t-2})-H(x_0)\eta_c\nabla\tilde{f}(\tilde{x}_{t-2}) - H(x_0)\textbf{w}_{t-2})-H(x_0)\textbf{w}_{t-1} \\
     % &= (1-\eta_cH(x_0))^2 \nabla\tilde{f}(\tilde{x}_{t-2}) -(1-\eta_cH(x_0))H(x_0)\textbf{w}_{t-2} -H(x_0)\textbf{w}_{t-1} \\
    % &= (1-\eta_cH(x_0))^3(\nabla\tilde{f}(\tilde{x}_{t-3})-(1-\eta_cH(x_0))^2H\textbf{w}_{t-3}- (1-\eta_cH(x_0))H(x_0)\textbf{w}_{t-2} \\
    % &\hspace{28em} - H(x_0)\textbf{w}_{t-1}
\end{aligned}
\end{equation}
where $I$ denotes the $d\times d$ identity matrix. Next, unrolling the term $\nabla \tilde{f}(\tilde{\textbf{x}}_{t-1})$ recursively,
\begin{equation}
    \begin{aligned} \label{nablaf1}
    \nabla\tilde{f}(\tilde{\textbf{x}}_t) = (I-\eta_c\textbf{H}(\textbf{x}_0))^t\nabla\tilde{f}   (\tilde{\textbf{x}}_{0}) - \textbf{H}(\textbf{x}_0)\sum_{\tau=0}^{t-1}(I-\eta_c\textbf{H}(\textbf{x}_0))^{t-\tau-1}\tilde{\textbf{w}}_{\tau}.
\end{aligned}
\end{equation}
% Note that unlike additive noise\citep{ge2015escaping}, we don't have the multiplication of $\eta_c$ with the second term of \eqref{nablaf1} as $\tilde{\textbf{w}}$ for us is $\eta_c\tilde{g}(\tilde{\textbf{x}}_{t-1})-\eta_c\nabla \tilde{f}(\tilde{\textbf{x}}_{t-1})+u_t\tilde{g}(\tilde{\textbf{x}}_{t-1})$ while it is $\eta\tilde{g}(\tilde{\textbf{x}}_{t-1})-\eta\nabla \tilde{f}(\tilde{\textbf{x}}_{t-1})+n$ where $n$ is a random vector on a unit sphere.
Using the triangle and Cauchy-Schwartz inequalities,
\begin{equation}
    \begin{aligned} 
    \norm{\nabla\tilde{f}(\tilde{\textbf{x}}_t)} &\leq \norm{(I-\eta_c\textbf{H}(\textbf{x}_0))^t\nabla\tilde{f}   (\tilde{\textbf{x}}_{0}) }+ \norm{\textbf{H}(\textbf{x}_0)\sum_{\tau=0}^{t-1}(I-\eta_c\textbf{H}(\textbf{x}_0))^{t-\tau-1}\tilde{\textbf{w}}_{\tau}}  \\
    &\leq \norm{(I-\eta_c\textbf{H}(\textbf{x}_0))^t}\; \norm{\nabla\tilde{f}(\tilde{\textbf{x}}_{0})}+\norm{\textbf{H}(\textbf{x}_0)}\; \norm{\sum_{\tau=0}^{t-1}(I-\eta_c\textbf{H}(\textbf{x}_0))^{t-\tau-1}\tilde{\textbf{w}}_{\tau}} \label{nabls}
\end{aligned}
\end{equation}
Note that the norm over the matrices refers to the matrix-induced norm. Since $\textbf{H}(\textbf{x}_0)$ is a real symmetric matrix, the induced norm gives the maximum eigenvalue of $\textbf{H}(\textbf{x}_0)$ i.e, $\lambda_{max}(\textbf{H}(\textbf{x}_0))\leq\beta$ by our $\beta$-smoothness assumption \hyperlink{hyper:a1}{\textbf{A1}}. 
In the case of $(I-\eta_c\textbf{H}(\textbf{x}_0))$ the induced norm gives $(1-\eta_c\lambda_{min}(\textbf{H}(\textbf{x}_0))$ which is $(1+\eta_c\gamma_o)$ as per our assumption that $\lambda_{min}(\textbf{H}(\textbf{x}_0))=-\gamma_o$. 
Also recall that $\norm{\nabla \tilde{f}(\tilde{\textbf{x}}_0)}\leq\epsilon$. Now \eqref{nabls} becomes,
\begin{equation}
    \begin{aligned}
    \norm{\nabla\tilde{f}(\tilde{\textbf{x}}_t)} &\leq (1+\eta_c\gamma_o)^t \epsilon +\beta\;\norm{\sum_{\tau=0}^{t-1}(I-\eta_c\textbf{H}(\textbf{x}_0))^{t-\tau-1}\tilde{\textbf{w}}_{\tau}},  \\
    % &\leq (1+\eta_c\gamma_o)^t \epsilon + \beta \sum_{\tau=0}^{t-1} \norm{(I-\eta_c\textbf{H}(\textbf{x}_0))^{t-\tau-1}} \norm{\tilde{\textbf{w}}_{\tau}} \\
    &\leq (1+\eta_c\gamma_o)^t \epsilon + \beta \sum_{\tau=0}^{t-1} (1+\eta_c\gamma_o)^{t-\tau-1} \norm{\tilde{\textbf{w}}_{\tau}}.
\end{aligned}
\end{equation}
Now, expanding the noise term $\tilde{\textbf{w}}_{\tau}$, 
\begin{equation*}
    \norm{\nabla\tilde{f}(\tilde{\textbf{x}}_t)} =(1+\eta_c\gamma_o)^t \epsilon + \beta \sum_{\tau=0}^{t-1} (1+\eta_c\gamma_o)^{t-\tau-1} \norm{\eta_c \tilde{g}(\tilde{\textbf{x}}_{\tau}) - \eta_c\nabla \tilde{f}(\tilde{\textbf{x}}_{\tau}) +  u_{\tau+1} \tilde{g}(\tilde{\textbf{x}}_{\tau})} \end{equation*}   
Now recall from our assumption \hyperlink{hyper:a3}{\textbf{A3}} that $\norm{ \tilde{g}(\tilde{\textbf{x}}_{\tau}) - \nabla \tilde{f}(\tilde{\textbf{x}}_{\tau})}\leq \tilde{Q}$. Hence, 
\begin{equation*}
    \begin{aligned} 
    \norm{\nabla\tilde{f}(\tilde{\textbf{x}}_t)}&\leq (1+\eta_c\gamma_o)^t \epsilon + \beta \sum_{\tau=0}^{t-1} (1+\eta_c\gamma_o)^{t-\tau-1} \left(\eta_c \tilde{Q} +|u_{\tau+1}| \norm{\tilde{g}(\tilde{\textbf{x}}_{\tau})-\nabla\tilde{f}(\tilde{\textbf{x}}_{\tau})+\nabla\tilde{f}(\tilde{\textbf{x}}_{\tau})} \right)\\
    &\leq (1+\eta_c\gamma_o)^t \epsilon + \beta \sum_{\tau=0}^{t-1} (1+\eta_c\gamma_o)^{t-\tau-1} \left(\eta_c \tilde{Q} +|u_{\tau+1}| \left(\tilde{Q} + \norm{\nabla\tilde{f}(\tilde{\textbf{x}}_{\tau})}\right)\right) 
    % &= (1+\eta_c\gamma_o)^t \epsilon + \beta \sum_{\tau=0}^{t-1} (1+\eta_c\gamma_o)^{t-\tau-1} \tilde{Q}\left(\eta_c + |u_{\tau+1}|\right) + \beta \sum_{\tau=0}^{t-1} (1+\eta_c\gamma_o)^{t-\tau-1} |u_{\tau+1}| \norm{\nabla\tilde{f}(\tilde{\textbf{x}}_{\tau})}
\end{aligned}
\end{equation*}
% \comm{fix overfull}
Using $\norm{\nabla\tilde{f}(\tilde{\textbf{x}}_0)}\leq \epsilon$ and $\norm{\nabla\tilde{f}(\tilde{\textbf{x}}_1)}\leq (1+\eta_c \gamma_o)\epsilon + \epsilon  +2\tilde{Q}$, it can be proved by induction that the general expression for $t\geq 2$ is given by,
\begin{equation}\label{generalt}
    \norm{\nabla \tilde{f}(\tilde{\textbf{x}}_t)} \leq 10\tilde{Q} \sum_{\tau=0}^{\frac{t(t-1)}{2}} (1+\eta_c\gamma_o)^{\tau}
\end{equation}
We give the proof of \eqref{generalt} by induction in Appendix \ref{ap:induction}.
Next, we prove the bound on $\tilde{\textbf{x}}_t - \tilde{\textbf{x}}_0$. Using the SGD-PLRS update,  
\begin{subequations}
    \begin{align}
    &\tilde{\textbf{x}}_t - \tilde{\textbf{x}}_0 = -\sum_{\tau=0}^{t-1}\left(\eta_c\nabla\tilde{f}(\tilde{\textbf{x}}_{\tau})+\tilde{\textbf{w}}_{\tau}\right) \nonumber \\
    &= - \sum_{\tau=0}^{t-1} \left(\eta_c \left( (I-\eta_c\textbf{H}(\textbf{x}_0))^{\tau}\nabla\tilde{f}   (\tilde{\textbf{x}}_{0}) - \textbf{H}(\textbf{x}_0)\sum_{\tau^{'}=0}^{\tau-1}(I-\eta_c\textbf{H}(\textbf{x}_0))^{\tau-\tau^{'}-1}\tilde{\textbf{w}}_{\tau^{'}} \right) +\tilde{\textbf{w}}_{\tau} \right) \label{delx1}\\
    % &=  -\sum_{\tau=0}^{t-1}\eta_c(I-\eta_c\textbf{H}(\textbf{x}_0))^{\tau}\nabla f(\textbf{x}_{0})- \sum_{\tau=0}^{t-2}(I-\eta_c\textbf{H}(\textbf{x}_0))^{t-\tau-1}\tilde{\textbf{w}}_{\tau} + \sum_{\tau=0}^{t-2}\tilde{\textbf{w}}_{\tau} -\sum_{\tau=0}^{t-1} \tilde{\textbf{w}}_{\tau} , \label{delx} \\
    % &=-\sum_{\tau=0}^{t-1}\eta_c(I-\eta_c\textbf{H}(\textbf{x}_0))^{\tau}\nabla f(\textbf{x}_{0}) - \sum_{\tau=0}^{t-2}(I-\eta_c\textbf{H}(\textbf{x}_0))^{t-\tau-1}\tilde{\textbf{w}}_{\tau} - \tilde{\textbf{w}}_{t-1}\\
    &=-\sum_{\tau=0}^{t-1}\eta_c(I-\eta_c\textbf{H}(\textbf{x}_0))^{\tau}\nabla f(\textbf{x}_{0}) - \sum_{\tau=0}^{t-1}(I-\eta_c\textbf{H}(\textbf{x}_0))^{t-\tau-1}\tilde{\textbf{w}}_{\tau},\label{xt_x0}
\end{align}
\end{subequations}
where the equation \eqref{delx1} is obtained by using \eqref{nablaf1}. We obtain \eqref{xt_x0} by using the summation of geometric series as $\textbf{H}(\textbf{x}_0)$ is invertible by the strict saddle property. 
% Its derivation is given in Appendix \ref{ap:geometric_sum}. 
As $\tilde{\textbf{x}}_0 = \textbf{x}_0$, we can write $\nabla\tilde{f}(\tilde{\textbf{x}}_0)=\nabla f(\textbf{x}_0)$. 
%Note that unlike in \cite{ge2015escaping}, we have additional summation terms due to multiplicative noise. 
Taking norm,
% \vspace{-0.4cm}
\begin{equation}
\begin{aligned}
    \norm{\tilde{\textbf{x}_t} - \tilde{\textbf{x}}_0}&\leq \norm{\sum_{\tau=0}^{t-1}\eta_c(I-\eta_c\textbf{H}(\textbf{x}_0))^{\tau}\nabla f(\textbf{x}_{0})} + \norm{\sum_{\tau=0}^{t-1}(I-\eta_c\textbf{H}(\textbf{x}_0))^{t-\tau-1}\tilde{\textbf{w}}_{\tau}} \\
    &\leq \sum_{\tau=0}^{t-1}\norm{\eta_c(I-\eta_c\textbf{H}(\textbf{x}_0))^{\tau}\nabla f(\textbf{x}_{0})} +\sum_{\tau=0}^{t-1}\norm{(I-\eta_c\textbf{H}(\textbf{x}_0))^{t-\tau-1}\tilde{\textbf{w}}_{\tau}}   \\
    % &\leq \sum_{\tau=0}^{t-1} \norm{\eta_c(I-\eta_c\textbf{H}(\textbf{x}_0))^{\tau}}.\norm{\nabla f(\textbf{x}_{0})} +\sum_{\tau=0}^{t-1}\norm{(I-\eta_c\textbf{H}(\textbf{x}_0))^{t-\tau-1}}\norm{\tilde{\textbf{w}}_{\tau}}\\
    &\leq \eta_c \epsilon \sum_{\tau=0}^{t-1} (1+\eta_c\gamma_o)^{\tau}+\sum_{\tau=0}^{t-1}(1+\eta_c\gamma_o)^{t-\tau-1}\norm{\tilde{\textbf{w}}_{\tau}} \label{xteq1}.
\end{aligned}
\end{equation}
% \vspace{-0.1cm}
In \eqref{xteq1}, it can be seen that the first term is arbitrarily small by the initial assumption and that the second term decides the order of $\norm{\tilde{\textbf{x}}_t -\tilde{\textbf{x}}_0}$. Hence, in order to bound $\norm{\tilde{\textbf{x}_t} - \tilde{\textbf{x}}_0}$ probabilistically, it is sufficient to bound the second term, $\sum_{\tau=0}^{t-1}(1+\eta_c\gamma_o)^{t-\tau-1}\norm{\tilde{\textbf{w}}_{\tau}}$.
Now,
\begin{equation*}
    \begin{aligned}
       \sum_{\tau=0}^{t-1}(1+\eta_c\gamma_o)^{t-\tau-1}\norm{\tilde{\textbf{w}}_{\tau}}&=\sum_{\tau=0}^{t-1} (1+\eta_c\gamma_o)^{t-\tau-1}\norm{\eta_c\tilde{g}(\tilde{\textbf{x}}_{\tau})-\eta_c\nabla \tilde{f}(\tilde{\textbf{x}}_{\tau})+u_{\tau+1}\tilde{g}(\tilde{\textbf{x}}_{\tau})}\\
       % &\leq \sum_{\tau=0}^{t-1}(1+\eta_c\gamma_o)^{t-\tau-1}\left( \eta_c \tilde{Q} +\abs{u_{\tau+1}}\norm{\tilde{g}(\tilde{\textbf{x}}_{\tau})}  \right) \\
       &\hspace{-30mm}=\sum_{\tau=0}^{t-1}(1+\eta_c\gamma_o)^{t-\tau-1}\left( \eta_c \tilde{Q} +\abs{u_{\tau+1}}\norm{ \tilde{g}(\tilde{\textbf{x}}_{\tau}) - \nabla \tilde{f}(\tilde{\textbf{x}}_{\tau}) +\nabla \tilde{f}(\tilde{\textbf{x}}_{\tau}) }   \right)\\
       % &\leq \sum_{\tau=0}^{t-1}(1+\eta_c\gamma_o)^{t-\tau-1}\left( \eta_c \tilde{Q}+ \abs{u_{\tau+1}}\left( \tilde{Q} +\norm{\nabla \tilde{f}(\tilde{\textbf{x}}_{\tau})} \right) \right)\\
       &\hspace{-30mm}=\sum_{\tau=0}^{t-1} (1+\eta_c\gamma_o)^{t-\tau-1}\tilde{Q}\left(\eta_c +\abs{u_{\tau+1}} \right)+\sum_{\tau=0}^{t-1}(1+\eta_c\gamma_o)^{t-\tau-1} \abs{u_{\tau+1}} \norm{\nabla \tilde{f}(\tilde{\textbf{x}}_{\tau})} 
    \end{aligned}
\end{equation*}
Now, using $\norm{\nabla\tilde{f}(\tilde{\textbf{x}}_0)}\leq \epsilon$, $\norm{\nabla\tilde{f}(\tilde{\textbf{x}}_1)}\leq (1+\eta_c \gamma_o)\epsilon + \epsilon  +2\tilde{Q}$ and \eqref{generalt} we write,
\begin{equation}
    \begin{aligned}
      &\sum_{\tau=0}^{t-1}(1+\eta_c\gamma_o)^{t-\tau-1}\norm{\tilde{\textbf{w}}_{\tau}} \leq \sum_{\tau=0}^{t-1}(1+\eta_c\gamma_o)^{t-\tau-1} \tilde{Q}\left(\eta_c +\abs{u_{\tau+1}} \right)+ (1+\eta_c\gamma_o)^{t-1}\abs{u_1}\epsilon+\\
       &(1+\eta_c\gamma_o)^{t-2}\abs{u_2}\left((1+\eta_c \gamma_o)\epsilon + \epsilon  +2\tilde{Q}\right)+\sum_{\tau=2}^{t-1}(1+\eta_c\gamma_o)^{t-\tau-1}\abs{u_{\tau+1}} 10\tilde{Q} \sum_{\tau^{'}=0}^{\frac{\tau(\tau-1)}{2}} (1+\eta_c\gamma_o)^{\tau^{'}}\label{wtau}
    \end{aligned}
\end{equation}
It can be observed from \eqref{wtau} that the last term dominates the expression of and hence, it determines the order of $\norm{\tilde{\textbf{x}}_t - \tilde{\textbf{x}}_0}$. We now apply Hoeffding's inequality to derive a probabilistic bound on $\norm{\tilde{\textbf{x}_t} - \tilde{\textbf{x}}_0}$. According to Hoeffding's inequality for any summation $S_n=X_1+\cdots+X_n$ such that $a_i\leq X_i\leq b_i$,
   $ \mathbb{P}\left(S_n - \mathbb{E}[S_n] \geq  \delta  \right)\leq \exp\left(\frac{-2\delta^2}{\sum_{i=1}^n (b_i - a_i)^2}\right) $.
Now, setting $T=\tilde{O}\left(L_{max}^{-1/4}\right)$ from \eqref{theorem2_4} and assuming $\eta_c\leq\eta_{max}\leq \frac{\sqrt{2}-1}{\gamma^{'}}, \gamma_o\leq \gamma^{'}$, the squared bound of the summation $\sum_{\tau=2}^{t-1}(1+\eta_c\gamma_o)^{t-\tau-1}\abs{u_{\tau+1}} 10\tilde{Q} \sum_{\tau^{'}=0}^{\frac{\tau(\tau-1)}{2}} (1+\eta_c\gamma_o)^{\tau^{'}}\leq \tilde{O}\left(L_{max}^{3/4}\right)$,
% \begin{subequations}
%     \begin{align}
%         \sum_{\tau=2}^{t-1}\left((1+\eta_c\gamma_o)^{t-\tau-1}\abs{u_{\tau+1}} 10\tilde{Q} \sum_{\tau^{'}=0}^{\frac{\tau(\tau-1)}{2}} (1+\eta_c\gamma_o)^{\tau^{'}}\right)^2&\leq (t-2)(1+\eta_c\gamma_o)^{t-1}\left(L_{max}10\tilde{Q}\sum_{\tau=0}^{\frac{t(t-1)}{2}} (1+\eta_c\gamma_o)^{\tau}\right)^2\\
%         &\leq \tilde{O}\left(L_{max}^{-1/4}\right)\tilde{O}(1)\left(L_{max}\tilde{O}\left(\frac{1}{L_{max}^{1/2}}\right)\tilde{O}(1)\right)^2\\
%         &=\tilde{O}\left(L_{max}^{3/4}\right)
%     \end{align}
% \end{subequations}
Setting $\delta = \tilde{O}\left(\sqrt{L_{max}^{3/4}}\log \left(\frac{1}{L_{max}}\right)\right)$, for some $t \leq T$, 
\begin{equation*}
    \begin{aligned}
        \mathbb{P}\left(\sum_{\tau=2}^{t-1}(1+\eta_c\gamma_o)^{t-\tau-1}\abs{u_{\tau+1}} 10\tilde{Q} \sum_{\tau^{'}=0}^{\frac{\tau(\tau-1)}{2}} (1+\eta_c\gamma_o)^{\tau^{'}} \geq  \tilde{O}\left(L_{max}^{3/8} \log \left(\frac{1}{L_{max}}\right)\right)  \right)&\\
        &\hspace{-45mm}\leq \tilde{O}(L^4_{max}).
    \end{aligned}
\end{equation*}
Taking the union bound over all $t\leq T$,
\begin{equation*}
\begin{aligned}
    \mathbb{P}\left(\forall t\leq T,\quad\sum_{\tau=2}^{t-1}(1+\eta_c\gamma_o)^{t-\tau-1}\abs{u_{\tau+1}} 10\tilde{Q} \sum_{\tau^{'}=0}^{\frac{\tau(\tau-1)}{2}} (1+\eta_c\gamma_o)^{\tau^{'}} \geq  \tilde{O}\left(L_{max}^{3/8} \log \left(\frac{1}{L_{max}}\right)\right)  \right)&\\
    &\hspace{-45mm}\leq \tilde{O}\left( L_{max}^{15/4} \right),
    \end{aligned}
\end{equation*}
which completes our proof.
\end{proof}
\subsection{Proof of Lemma \ref{lemma:xt_tildext}}
\label{ap:lemma2}
This lemma is used to derive an expression for a high probability upper bound of $\norm{\textbf{x}_t - \tilde{\textbf{x}}_t}$ and $\norm{\nabla f(\textbf{x}_t) - \nabla \tilde{f}(\tilde{\textbf{x}}_t)}$.
\begin{lemma}\label{lemma:xt_tildext}
    Let $f:\mathbb{R}^d\rightarrow\mathbb{R}$ satisfy Assumptions \hyperlink{hyper:a1}{\textbf{A1}} - \hyperlink{hyper:a4}{\textbf{A4}}. Let $\tilde{f}$ be the second order Taylor's approximation of $f$ and let $\textbf{x}_t$, $\tilde{\textbf{x}}_t$ be the iterates at time step $t$ obtained using the SGD-PLRS update on $f$, $\tilde{f}$ respectively; let $\tilde{\textbf{x}}_0=\textbf{x}_0$ and $\norm{\nabla f(\textbf{x}_0)}\leq \epsilon$. Let the minimum eigenvalue of the Hessian at $\textbf{x}_0$ be $\lambda_{min}(\nabla^2(f(\textbf{x}_0))) = -\gamma_o$, where $\gamma_o>0$.
Then $\forall t\leq T=O\left(L_{max}^{-1/4}\right)$, with a probability of at least $1-\tilde{O}(L_{max}^{7/2})$,
\begin{equation*}
\begin{aligned}
\norm{\textbf{x}_t - \tilde{\textbf{x}}_t} &\leq O\left( L_{max}^{3/4}\right) \quad \text{ and } \quad
\norm{\nabla f(\textbf{x}_t) - \nabla \tilde{f}(\tilde{\textbf{x}}_t)} &\leq O\left( L_{max}^{3/8}\log\frac{1}{L_{max}}\right).
\end{aligned}
\end{equation*}
\end{lemma}
\begin{proof} 
The expression for $\textbf{x}_t - \tilde{\textbf{x}}_t $ can be written as,
\begin{equation}\label{xt}
    \begin{aligned}
\textbf{x}_t - \tilde{\textbf{x}}_t &= (\textbf{x}_t -\textbf{x}_0) - (\tilde{\textbf{x}}_t -\textbf{x}_0) \\
&= - \sum_{\tau=0}^{t-1} \big(\eta_c \nabla f(\textbf{x}_{\tau})+\textbf{w}_{\tau}\big) - \left(- \sum_{\tau=0}^{t-1} \big(\eta_c \nabla \tilde{f}(\tilde{\textbf{x}}_{\tau})+\tilde{\textbf{w}}_{\tau}\big)\right) = - \sum_{\tau=0}^{t-1} \left(\eta_c\Delta_{\tau}+( \textbf{w}_{\tau}-\tilde{\textbf{w}}_{\tau} )  \right).
    \end{aligned}
\end{equation}
where we define $\Delta_t = \nabla f(\textbf{x}_t) - \nabla \tilde{f}(\tilde{\textbf{x}}_t)$.
Now in order to bound $\norm{\textbf{x}_t-\tilde{\textbf{x}}_t}$, we derive expressions for both $\textbf{w}_{\tau}-\tilde{\textbf{w}}_{\tau}$ and $\Delta_{\tau}$. We initially focus on the term $\textbf{w}_{\tau}-\tilde{\textbf{w}}_{\tau}$.
\begin{equation}
    \begin{aligned}
  \textbf{w}_{\tau}-\tilde{\textbf{w}}_{\tau} &=\eta_c g(\textbf{x}_{\tau}) - \eta_c\nabla f_{\tau} +u_{\tau+1}g(\textbf{x}_{\tau}) -\left(\eta_c\tilde{g}(\tilde{\textbf{x}}_{\tau})-\eta_c\nabla \tilde{f}(\tilde{\textbf{x}}_{\tau})+u_{\tau+1}\tilde{g}(\tilde{\textbf{x}}_{\tau})\right)      \\
  % &=\left( u_{\tau+1}+\eta_c \right)\left(g(\textbf{x}_{\tau}) - \tilde{g}(\tilde{\textbf{x}}_{\tau}) \right)-\eta_c\Delta_{\tau} \\
  % &=\left( u_{\tau+1}+\eta_c \right)\left(g(\textbf{x}_{\tau}) -\nabla f(\textbf{x}_{\tau})+\nabla f(\textbf{x}_{\tau}) - \tilde{g}(\tilde{\textbf{x}}_{\tau}) +\nabla \tilde{f}(\tilde{\textbf{x}}_{\tau})-\nabla \tilde{f}(\tilde{\textbf{x}}_{\tau})\right)-\eta_c\Delta_{\tau} \\
  % &=\left( u_{\tau+1}+\eta_c \right)\left(\big(g(\textbf{x}_{\tau}) -\nabla f(\textbf{x}_{\tau})\big)- \big(\tilde{g}(\tilde{\textbf{x}}_{\tau}) -\nabla \tilde{f}(\tilde{\textbf{x}}_{\tau})\big)+\Delta_{\tau}\right)-\eta_c\Delta_{\tau} \\
  &=\left( u_{\tau+1}+\eta_c \right)\left(\big(g(\textbf{x}_{\tau}) -\nabla f(\textbf{x}_{\tau})\big)- \big(\tilde{g}(\tilde{\textbf{x}}_{\tau}) -\nabla \tilde{f}(\tilde{\textbf{x}}_{\tau})\big)\right)+u_{\tau+1}\Delta_{\tau}\label{lemma2_1}.
    \end{aligned}
\end{equation}
Taking norm on both sides,
\begin{equation}
        \norm{\textbf{w}_{\tau}-\tilde{\textbf{w}}_{\tau}} \leq \abs{u_{\tau+1}+\eta_c}\left(Q+\tilde{Q} \right) +\abs{u_{\tau+1}}\norm{\Delta_{\tau}}\label{lemma2_2}
\end{equation}
Using \eqref{lemma2_1} and \eqref{lemma2_2} in \eqref{xt}, and assumption \hyperlink{hyper:a3}{\textbf{A3}} that stochastic noise is bounded, and applying norm, 
\begin{equation}
    \begin{aligned}
        \norm{\textbf{x}_t - \tilde{\textbf{x}}_t}&= \norm{- \sum_{\tau=0}^{t-1} \left(\eta_c\Delta_{\tau}+( \textbf{w}_{\tau}-\tilde{\textbf{w}}_{\tau} )  \right)} \leq \sum_{\tau=0}^{t-1} \norm{\eta_c\Delta_{\tau}+( \textbf{w}_{\tau}-\tilde{\textbf{w}}_{\tau} ) } \\
        % &\leq\sum_{\tau=0}^{t-1} \eta_c \norm{\Delta_{\tau}} + \abs{u_{\tau+1}+\eta_c}\left(Q+\tilde{Q} \right) +\abs{u_{\tau+1}}\norm{\Delta_{\tau}}\\
        &\leq\sum_{\tau=0}^{t-1}(\eta_c+\abs{u_{\tau+1}}) \left(  \norm{\Delta_{\tau}}+Q+\tilde{Q}\right)\label{lemma2_3}
    \end{aligned}
\end{equation}
Next, we focus on providing a bound for $\norm{\Delta_t}$. Recall that $\Delta_t = \nabla f(\textbf{x}_t) - \nabla \tilde{f}(\tilde{\textbf{x}}_t)$. The gradient can be written as \citep{nesterov},
\begin{equation*}
    \begin{aligned}
        \nabla f(\textbf{x}_t) &= \nabla f(\textbf{x}_{t-1}) +(\textbf{x}_t - \textbf{x}_{t-1})\left(\int_0^1 \textbf{H}(\textbf{x}_{t-1} +v(\textbf{x}_t - \textbf{x}_{t-1})) dv\right)  \\
        &\hspace{-1mm}= \nabla f(\textbf{x}_{t-1}) + (\textbf{x}_t - \textbf{x}_{t-1}) \left(\int_0^1 \big(\textbf{H}(\textbf{x}_{t-1} +v(\textbf{x}_t - \textbf{x}_{t-1})) +\textbf{H}(\textbf{x}_{t-1}) - \textbf{H}(\textbf{x}_{t-1}) \big)dv\right)  \\
        % &= \nabla f(x_{t-1}) + H(x_{t-1})(x_t - x_{t-1}) + \int_0^1 (H(x_{t-1} +t(x_t - x_{t-1}))  - H(x_{t-1}) )dt. (x_t - x_{t-1}) \\
        &\hspace{-1mm}= \nabla f(\textbf{x}_{t-1}) + \textbf{H}(\textbf{x}_{t-1})(\textbf{x}_t - \textbf{x}_{t-1}) + \theta_{t-1}, 
    \end{aligned}
\end{equation*}
%\textcolor{red}{Why are we keeping theta seperately}

where $\theta_{t-1} = \left(\int_0^1 \big(\textbf{H}(\textbf{x}_{t-1} +v(\textbf{x}_t - \textbf{x}_{t-1}))  - \textbf{H}(\textbf{x}_{t-1}) \big)dv\right) (\textbf{x}_t - \textbf{x}_{t-1})$. Let $H^{'}_{t-1} = \textbf{H}(\textbf{x}_{t-1}) - \textbf{H}(\textbf{x}_0)$. Using the SGD-PLRS update,
\begin{equation}
    \begin{aligned}
\nabla f(\textbf{x}_t) &= \nabla f(\textbf{x}_{t-1}) - (H_{t-1}^{'} + \textbf{H}(\textbf{x}_0)) (\eta_c \nabla f(\textbf{x}_{t-1}) + \textbf{w}_{t-1}) + \theta_{t-1}  \\
&= \nabla f(\textbf{x}_{t-1})(I-\eta_c \textbf{H}(\textbf{x}_0)) - \textbf{H}(\textbf{x}_0)\textbf{w}_{t-1} -\eta_c H^{'}_{t-1} \nabla f(\textbf{x}_{t-1}) - H^{'}_{t-1} \textbf{w}_{t-1} + \theta_{t-1},  \label{nablafxt}
    \end{aligned}
\end{equation}
 From \eqref{quadratic} in the proof of Lemma \ref{lemma:xt_x0}, 
\begin{equation}\label{nablaftilde}
\nabla \tilde{f}(\tilde{\textbf{x}}_t) = \nabla \tilde{f}(\tilde{\textbf{x}}_{t-1}) + \textbf{H}(\textbf{x}_0)(\tilde{\textbf{x}}_t - \tilde{\textbf{x}}_{t-1}). 
\end{equation}
Subtracting \eqref{nablaftilde} from \eqref{nablafxt}, we obtain $\Delta_t$ as,
\allowdisplaybreaks
\begin{equation}
    \begin{aligned}
\Delta_t &= \nabla f(\textbf{x}_{t-1})(I-\eta_c \textbf{H}(\textbf{x}_0)) - \textbf{H}(\textbf{x}_0)\textbf{w}_{t-1} -\eta_c H^{'}_{t-1} \nabla f(\textbf{x}_{t-1}) - H^{'}_{t-1} \textbf{w}_{t-1} + \theta_{t-1}  \\
&\hspace{24em}-  \nabla \tilde{f}(\tilde{\textbf{x}}_{t-1}) - \textbf{H}(\textbf{x}_0)(\tilde{\textbf{x}}_t- \tilde{\textbf{x}}_{t-1}) \\
 &= (I-\eta_c \textbf{H}(\textbf{x}_0))\Delta_{t-1} -\textbf{H}(\textbf{x}_0)\left(\textbf{w}_{t-1} - \tilde{\textbf{w}}_{t-1} \right)- H^{'}_{t-1}\big( \eta_c \Delta_{t-1}+  \eta_c \nabla \tilde{f}(\tilde{\textbf{x}}_{t-1}) \big)\\
 &\hspace{24em} -H^{'}_{t-1}\textbf{w}_{t-1}+ \theta_{t-1}, \label{deltat}
    \end{aligned}
\end{equation}
%\textcolor{red}{Try to prove above eqn by induction}
We now have an expression for $\Delta_t$. However, the derived expression is recursive and contains $\Delta_{t-1}$. We focus on eliminating the recursive dependence and obtain a stand-alone bound for $\norm{\Delta_t}$ $\forall t \leq T$. Now, we bound each of the five terms (we term them $T_1, \cdots, T_5$) of \eqref{deltat}. First, let us define the events,
\begin{align*}
    R_t &= \left\{\forall \tau\leq t,\quad \norm{\nabla \tilde{f}(\tilde{\textbf{x}}_{\tau}) }\leq \tilde{O}\left(\frac{1}{\sqrt{L_{max}}}\right), \quad \norm{\tilde{\textbf{x}}_{\tau} - \textbf{x}_0}\leq \tilde{O}\left(L_{max}^{3/8}\log\left(\frac{1}{L_{max}}\right)\right)\right\}\\
    C_t &=\left\{\forall \tau\leq t,\quad \norm{\Delta_{\tau}}\leq \mu  L_{max}^{3/8} \log\left(\frac{1}{L_{max}}\right)\right\}.
\end{align*} 
It can be seen that $R_t \subset R_{t-1}$ and $C_t \subset C_{t-1}$. Note that, from Lemma \ref{lemma:xt_x0}, we know the probabilistic characterization of $R_t$. We comment on the parameter $\mu$ later in the proof. Now, we derive bounds for each term of $\Delta_t$ \textit{conditioned} on the event $R_{t-1}\cap C_{t-1}$ for time $t\leq T=O\left(L_{max}^{-1/4}\right)$. 
\begin{equation}
    \begin{aligned}
        T_1: \quad \norm{(I-\eta_c \textbf{H}(\textbf{x}_0))\Delta_{t-1}}&\leq \norm{\Delta_{t-1}} + \norm{-\eta_c \textbf{H}(\textbf{x}_0)\Delta_{t-1}}\\
        &\leq \mu  L_{max}^{3/8} \log\left(\frac{1}{L_{max}}\right) + \tilde{O}\left(\mu  L_{max}^{11/8} \log\left(\frac{1}{L_{max}}\right)\right)\label{b1}\\
        &=\tilde{O}\left(\mu  L_{max}^{3/8} \log\left(\frac{1}{L_{max}}\right)\right),
    \end{aligned}
\end{equation}

%\textcolor{red}{Use O tilde here}

where \eqref{b1} follows from the definition of event $C_{t-1}$. Note that the first term in \eqref{b1} governs the order of the expression (as $0 \leq L_{max} \leq 1$).
\begin{equation*}
    \begin{aligned}
        T_2: \quad \norm{\textbf{H}(\textbf{x}_0)\left(\textbf{w}_{t-1} - \tilde{\textbf{w}}_{t-1} \right)}&\leq \norm{\textbf{H}(\textbf{x}_0)}\norm{\textbf{w}_{t-1} - \tilde{\textbf{w}}_{t-1}}\\
        &\leq  \norm{\textbf{H}(\textbf{x}_0)}\left(  \abs{u_{\tau+1}+\eta_c}\left(Q+\tilde{Q} \right) +\abs{u_{\tau+1}}\norm{\Delta_{\tau}}\right)\\
        &\leq \tilde{O}(L_{max}) + \tilde{O}\left(\mu  L_{max}^{11/8} \log\left(\frac{1}{L_{max}}\right)\right)=\tilde{O}(L_{max}),
    \end{aligned}
\end{equation*}
where the substitution follows from \eqref{lemma2_2}. 
% \begin{equation}\label{lemma2_7}
%         T_3: \quad \norm{H^{'}_{t-1}\eta_c(\Delta_{t-1}+\nabla \tilde{f}(\tilde{\textbf{x}}_{t-1}))}\leq \eta_c \norm{H^{'}_{t-1}\Delta_{t-1}}+\eta_c \norm{H^{'}_{t-1}\nabla \tilde{f}(\tilde{\textbf{x}}_{t-1}))}.
% \end{equation}
To bound $T_3$ and $T_4$, we first bound $H^{'}_{t-1}$,
\begin{subequations}
    \begin{align}
        \norm{H^{'}_{t-1}}&=\norm{\textbf{H}(\textbf{x}_{t-1})-\textbf{H}(\textbf{x}_0)}\leq \rho \norm{\textbf{x}_{t-1} -\textbf{x}_0}\label{assume_a2}\\
        &\leq \rho \left( \norm{\textbf{x}_{t-1} -\tilde{\textbf{x}}_{t-1}}+\norm{\tilde{\textbf{x}}_{t-1}-\textbf{x}_0}\right) \nonumber\\
        &\leq \rho \left( \sum_{\tau=0}^{t-1}(\eta_c+\abs{u_{\tau+1}}) \left(  \norm{\Delta_{\tau}}+Q+\tilde{Q}\right)\right)+\rho \tilde{O}\left( L_{max}^{3/8}\log \frac{1}{L_{max}} \right) \label{lemma2_5}\\
        &=\tilde{O}\left(\frac{1}{L_{max}^{1/4}}\right)\tilde{O}\left(\mu L_{max}^{11/8}\log\frac{1}{L_{max}} \right) + \tilde{O}\left(\frac{1}{L_{max}^{1/4}}\right)\tilde{O}(L_{max}) +\tilde{O}\left( L_{max}^{3/8}\log \frac{1}{L_{max}} \right)\label{lemma2_6}\\
        &\leq \tilde{O}(L_{max}^{3/4}) +\tilde{O}\left( L_{max}^{3/8}\log \frac{1}{L_{max}} \right)\leq \tilde{O}\left( L_{max}^{3/8}\log \frac{1}{L_{max}} \right)\label{lemma2_8},
    \end{align}
\end{subequations}
where \eqref{assume_a2} follows from the assumption \hyperlink{hyper:a2}{\textbf{A2}} while
 \eqref{lemma2_5} follows from \eqref{lemma2_3}. We use the bounds defined for events $R_{t-1}\cap C_{t-1}$ in \eqref{lemma2_5} and \eqref{lemma2_6}. Now, using the bound for $\norm{H^{'}_{t-1}}$, $T_3$ can be bounded as follows.
\begin{equation*}
    \begin{aligned}
        T_3 :  \quad \norm{H^{'}_{t-1}\eta_c(\Delta_{t-1}+\nabla \tilde{f}(\tilde{\textbf{x}}_{t-1}))} &\leq \eta_c \norm{H^{'}_{t-1}\Delta_{t-1}}+\eta_c \norm{H^{'}_{t-1}\nabla \tilde{f}(\tilde{\textbf{x}}_{t-1})}\\
        % &\leq \eta_c\norm{H^{'}_{t-1}}\norm{\Delta_{t-1}} +\eta_c \norm{H^{'}_{t-1}} \norm{\nabla \tilde{f}(\tilde{\textbf{x}}_{t-1})}\nonumber\\
        &\leq O(L_{max})\tilde{O}\left( L_{max}^{3/8}\log \frac{1}{L_{max}} \right)\mu L_{max}^{3/8}\log\frac{1}{L_{max}} \nonumber\\
        &+O(L_{max}) \tilde{O}\left( L_{max}^{3/8}\log \frac{1}{L_{max}} \right)\tilde{O}\left(\frac{1}{\sqrt{L_{max}}}\right)\label{lemma2_9}\\
        &=\tilde{O}\left( L_{max}^{7/8}\log \frac{1}{L_{max}} \right),
    \end{aligned}
\end{equation*}
where we use the bounds in the event $R_{t-1}\cap C_{t-1}$ and \eqref{lemma2_8}.
\begin{subequations}
    \begin{align}
        T_4: \quad \norm{H^{'}_{t-1}\textbf{w}_{t-1}}&\leq \norm{H^{'}_{t-1}}\norm{\textbf{w}_{t-1}} =\norm{H^{'}_{t-1}}\norm{\eta_c g(\textbf{x}_{t-1})-\eta_c \nabla f(\textbf{x}_{t-1} +u_t g(\textbf{x}_t)}\nonumber\\
        % &\leq \norm{H^{'}_{t-1}}\left(\eta_cQ +\norm{u_t\left( g(\textbf{x}_t)-\nabla f(\textbf{x}_{t-1})+\nabla f(\textbf{x}_{t-1})\right)}\right)\label{lemma2_10}\\
        &\leq \norm{H^{'}_{t-1}}\left(\eta_cQ+\abs{u_t}Q +\abs{u_t}\norm{\nabla f(\textbf{x}_{t-1})} \right)\label{lemma2_11}\\
        &=(\eta_c+\abs{u_t})Q\norm{H^{'}_{t-1}} +\abs{u_t} \norm{H^{'}_{t-1}}\norm{\Delta_{t-1}} +\abs{u_t} \norm{H^{'}_{t-1}}\norm{\nabla \tilde{f}(\tilde{\textbf{x}}_{t-1})}\nonumber\\
        &=\tilde{O}\left( L_{max}^{11/8}\log \frac{1}{L_{max}} \right)+\tilde{O}\left(\mu L_{max}^{14/8}\log^2\frac{1}{L_{max}} \right)+\tilde{O}\left( L_{max}^{7/8}\log \frac{1}{L_{max}} \right)\label{lemma2_12}\\
        &=\tilde{O}\left( L_{max}^{7/8}\log \frac{1}{L_{max}} \right) \nonumber,
    \end{align}
\end{subequations}
where we use assumption \hyperlink{hyper:a3}{\textbf{A3}} in \eqref{lemma2_11} and the bounds of  $R_{t-1}\cap C_{t-1}$ and \eqref{lemma2_8} in \eqref{lemma2_12}.
\begin{subequations}
    \begin{align}
        T_5: \quad \norm{\theta_{t-1}}&= \norm{\left(\int_0^1 \big(\textbf{H}(\textbf{x}_{t-1} +v(\textbf{x}_t - \textbf{x}_{t-1}))  - \textbf{H}(\textbf{x}_{t-1}) \big) \; dv \right) \; (\textbf{x}_t - \textbf{x}_{t-1})}\nonumber \\
% &\leq \int_0^1 \norm{ (\textbf{H}(\textbf{x}_{t-1} +t(\textbf{x}_t - \textbf{x}_{t-1}))  - \textbf{H}(\textbf{x}_{t-1}) ) } \; dt \; \norm{\textbf{x}_t - \textbf{x}_{t-1}} \\
&\leq \left(\int_0^1 \rho \norm{ \textbf{x}_{t-1} +v(\textbf{x}_t - \textbf{x}_{t-1}) - \textbf{x}_{t-1}} \; dv \right)\; \norm{\textbf{x}_t - \textbf{x}_{t-1}} \label{theta1}\\
% &\leq \left(\int_0^1 \rho v \; dv \right)\; \norm{\textbf{x}_t - \textbf{x}_{t-1}}^2 \nonumber\\ 
&\leq \frac{\rho}{2} \norm{ \textbf{x}_t - \textbf{x}_{t-1}}^2 \leq \frac{\rho}{2} \norm{ -\eta_c \nabla f(\textbf{x}_{t-1}) -\textbf{w}_{t-1}}^2 \nonumber\\
% \end{align}
% \end{subequations}
% where \eqref{theta1} follows from the assumption\hyperlink{hyper:a2}{\textbf{A2}} and \eqref{theta2} is obtained using the SGD-PLRS update. Using $\textbf{w}_{t-1} = \eta_c g(\textbf{x}_{t-1})-\eta_c\nabla f(\textbf{x}_{t-1}) +u_t g(\textbf{x}_{t-1})$, we write,
% \begin{subequations}
%     \begin{align}
   &\leq\frac{\rho}{2} \norm{ -\eta_c \nabla f(\textbf{x}_{t-1}) -\eta_c g(\textbf{x}_{t-1})+\eta_c\nabla f(\textbf{x}_{t-1}) -u_t g(\textbf{x}_{t-1})}^2 \nonumber\\
% &=\frac{\rho}{2} \norm{  -(\eta_c +u_t) g(\textbf{x}_{t-1})}^2\nonumber\\
% &=\frac{\rho \abs{\eta_c +u_t}^2}{2}\norm{g(\textbf{x}_{t-1})-\nabla f(\textbf{x}_{t-1})+\nabla f(\textbf{x}_{t-1})}^2\nonumber\\
% &\leq\frac{\rho \abs{\eta_c +u_t}^2}{2}\left( Q^2 + \norm{\nabla f(\textbf{x}_{t-1})}^2 + 2\nabla f(\textbf{x}_{t-1})^T \left(g(\textbf{x}_{t-1})-\nabla f(\textbf{x}_{t-1})\right) \right)\label{lemma2_13}\\
% &\leq\frac{\rho \abs{\eta_c +u_t}^2}{2}\left( Q^2 + \norm{\nabla f(\textbf{x}_{t-1})}^2 + 2\norm{\nabla f(\textbf{x}_{t-1})}\norm{g(\textbf{x}_{t-1})-\nabla f(\textbf{x}_{t-1})} \right)\nonumber\\
&\leq \frac{\rho \abs{\eta_c +u_t}^2}{2} \left(Q^2 + \norm{\nabla f(\textbf{x}_{t-1})}^2 +2Q \norm{\nabla f(\textbf{x}_{t-1})} \right)\nonumber\\
% &\leq\frac{\rho \abs{\eta_c +u_t}^2}{2} \left( Q^2 + \left(\norm{\Delta_{t-1}}+\norm{\nabla\tilde{f}(\tilde{\textbf{x}}_{t-1})}\right)^2 +2Q\left(\norm{\Delta_{t-1}}+\norm{\nabla\tilde{f}(\tilde{\textbf{x}}_{t-1})}\right) \right)\nonumber\\
&=\frac{\rho \abs{\eta_c +u_t}^2}{2} \left( Q^2 + \norm{\Delta_{t-1}}^2 +  \norm{\nabla\tilde{f}(\tilde{\textbf{x}}_{t-1})}^2 +2\norm{\Delta_{t-1}}\norm{\nabla\tilde{f}(\tilde{\textbf{x}}_{t-1})} \right.\nonumber\\
&\left.\hspace{19em}+2 Q\norm{\Delta_{t-1}} + 2Q\norm{\nabla\tilde{f}(\tilde{\textbf{x}}_{t-1})}\right)\nonumber\\
% &\leq\tilde{O}(L_{max}^2)\left( \tilde{O}(1) + \mu^2 L_{max}^{3/4}\log^2\frac{1}{L_{max}}  +\tilde{O}\left(\frac{1}{L_{max}}  \right)\right.\nonumber\\
% &\left. +\tilde{O}\left(\mu L_{max}^{3/8}\log\frac{1}{L_{max}} \right)\tilde{O}\left(\frac{1}{\sqrt{L_{max}}}  \right) +\tilde{O}\left(\mu L_{max}^{3/8}\log\frac{1}{L_{max}} \right)+\tilde{O}\left(\frac{1}{\sqrt{L_{max}}}  \right)\right)\label{lemma2_15}\\
&=\tilde{O}(L_{max}^2) +  \tilde{O}\left(\mu^2 L_{max}^{11/4}\log^2\frac{1}{L_{max}} \right)+\tilde{O}(L_{max}) +\tilde{O}\left(\mu L_{max}^{15/8}\log\frac{1}{L_{max}} \right)\nonumber\\
& +\tilde{O}\left(\mu L_{max}^{19/8}\log\frac{1}{L_{max}} \right) + \tilde{O}(L_{max}^{3/2})=\tilde{O}(L_{max}).   \label{theta2}
    \end{align}
\end{subequations}
Here, we use assumption \hyperlink{hyper:a3}{\textbf{A3}} and the bounds of the event $R_{t-1}\cap C_{t-1}$ in \eqref{theta2}. Note that we have derived bounds so far conditioned on the event $R_{t-1}\cap C_{t-1}$. We now include this conditioning explicitly in our notations going forward.

To characterize $\norm{\Delta_t}^2$, we construct a supermartingale process; and to do so, we focus on finding $\mathbb{E}[\norm{\Delta_t}^2\mathbf{1}_{R_{t-1}\cap C_{t-1}}]$ using the bounds derived for the terms $T_1, \cdots, T_5$. Later, we use the Azuma-Hoeffding inequality to obtain a probabilistic bound of $\norm{\Delta_t}$. 
\begin{equation}
    \begin{aligned}
        \mathbb{E}[\norm{\Delta_t}^2\mathbf{1}_{R_{t-1}\cap C_{t-1}}|S_{t-1}]&\leq 
        % [(1+\eta_c\gamma_o)^2\norm{\Delta_{t-1}}^2+\tilde{O}\left(\mu  L_{max}^{3/8} \log\frac{1}{L_{max}}\right)\tilde{O}\left( L_{max}^{7/8}\log \frac{1}{L_{max}} \right)\nonumber\\
        % +\tilde{O}\left(\mu  L_{max}^{3/8} \log\frac{1}{L_{max}}\right)\tilde{O}(L_{max})+\tilde{O}(L_{max}^2) +\tilde{O}\left( L_{max}^{7/8}\log \frac{1}{L_{max}} \right)\tilde{O}(L_{max})\nonumber\\       
        % &+\tilde{O}\left( L_{max}^{7/4}\log^2 \frac{1}{L_{max}} \right)+\tilde{O}\left( L_{max}^{7/8}\log \frac{1}{L_{max}} \right)\tilde{O}(L_{max})\nonumber\\
        % &+\tilde{O}\left( L_{max}^{7/4}\log^2 \frac{1}{L_{max}} \right)+\tilde{O}(L_{max}^2)]\mathbf{1}_{R_{t-1}\cap C_{t-1}}\\        
        \Bigg[(1+\eta_c\gamma_o)^2\norm{\Delta_{t-1}}^2+\tilde{O}\left(\mu  L_{max}^{3/8} \log\frac{1}{L_{max}}\right)\tilde{O}\left( L_{max}^{7/8}\log \frac{1}{L_{max}} \right)\\
        &+\tilde{O}\left(\mu  L_{max}^{3/8} \log\frac{1}{L_{max}}\right)\tilde{O}(L_{max})+\tilde{O}(L_{max}^2)\\
        &+\tilde{O}\left( L_{max}^{7/8}\log \frac{1}{L_{max}} \right)\tilde{O}(L_{max})+\tilde{O}\left( L_{max}^{7/4}\log^2 \frac{1}{L_{max}} \right)\Bigg]\mathbf{1}_{R_{t-1}\cap C_{t-1}}\\
        % &=[(1+\eta_c\gamma_o)^2\norm{\Delta_{t-1}}^2+\tilde{O}\left(\mu  L_{max}^{10/8} \log^2\frac{1}{L_{max}}\right)+\tilde{O}\left(\mu  L_{max}^{11/8} \log\frac{1}{L_{max}}\right)\nonumber\\
        % &+\tilde{O}(L_{max}^2)+\tilde{O}\left( L_{max}^{15/8}\log \frac{1}{L_{max}} \right)+\tilde{O}\left( L_{max}^{7/4}\log^2 \frac{1}{L_{max}} \right)]\mathbf{1}_{R_{t-1}\cap C_{t-1}}\\
        &\leq \Bigg[(1+\eta_c\gamma_o)^2\norm{\Delta_{t-1}}^2 +\tilde{O}\left(\mu  L_{max}^{7/8} \log\frac{1}{L_{max}}\right)\Bigg]\mathbf{1}_{R_{t-1}\cap C_{t-1}} \label{lemma2_19}
    \end{aligned}
\end{equation}
% \textcolor{green}{Note: Note that so far we have only picked one of the terms as the upper bound. But here we pick a term which is not a part of the equation but which upper bounds all of its terms. This is so that while removing the conditioning on the event $C_t$, we will end up with a bound for $\norm{\Delta_t}$ which agrees with $C_t$(see \eqref{lemm9}.
% }

Now, let 
\begin{equation}
    G_t = (1+\eta_c\gamma_o)^{-2t}\left[\norm{\Delta_{t}}^2 + \tilde{O}\left(\mu  L_{max}^{7/8} \log\frac{1}{L_{max}}\right)\right].\label{process_gt}
\end{equation} Now, in order to prove the process $G_t\mathbf{1}_{R_{t-1}\cap C_{t-1}}$ is a supermartingale, we prove that $\mathbb{E}[G_t\mathbf{1}_{R_{t-1}\cap C_{t-1}}|S_{t-1}]\leq G_{t-1}\mathbf{1}_{R_{t-2}\cap C_{t-2}}$. We define a filtration $S_t=s\{\textbf{w}_0,\hdots,\textbf{w}_{t-1}\}$ where $s\{.\}$ denotes a sigma-algebra field.
\begin{subequations}
    \begin{align}
        &\mathbb{E}[G_t\mathbf{1}_{R_{t-1}\cap C_{t-1}}|S_{t-1}] \nonumber\\
        &\leq (1+\eta_c\gamma_o)^{-2t} \left((1+\eta_c\gamma_o)^2\norm{\Delta_{t-1}}^2+2\tilde{O}\left(\mu  L_{max}^{7/8} \log\frac{1}{L_{max}}\right)\right)\mathbf{1}_{R_{t-1}\cap C_{t-1}}\label{lemma2_20a}\\
        &\leq (1+\eta_c\gamma_o)^{-2t} \left((1+\eta_c\gamma_o)^2\norm{\Delta_{t-1}}^2+2(1+\eta_c\gamma_o)^2\tilde{O}\left(\mu  L_{max}^{7/8} \log\frac{1}{L_{max}}\right)\right)\mathbf{1}_{R_{t-1}\cap C_{t-1}}\label{lemma2_20}\\
        &=(1+\eta_c\gamma_o)^{-2(t-1)} \left( \norm{\Delta_{t-1}}^2+ \tilde{O}\left(\mu  L_{max}^{7/8} \log\frac{1}{L_{max}}\right)\right)\mathbf{1}_{R_{t-1}\cap C_{t-1}}\nonumber\\
        &=G_{t-1}\mathbf{1}_{R_{t-1}\cap C_{t-1}}\leq G_{t-1}\mathbf{1}_{R_{t-2}\cap C_{t-2}}\nonumber.
    \end{align}
\end{subequations}
% where \eqref{lemma2_20} holds as $\eta_c\gamma_o\geq 1$. 
To obtain \eqref{lemma2_20a}, we use \eqref{lemma2_19} to find $\mathbb{E}[G_t\mathbf{1}_{R_{t-1}\cap C_{t-1}}|S_{t-1}]$.
In \eqref{lemma2_20}, we upper bound by the multiplication of a positive term $(1+\eta_c\gamma_o)^2$.
Therefore, $G_t\mathbf{1}_{R_{t-1}\cap C_{t-1}}$ is a supermartingale.
\begin{equation*}
    \begin{aligned}
        \norm{\Delta_t}^2&-\mathbb{E}[\norm{\Delta_t}^2|S_{t-1}]\mathbf{1}_{R_{t-1}\cap C_{t-1}}\leq -2 \norm{(I-\eta_c \textbf{H}(\textbf{x}_0))\Delta_{t-1}}\norm{\textbf{H}(\textbf{x}_0)\left(\textbf{w}_{t-1} - \tilde{\textbf{w}}_{t-1} \right)}\\
        &-2\norm{(I-\eta_c \textbf{H}(\textbf{x}_0))\Delta_{t-1}}\norm{H^{'}_{t-1}\textbf{w}_{t-1}}+2\norm{(I-\eta_c \textbf{H}(\textbf{x}_0))\Delta_{t-1}}\norm{\theta_{t-1}}\\
        &+\norm{\textbf{H}(\textbf{x}_0)\left(\textbf{w}_{t-1} - \tilde{\textbf{w}}_{t-1} \right)}^2+\norm{H^{'}_{t-1}\textbf{w}_{t-1}}^2+2\norm{\textbf{H}(\textbf{x}_0)\left(\textbf{w}_{t-1} - \tilde{\textbf{w}}_{t-1} \right)}\norm{H^{'}_{t-1}\textbf{w}_{t-1}}\\
        &+2\norm{\textbf{H}(\textbf{x}_0)\left(\textbf{w}_{t-1} - \tilde{\textbf{w}}_{t-1} \right)}\norm{H^{'}_{t-1}\big( \eta_c \Delta_{t-1}+  \eta_c \nabla \tilde{f}(\tilde{\textbf{x}}_{t-1})  \big)}\\
        &-2\norm{\textbf{H}(\textbf{x}_0)\left(\textbf{w}_{t-1} - \tilde{\textbf{w}}_{t-1} \right)}\norm{\theta_{t-1}}+2\norm{H^{'}_{t-1}\big( \eta_c \Delta_{t-1}+  \eta_c \nabla \tilde{f}(\tilde{\textbf{x}}_{t-1})  \big)}\norm{H^{'}_{t-1}\textbf{w}_{t-1}}\\
        &-2\norm{H^{'}_{t-1}\big( \eta_c \Delta_{t-1}+  \eta_c \nabla \tilde{f}(\tilde{\textbf{x}}_{t-1})  \big)}\norm{\theta_{t-1}}-2\norm{H^{'}_{t-1}\textbf{w}_{t-1}}\norm{\theta_{t-1}}+\norm{\theta_{t-1}}^2\\
        % &\leq \tilde{O}\left(\mu  L_{max}^{3/8} \log\frac{1}{L_{max}}\right)\tilde{O}(L_{max}) +\tilde{O}\left(\mu  L_{max}^{3/8} \log\frac{1}{L_{max}}\right)\tilde{O}\left( L_{max}^{7/8}\log \frac{1}{L_{max}} \right)\nonumber\\
        % &+\tilde{O}\left(\mu  L_{max}^{3/8} \log\frac{1}{L_{max}}\right)\tilde{O}(L_{max})+\tilde{O}(L_{max}^2)+\tilde{O}(L_{max})\tilde{O}\left( L_{max}^{7/8}\log \frac{1}{L_{max}} \right)\nonumber\\
        % &+\tilde{O}(L_{max})\tilde{O}\left( L_{max}^{7/8}\log \frac{1}{L_{max}} \right)+\tilde{O}(L_{max}^2)+\tilde{O}\left( L_{max}^{7/4}\log^2 \frac{1}{L_{max}} \right)+\tilde{O}\left( L_{max}^{7/8}\log \frac{1}{L_{max}} \right)\tilde{O}(L_{max})\nonumber\\
        % &+\tilde{O}\left( L_{max}^{7/4}\log^2 \frac{1}{L_{max}} \right)+\tilde{O}\left( L_{max}^{7/8}\log \frac{1}{L_{max}} \right)\tilde{O}(L_{max})+\tilde{O}(L_{max}^2)\\
        % &=\tilde{O}\left(\mu  L_{max}^{3/8} \log\frac{1}{L_{max}}\right)\tilde{O}(L_{max}) +\tilde{O}\left(\mu  L_{max}^{3/8} \log\frac{1}{L_{max}}\right)\tilde{O}\left( L_{max}^{7/8}\log \frac{1}{L_{max}} \right)\nonumber\\
        % &+\tilde{O}(L_{max}^2)+\tilde{O}(L_{max})\tilde{O}\left( L_{max}^{7/8}\log \frac{1}{L_{max}} \right)+\tilde{O}\left( L_{max}^{7/4}\log^2 \frac{1}{L_{max}} \right)\\
        &=\tilde{O}\left(\mu  L_{max}^{11/8} \log\frac{1}{L_{max}}\right)+\tilde{O}\left(\mu  L_{max}^{10/8} \log^2\frac{1}{L_{max}}\right)+\tilde{O}(L_{max}^2)+\tilde{O}\left( L_{max}^{15/8}\log \frac{1}{L_{max}} \right)\nonumber\\
        &+\tilde{O}\left( L_{max}^{7/4}\log^2 \frac{1}{L_{max}} \right)\leq \tilde{O}\left( \mu L_{max}^{7/8}\log \frac{1}{L_{max}} \right)
    \end{aligned}
\end{equation*}
Note that the above expression is obtained by the observation that the only random terms of $\Delta_t$ conditioned on the filtration $S_{t-1} = s\{\textbf{w}_0,\textbf{w}_1,\hdots,\textbf{w}_{t-2}\}$ are $\textbf{H}(\textbf{x}_0)\left(\textbf{w}_{t-1} - \tilde{\textbf{w}}_{t-1} \right)$, $H^{'}_{t-1}\textbf{w}_{t-1}$ and $\theta_{t-1}$(see \eqref{theta1}). 
% From \eqref{nablaf1} and \eqref{nablafxt} it is clear that $\nabla f(\textbf{x}_{t-1})$ and $\nabla \tilde{f}(\tilde{\textbf{x}}_{t-1})$ depend only on indices of $w$ between $0$ and $t-2$. 
Hence, we cancel out the deterministic terms in $\norm{\Delta_t}^2$ and $\mathbb{E}\norm{\Delta_t}^2$ and neglect the negative terms while upper bounding.
% Note that the remaining expectation terms will be positive and can be upper bounded by \eqref{lem2}. 

The Azuma-Hoeffding inequality for martingales and supermartingales \citep{hoeffding1994probability} states that if $\{G_t \mathbf{1}_{R_{t-1}\cap C_{t-1}}\}$ is a supermartingale and $|G_t\mathbf{1}_{R_{t-1}\cap C_{t-1}} -G_{t-1}\mathbf{1}_{R_{t-2}\cap C_{t-2}}| \leq c_t$ almost surely, then for all positive integers $t$ and positive reals $\delta$,
\begin{equation*}
    \mathbb{P}(G_t \mathbf{1}_{R_{t-1}\cap C_{t-1}}-G_0 \mathbf{1}_{R_{-1}\cap C_{-1}}  \geq \delta)\leq \exp\left(-\frac{\delta^2}{2\sum_{\tau=0}^{t-1} c_{\tau}^2}\right).
\end{equation*}
The bound of $|G_t\mathbf{1}_{R_{t-1}\cap C_{t-1}} -G_{t-1}\mathbf{1}_{R_{t-2}\cap C_{t-2}}|$ can be obtained using the definition of the process $G_t$ in \eqref{process_gt}. Recollecting our assumption that $\eta_c\leq\eta_{max}\leq \frac{\sqrt{2}-1}{\gamma^{'}}, \gamma_o\leq \gamma^{'}$, we see that $(1+\eta_c\gamma_o)^{-2t}\leq\tilde{O}(1)$. Therefore,
\begin{equation*}
    \begin{aligned}
|G_t\mathbf{1}_{R_{t-1}\cap C_{t-1}} - \mathbb{E}[G_t\mathbf{1}_{R_{t-1}\cap C_{t-1}}|S_{t-1}]|&=(1+\eta_c\gamma_o)^{-2t}\left| \norm{\Delta_t}^2 -\mathbb{E}[\norm{\Delta_t}^2|S_{t-1}]\right| \mathbf{1}_{R_{t-1}\cap C_{t-1}} \\
&\leq \tilde{O}\left( \mu L_{max}^{7/8}\log \frac{1}{L_{max}}\right).
    \end{aligned}
\end{equation*}
\allowdisplaybreaks
We denote the bound obtained for $|G_t\mathbf{1}_{R_{t-1}\cap C_{t-1}} - \mathbb{E}[G_t\mathbf{1}_{R_{t-1}\cap C_{t-1}}|S_{t-1}]|$ as $c_{t-1}$.
Now, let $\delta = \sqrt{\sum_{\tau =0}^{t-1} c_{\tau}^2} \log\frac{1}{L_{max}}$ in the Azuma-Hoeffding inequality. Now, for any $t\leq T=O\left(L_{max}^{-1/4}\right)$, 
  $\delta = \sqrt{O\left(\frac{1}{L_{max}^{1/4}}\right)\tilde{O}\left(\mu^2 L_{max}^{7/4}\log^2\frac{1}{L_{max}}\right)} \log\frac{1}{L_{max}} = \tilde{O}\left(\mu L_{max}^{3/4}\log^2\frac{1}{L_{max}}\right).$  
\begin{equation*}
    \begin{aligned}
\mathbb{P}\left(G_t\mathbf{1}_{R_{t-1}\cap C_{t-1}} -G_0.1 \geq \tilde{O}\left(\mu L_{max}^{3/4}\log^2\frac{1}{L_{max}}\right) \right) &\leq \exp\left(-\tilde{\Omega}\left(\log^2\frac{1}{L_{max}}\right)\right) \\
& \leq \tilde{O}(L_{max}^4).
    \end{aligned}
\end{equation*}
After taking union bound $\forall \; t\leq T$,
\begin{equation*}
 \mathbb{P}\left(\forall \; t\leq T,\;G_t \mathbf{1}_{R_{t-1}\cap C_{t-1}}-G_0 \geq \tilde{O}\left(\mu L_{max}^{3/4}\log^2\frac{1}{L_{max}}\right) \right) \leq  \tilde{O}(L_{max}^{15/4}).  
\end{equation*}

 We represent the hidden constants in $\tilde{O}\left(\mu L_{max}^{3/4}\log^2\frac{1}{L_{max}}\right)$ by $\tilde{c}$ and choose $\mu$ such that $\mu <\tilde{c}$. Then, the following equation holds true.

\begin{equation*}
\begin{aligned}
% \mathbb{P}\left(G_t \mathbf{1}_{R_{t-1}\cap C_{t-1}}-G_0 \geq \tilde{c}\mu L_{max}^{3/4}\log^2\frac{1}{L_{max}}\right)\leq \tilde{O}( L_{max}^{15/4}) \\
\mathbb{P}\left(G_t \mathbf{1}_{R_{t-1}\cap C_{t-1}} -G_0\geq \mu^2L_{max}^{3/4}\log^2\frac{1}{L_{max}}\right)\leq \tilde{O}(L_{max}^{15/4}).\\
\end{aligned}
\end{equation*}
Hence we can write,
\begin{equation}
\begin{aligned}
% \mathbb{P}\left(R_{t-1}\cap C_{t-1} \cap \left\{\norm{\Delta_t}^2\geq \mu^2 L_{max}^{3/4}\log^2\frac{1}{L_{max}}\right\}\right)\leq \tilde{O}(L_{max}^{15/4}) \\
\mathbb{P}\left(R_{t-1}\cap C_{t-1} \cap \left\{\norm{\Delta_t}\geq \mu L_{max}^{3/8}\log\frac{1}{L_{max}}\right\}\right)\leq \tilde{O}(L_{max}^{15/4}). \label{lemm11}
\end{aligned}
\end{equation}
We need the probability of the event $C_t,$ $\forall t\leq T$ in order to prove the lemma. From Lemma \ref{lemma:xt_x0}, we get the probability of the event $\Bar{R_t}$ as $\tilde{O}(L_{max}^{15/4})$. Then,
\begin{equation}
    \begin{aligned}
\mathbb{P}\left( C_{t-1} \cap \left\{\norm{\Delta_t}\geq \mu L_{max}^{3/8}\log\frac{1}{L_{max}}\right\}\right) &= \mathbb{P}\left(R_{t-1}\cap C_{t-1} \cap \left\{\norm{\Delta_t}\geq \mu L_{max}^{3/8}\log\frac{1}{L_{max}}\right\}\right)  \\
&+ \mathbb{P}\left(\bar{R}_{t-1}\cap C_{t-1} \cap \left\{\norm{\Delta_t}\geq \mu L_{max}^{3/8}\log\frac{1}{L_{max}}\right\}\right) \\
&\leq \tilde{O}(L_{max}^{15/4}) + \mathbb{P}(\bar{R}_{t-1}) \label{lemm10}\leq \tilde{O}(L_{max}^{15/4}),
    \end{aligned}
\end{equation}
where the first term of \eqref{lemm10} follows from \eqref{lemm11}. The second term of \eqref{lemm10} can be bounded by $\mathbb{P}(\Bar{R}_{t-1})$ which is known by Lemma \ref{lemma:xt_x0}.
Finally,
\begin{equation*}
    \begin{aligned}
\mathbb{P}(\Bar{C}_t) &= \mathbb{P}\left( C_{t-1} \cap \left\{\norm{\Delta_t}\geq \mu L_{max}^{3/8}\log\frac{1}{L_{max}}\right\}\right) + \mathbb{P}(\bar{C}_{t-1}) \leq \tilde{O}(L_{max}^{15/4}) + \mathbb{P}(\bar{C}_{t-1}) .
    \end{aligned}
\end{equation*}
The probability $\mathbb{P}(\bar{C}_{t-1})$ can be found as,
\begin{equation*}
\begin{aligned}
\mathbb{P}(\bar{C}_{t-1}) &= \mathbb{P}\left( C_{t-2} \cap \left\{\norm{\Delta_{t-1}}\geq \mu L_{max}^{3/8}\log\frac{1}{L_{max}}\right\}\right) + \mathbb{P}(\bar{C}_{t-2}) \\
&=\mathbb{P}\left( C_{t-2} \cap \left\{\norm{\Delta_{t-1}}\geq \mu L_{max}^{3/8}\log\frac{1}{L_{max}}\right\}\right) +\hdots \\
% &+\mathbb{P}\left( C_{t-3} \cap \left\{\norm{\Delta_{t-2}}\geq \mu L_{max}^{3/8}\log\frac{1}{L_{max}}\right\}\right) +\hdots \\
&+\mathbb{P}\left( C_{0} \cap \left\{\norm{\Delta_1}\geq \mu L_{max}^{3/8}\log\frac{1}{L_{max}}\right\}\right) + \mathbb{P}(\bar{C}_0) .
\end{aligned}
\end{equation*}
As $T= O\left(L_{max}^{-1/4}\right)$, $\mathbb{P}(\Bar{C}_T)\leq \tilde{O}\left(L_{max}^{7/2}\right)$. From \eqref{lemma2_3},
\begin{equation*}
    \begin{aligned}
        \norm{\textbf{x}_t - \tilde{\textbf{x}}_t}&\leq\sum_{\tau=0}^{t-1}(\eta_c+\abs{u_{\tau+1}}) \left(  \norm{\Delta_{\tau}}+Q+\tilde{Q}\right)\\
        &\leq O\left(\frac{1}{L_{max}^{1/4}}\right)\left( \tilde{O}(L_{max}) \mu L_{max}^{3/8}\log\frac{1}{L_{max}}+\tilde{O}(L_{max})\right) \\
        &=O\left(\mu L_{max}^{9/8}\log\frac{1}{L_{max}}\right)+\tilde{O}(L_{max}^{3/4})\leq \tilde{O}(L_{max}^{3/4})
    \end{aligned}
\end{equation*}
This completes our proof. 
\end{proof}
\section{Proof of Theorem \ref{th:2}}
\label{ap:th2proof}

\begin{theorem}\label{th:2restated}(Theorem \ref{th:2} restated)
Consider $f$ satisfying Assumptions \hyperlink{hyper:a1}{\textbf{A1}} - \hyperlink{hyper:a5}{\textbf{A5}}. Let $\tilde{f}$ be the second order Taylor approximation of $f$; let $\{\textbf{x}_t\}$ and $\{\tilde{\textbf{x}}_t\}$ be the corresponding SGD iterates using PLRS, with $\tilde{\textbf{x}}_0=\textbf{x}_0$. Let $\textbf{x}_0$ correspond to \hyperlink{hyper:b2}{\textbf{B2}}, i.e., $\norm{\nabla f(\textbf{x}_0)}\leq \epsilon$ and $ \lambda_{min}(\textbf{H}(\textbf{x}_0)) \leq  -\gamma$ where $\epsilon,\gamma>0$. Then, there exists a $T=\tilde{O}\left(L_{max}^{-1/4}\right)$ such that with probability at least $1-\tilde{O}\left(L_{max}^{7/2}\right)$,
\[\mathbb{E}[f(\textbf{x}_T) - f(\textbf{x}_0)] \leq - \tilde{\Omega}\left(L_{max}^{3/4}\right).\]
\end{theorem}
\begin{proof}
% \comm{Are we using $\gamma_0$ anywhere?} Let the minimum eigenvalue of the Hessian of $f$, $ \lambda_{min}(\textbf{H}(\textbf{x}_0)) = -\gamma_{0}$, where $\gamma_0>0$. 
% Also, let $\gamma_o>(d-1)\lambda_{max}(H(x_0))$ where $\lambda_{max}(H(x_0))$ signifies the maximum eigenvalue of the Hessian at the initial iterate $x_0$. 

In this proof, we consider the case when the initial iterate $\textbf{x}_0$ is at a saddle point (corresponding to \hyperlink{hyper:b2}{\textbf{B2}}). This theorem shows that the SGD-PLRS algorithm escapes the saddle point in $T$ steps where $T=\tilde{O}\left(L_{max}^{-1/4}\right)$.

We use the Taylor series approximation in order to make the problem tractable. Similar to the SGD-PLRS updates for the function $f$, the SGD update on the function $\tilde{f}$ can be given as,
\begin{equation*}
    \tilde{\textbf{x}}_{t} = \tilde{\textbf{x}}_{t-1} - \eta_c \nabla \tilde{f}(\tilde{\textbf{x}}_{t-1})-\tilde{\textbf{w}}_{t-1} ,\quad\tilde{\textbf{w}}_{t-1} = \eta_c\tilde{g}(\tilde{\textbf{x}}_{t-1})-\eta_c \nabla \tilde{f}(\tilde{\textbf{x}}_{t-1})+u_t \tilde{g}(\tilde{\textbf{x}}_{t-1}).
\end{equation*}
As the function $f$ is $\rho$-Hessian, using \citep[Lemma 1.2.4]{nesterov} and the Taylor series expansion one obtains,
$f(\textbf{x})\leq f(\textbf{x}_0) +\nabla f(\textbf{x}_0)^T (\textbf{x}-\textbf{x}_0) +\frac{1}{2}(\textbf{x}-\textbf{x}_0)^T \textbf{H}(\textbf{x}_0) (\textbf{x}-\textbf{x}_0)+\frac{\rho}{6}\norm{\textbf{x}-\textbf{x}_0}^3.$
Let $\tilde{\boldsymbol{\kappa}} = \tilde{\textbf{x}}_T -\textbf{x}_0$, $\boldsymbol{\kappa} = \textbf{x}_T - \tilde{\textbf{x}}_T$. Note that $\tilde{\boldsymbol{\kappa}}+\boldsymbol{\kappa} = \textbf{x}_T -\textbf{x}_0$. Then, replacing $\textbf{x}$ by $\textbf{x}_T$,
\begin{align*}
        f(\textbf{x}_T) -f(\textbf{x}_0)&\leq \nabla f(\textbf{x}_0)^T (\textbf{x}_T-\textbf{x}_0) +\frac{1}{2}(\textbf{x}_T-\textbf{x}_0)^T \textbf{H}(\textbf{x}_0) (\textbf{x}_T-\textbf{x}_0)+\frac{\rho}{6}\norm{\textbf{x}_T-\textbf{x}_0}^3  \nonumber\\
        &= \nabla f(\textbf{x}_0)^T (\tilde{\boldsymbol{\kappa}}+\boldsymbol{\kappa}) +\frac{1}{2}(\tilde{\boldsymbol{\kappa}}+\boldsymbol{\kappa})^T \textbf{H}(\textbf{x}_0) (\tilde{\boldsymbol{\kappa}}+\boldsymbol{\kappa})+\frac{\rho}{6}\norm{\tilde{\boldsymbol{\kappa}}+\boldsymbol{\kappa}}^3 \nonumber \\
        &=\left(\nabla f(\textbf{x}_0)^T\tilde{\boldsymbol{\kappa}} + \frac{1}{2}\tilde{\boldsymbol{\kappa}}^T\textbf{H}(\textbf{x}_0)\tilde{\boldsymbol{\kappa}}\right)+\left( \nabla f(\textbf{x}_0)^T\boldsymbol{\kappa} +\tilde{\boldsymbol{\kappa}}^T\textbf{H}(\textbf{x}_0)\boldsymbol{\kappa} +\frac{1}{2}\boldsymbol{\kappa}^T\textbf{H}(\textbf{x}_0)\boldsymbol{\kappa} \right. \nonumber\\
        &\hspace{20em}\left. +\frac{\rho}{6}\norm{\tilde{\boldsymbol{\kappa}}+\boldsymbol{\kappa}}^3 \right).
    \end{align*}
Let the first term be $\tilde{\zeta} = \nabla f(\textbf{x}_0)^T\tilde{\boldsymbol{\kappa}} + \frac{1}{2}\tilde{\boldsymbol{\kappa}}^T\textbf{H}(\textbf{x}_0)\tilde{\boldsymbol{\kappa}}$ and the second term be $\zeta = \nabla f(\textbf{x}_0)^T\boldsymbol{\kappa} +\tilde{\boldsymbol{\kappa}}^T\textbf{H}(\textbf{x}_0)\boldsymbol{\kappa} +\frac{1}{2}\boldsymbol{\kappa}^T\textbf{H}(\textbf{x}_0)\boldsymbol{\kappa} +\frac{\rho}{6}\norm{\tilde{\boldsymbol{\kappa}}+\boldsymbol{\kappa}}^3$. Hence $f(\textbf{x}_T) -f(\textbf{x}_0) \leq \tilde{\zeta} + \zeta$. In order to prove the theorem, we require an upper bound on $\mathbb{E}[f(\textbf{x}_T) -f(\textbf{x}_0)]$.\\
Now, we introduce two mutually exclusive events $C_t$ and $\Bar{C}_t$ so that $\mathbb{E}[f(\textbf{x}_T)-f(\textbf{x}_0)]$ can be written in terms of events $C_t$ and $\Bar{C}_t$ as,
\begin{equation*}
    \begin{aligned}
\mathbb{E}[f(\textbf{x}_T)-f(\textbf{x}_0)] &= \mathbb{E}[f(\textbf{x}_T)-f(\textbf{x}_0)](\mathbb{E}[\mathbf{1}_{C_T}]+\mathbb{E}[\mathbf{1}_{\Bar{C}_T}])\\
&=\mathbb{E}[(f(\textbf{x}_T)-f(\textbf{x}_0))\mathbf{1}_{C_T}] + \mathbb{E}[(f(\textbf{x}_T)-f(\textbf{x}_0))\mathbf{1}_{\Bar{C}_T}] \\
&\leq \mathbb{E}[\tilde{\zeta}\mathbf{1}_{C_T}] + \mathbb{E}[\zeta\mathbf{1}_{C_T}]+ \mathbb{E}[(f(\textbf{x}_T)-f(\textbf{x}_0))\mathbf{1}_{\Bar{C}_T}] \\
% &=\mathbb{E}[\tilde{\zeta}\mathbf{1}_{C_T}] + \mathbb{E}[\tilde{\zeta}\mathbf{1}_{\Bar{C}_T}] -\mathbb{E}[\tilde{\zeta}\mathbf{1}_{\Bar{C}_T}] + \mathbb{E}[\zeta\mathbf{1}_{C_T}]+ \mathbb{E}[(f(\textbf{x}_T)-f(\textbf{x}_0))\mathbf{1}_{\Bar{C}_T}] \\
&=\mathbb{E}[\tilde{\zeta}] + \mathbb{E}[\zeta\mathbf{1}_{C_T}]+ \mathbb{E}[(f(\textbf{x}_T)-f(\textbf{x}_0))\mathbf{1}_{\Bar{C}_T}]-\mathbb{E}[\tilde{\zeta}\mathbf{1}_{\Bar{C}_T}].
    \end{aligned}
\end{equation*}
Let $K_1 = \mathbb{E}[\tilde{\zeta}]$, $K_2 = \mathbb{E}[\zeta\mathbf{1}_{C_T}]$ and $K_3=\mathbb{E}[(f(\textbf{x}_T)-f(\textbf{x}_0))\mathbf{1}_{\Bar{C}_T}]-\mathbb{E}[\tilde{\zeta}\mathbf{1}_{\Bar{C}_T}]$. In the remainder of the proof, we focus on deriving the bounds for individual terms, $K_1$, $K_2$ and $K_3$, and then finally put them together to obtain the result of the theorem.
\subsection{Bounding $K_1$}
Using \eqref{xt_x0} from the proof of Lemma \ref{lemma:xt_x0} in Appendix \ref{ap:lemma1}, we obtain the bound for the term $K_1= \mathbb{E}[\tilde{\zeta}]$ as,
\begin{equation*}
    \begin{aligned}
    \mathbb{E}[\tilde{\zeta}]&=\mathbb{E}\left[\nabla f(\textbf{x}_0)^T(\tilde{\textbf{x}}_T -\textbf{x}_0) + \frac{1}{2}(\tilde{\textbf{x}}_T -\textbf{x}_0)^T\textbf{H}(\textbf{x}_0)(\tilde{\textbf{x}}_T -\textbf{x}_0)\right]\\
&=\mathbb{E}\left[\nabla f(\textbf{x}_0)^T\left( -\sum_{\tau=0}^{T-1}\eta_c(I-\eta_c\textbf{H}(\textbf{x}_0))^{\tau}\nabla f(\textbf{x}_{0})- \sum_{\tau=0}^{T-1}(I-\eta_c\textbf{H}(\textbf{x}_0))^{T-\tau-1}\tilde{\textbf{w}}_{\tau}  \right)\right]\nonumber\\
&+\frac{1}{2}\mathbb{E}\Bigg[\Bigg( -\sum_{\tau=0}^{T-1}\eta_c(I-\eta_c\textbf{H}(\textbf{x}_0))^{\tau}\nabla f(\textbf{x}_{0})- \sum_{\tau=0}^{T-1}(I-\eta_c\textbf{H}(\textbf{x}_0))^{T-\tau-1}\tilde{\textbf{w}}_{\tau}  \Bigg)^T \textbf{H}(\textbf{x}_0) \nonumber\\
&\Bigg( -\sum_{\tau=0}^{T-1}\eta_c(I-\eta_c\textbf{H}(\textbf{x}_0))^{\tau}\nabla f(\textbf{x}_{0})- \sum_{\tau=0}^{T-1}(I-\eta_c\textbf{H}(\textbf{x}_0))^{T-\tau-1}\tilde{\textbf{w}}_{\tau}  \Bigg)\Bigg].
\end{aligned}
\end{equation*}
% Note that we have some terms that involve only $\mathbb{E}[\tilde{\textbf{w}}_{\tau}]$ and one term with expectation over quadratic function of noise $\tilde{\textbf{w}}_{\tau}$. 
Since $\tilde{\textbf{w}}_{\tau}= \textbf{0}$, all the terms with $\mathbb{E}[\tilde{\textbf{w}}_{\tau}]$ will go to zero. Hence we obtain,

    \begin{align*}
\mathbb{E}[\tilde{\zeta}]&=\nabla f(\textbf{x}_0)^T\left( -\sum_{\tau=0}^{T-1}\eta_c(I-\eta_c\textbf{H}(\textbf{x}_0))^{\tau}\nabla f(\textbf{x}_{0})\right) +\nonumber\\
&\frac{1}{2}\Bigg(-\sum_{\tau=0}^{T-1}\eta_c(I-\eta_c\textbf{H}(\textbf{x}_0))^{\tau}\nabla f(\textbf{x}_{0})\Bigg)^T\textbf{H}(\textbf{x}_0)\Bigg(-\sum_{\tau=0}^{T-1}\eta_c(I-\eta_c\textbf{H}(\textbf{x}_0))^{\tau}\nabla f(\textbf{x}_{0})\Bigg)\nonumber\\
&+\frac{1}{2}\mathbb{E}\Bigg[\Bigg(- \sum_{\tau=0}^{T-1}(I-\eta_c\textbf{H}(\textbf{x}_0))^{T-\tau-1}\tilde{\textbf{w}}_{\tau}\Bigg)^T\textbf{H}(\textbf{x}_0)\Bigg(- \sum_{\tau=0}^{T-1}(I-\eta_c\textbf{H}(\textbf{x}_0))^{T-\tau-1}\tilde{\textbf{w}}_{\tau}\Bigg)\Bigg].
\end{align*} 

Let $\lambda_1,\hdots,\lambda_d$ be the eigenvalues of the Hessian matrix at $\textbf{x}_0$, $\textbf{H}(\textbf{x}_0)$. Now, we simplify similar to Ge et al. \citep{ge2015escaping} as,
% \comm{Do we need to say similar to Ge here?}
    \begin{align*}
\mathbb{E}[\tilde{\zeta}]&=-\sum_{i=1}^d \sum_{\tau=0}^{T-1} \eta_c(1-\eta_c\lambda_i)^{\tau}\abs{\nabla_i f(\textbf{x}_0)}^2 + \frac{1}{2}\sum_{i=1}^d \lambda_i \sum_{\tau=0}^{T-1}\eta_c^2(1-\eta_c\lambda_i)^{2\tau}\abs{\nabla_i f(\textbf{x}_0)}^2 \nonumber\\
&+ \frac{1}{2}\sum_{i=1}^d \lambda_i \sum_{\tau=0}^{T-1}(1-\eta_c\lambda_i)^{2(T-\tau-1)} \mathbb{E}[\abs{\tilde{\textbf{w}}_{\tau,i}}^2] .
    \end{align*} 

Note that for the case of very small gradients (as per our initial conditions), $\abs{\nabla_i f(\textbf{x}_0)}^2\leq \norm{\nabla f(\textbf{x}_0)}\leq \epsilon$. Therefore, the first and second terms can be made arbitrarily small so that they do not contribute to the order of the equation. Hence, we focus on the third term. We first characterize $\mathbb{E}[\abs{\tilde{\textbf{w}}_{\tau,i}}^2]$ as follows. Since the norm of the stochastic noise is bounded as per the assumption \hyperlink{hyper:a3}{\textbf{A3}}, we assume that $\tilde{g}_i(\tilde{\textbf{x}}_t)-\nabla_i \tilde{f}(\tilde{\textbf{x}}_t)\leq \tilde{q}$ and $\mathbb{E}[\tilde{q}]\leq \tilde{\sigma}^2$.\\
\begin{equation*}
    \begin{aligned}
\tilde{\textbf{w}}_{\tau,i}&=\eta_c\tilde{g}_i(\tilde{\textbf{x}}_t)-\eta_c\nabla_i \tilde{f}(\tilde{\textbf{x}}_t)+u_{t+1}\tilde{g}_i(\tilde{\textbf{x}}_t)\\
        &\leq \eta_c\tilde{q} +u_{t+1}\left(\tilde{g}_i(\tilde{\textbf{x}}_t)-\nabla_i \tilde{f}(\tilde{\textbf{x}}_t)+\nabla_i \tilde{f}(\tilde{\textbf{x}}_t)\right)\\
        &\leq \tilde{q}(\eta_c +u_{t+1})+u_{t+1}\nabla_i \tilde{f}(\tilde{\textbf{x}}_t)\\
        \abs{\tilde{\textbf{w}}_{\tau,i}}^2&\leq \left(\tilde{q}(\eta_c +u_{t+1})+u_{t+1}\nabla_i \tilde{f}(\tilde{\textbf{x}}_t)\right)^2\\
        &=\tilde{q}^2(\eta_c^2+2\eta_c u_{t+1} +u_{t+1}^2)+2\tilde{q}\eta_c u_{t+1}\nabla_i\tilde{f}(\tilde{\textbf{x}}_t) +2\tilde{q}u_{t+1}^2\nabla_i\tilde{f}(\tilde{\textbf{x}}_t) + u_{t+1}^2\abs{\nabla_i \tilde{f}(\tilde{\textbf{x}}_t)}^2.
    \end{aligned}
\end{equation*}
% \textcolor{red}{Use variance of q instead of Q notation. Write it in terms of $\sigma^2$}
% \textcolor{green}{It will still be lesser than $\sigma^2$ because $|\textbf{x}_i|^2\leq \norm{x}^2$.}
Taking expectation with respect to $\tilde{q}$ and the uniformly distributed random variable $u_{t+1}$ and recalling that $\mathbb{E}[u_{t+1}]=0$, we set expectation over linear functions of $u_{t+1}$ to zero.
\begin{equation}
    \begin{aligned}
        \mathbb{E}[\abs{\tilde{\textbf{w}}_{\tau,i}}^2]&\leq 
        % \mathbb{E}\left[\tilde{q}^2(\eta_c^2+2\eta_c u_{t+1} +u_{t+1}^2)+2\tilde{q}\eta_c u_{t+1}\nabla_i\tilde{f}(\tilde{\textbf{x}}_t) +2\tilde{q}u_{t+1}^2\nabla_i\tilde{f}(\tilde{\textbf{x}}_t) + u_{t+1}^2\abs{\nabla_i \tilde{f}(\tilde{\textbf{x}}_t)}^2\right]\\
        % &=
        \tilde{\sigma}^2\eta_c^2+\tilde{\sigma}^2\mathbb{E}[u_{t+1}^2]+2\tilde{\sigma}^2\mathbb{E}[u_{t+1}^2]\nabla_i\tilde{f}(\tilde{\textbf{x}}_t)+\mathbb{E}[u_{t+1}^2]\abs{\nabla_i \tilde{f}(\tilde{\textbf{x}}_t)}^2\\
        &\leq \tilde{O}(L^2_{max})+\tilde{O}(L^2_{max})+\tilde{O}(L^2_{max})\tilde{O}\left(\frac{1}{\sqrt{L_{max}}}\right)+\tilde{O}(L_{max}^2)\tilde{O}\left(\frac{1}{L_{max}}\right)\\
        &=\tilde{O}(L^2_{max}) +\tilde{O}(L_{max}^{1.5})+\tilde{O}(L_{max})=\tilde{O}(L_{max})\label{theorem2_1}.
    \end{aligned}
\end{equation}
Here, we use $\mathbb{E}[u_{t+1}^2]=\frac{(L_{max}-L_{min})^2}{12}=\tilde{O}(L^2_{max})$. From \eqref{generalt} in the proof of Lemma \ref{lemma:xt_x0} (Appendix \ref{ap:lemma1}), $\norm{\nabla \tilde{f}(\tilde{\textbf{x}}_t)} \leq 10\tilde{Q} \sum_{\tau=0}^{\frac{t(t-1)}{2}} (1+\eta_c\gamma_o)^{\tau} = \tilde{O}\left(\frac{1}{\sqrt{L_{max}}}\right)$ as $t\leq T=\tilde{O}\left(L_{max}^{-1/4}\right)$. Also, note that $\tilde{q}$ and $u_{t+1}$ are independent of each other. As $\lambda_{min}(\textbf{H}(\textbf{x}_0))=-\gamma_o$,
\begin{subequations}
    \begin{align}
        &\frac{1}{2}\sum_{i=1}^d \lambda_i \sum_{\tau=0}^{T-1}(1-\eta_c\lambda_i)^{2(T-\tau-1)} \mathbb{E}[\abs{\tilde{\textbf{w}}_{\tau,i}}^2] \nonumber\\
        &\leq \frac{1}{2}\sum_{i=1}^d \lambda_i \sum_{\tau=0}^{T-1}(1+\eta_c\gamma_o)^{2\tau} \mathbb{E}[\abs{\tilde{\textbf{w}}_{\tau,i}}^2]
        \leq \frac{\tilde{O}(L_{max})}{2} \sum_{i=1}^d \lambda_i \sum_{\tau=0}^{T-1}(1+\eta_c\gamma_o)^{2\tau} \label{theorem2_2}\\
        &=\frac{\tilde{O}(L_{max})}{2}\Bigg(-\gamma_o\sum_{\tau=0}^{T-1}(1+\eta_c\gamma_o)^{2\tau} +(d-1)\lambda_{max}(\textbf{H}(\textbf{x}_0))\sum_{\tau=0}^{T-1}(1+\eta_c\gamma_o)^{2\tau} \Bigg)\label{theorem2_3},
    \end{align}
\end{subequations}
where we use the upper bound of $\mathbb{E}[\abs{\tilde{\textbf{w}}_{\tau,i}}^2]$ obtained from \eqref{theorem2_1} in \eqref{theorem2_2}. We use the fact that one of the eigenvalues of $\textbf{H}(\textbf{x}_0)$ is $-\gamma_o$ and then upper bound the other eigenvalues by the maximum eigenvalue $\lambda_{max}(\textbf{H}(\textbf{x}_0))$ in \eqref{theorem2_3}.

Let $\eta_c\leq \eta_{max}\leq \frac{\sqrt{2}-1}{\gamma^{'}}$ where $\gamma\leq\gamma_o\leq\gamma^{'}$. As $\sum_{\tau=0}^{T-1}(1+\eta_c\gamma_o)^{2\tau} $ is a monotonically increasing sequence, we choose the smallest $T$ that satisfies $\frac{d}{\eta_c^{1/4}\gamma_o}\leq \sum_{\tau=0}^{T-1}(1+\eta_c\gamma_o)^{2\tau}$. Therefore, $\sum_{\tau=0}^{T-2}(1+\eta_c\gamma_o)^{2\tau}\leq \frac{d}{\eta_c^{1/4}\gamma_o}$. Now,
\begin{equation*}
     \sum_{\tau=0}^{T-1}(1+\eta_c\gamma_o)^{2\tau} = 1+(1+\eta_c\gamma_o)^2\sum_{\tau=0}^{T-2}(1+\eta_c\gamma_o)^{2\tau}\leq 1+\frac{2d}{\eta_c^{1/4}\gamma_o},
\end{equation*}
which follows from our constraints that $\eta_c<\frac{\sqrt{2}-1}{\gamma^{'}}$ and $\gamma_o\leq \gamma^{'}$ making $(1+\eta_c\gamma)^2\leq \left(1+\frac{\sqrt{2}-1}{\gamma^{'}}\gamma^{'}\right)^2\leq 2$. Further using $\eta_c\gamma_o\leq \eta_c^{1/4}\gamma_o\leq \frac{\sqrt{2}-1}{\gamma^{'}}\gamma^{'}<d$, 
\begin{equation}\label{theorem2_4}
    \frac{d}{\eta_c^{1/4}\gamma_o}\leq \sum_{\tau=0}^{T-1}(1+\eta_c\gamma_o)^{2\tau} \leq 1+\frac{2d}{\eta_c^{1/4}\gamma_o}\leq \frac{3d}{\eta_c^{1/4}\gamma_o}
\end{equation}
Hence the order of $T$ is given by $T=O\left(\frac{\log d}{L_{max}^{1/4}\gamma_o}\right)$. We hide the dependence on $d$ when we use $T=\tilde{O}\left(L_{max}^{-1/4}\right)$. Using \eqref{theorem2_4} it can be proved that,
\begin{equation*}
    \frac{1}{2}\sum_{i=1}^d \lambda_i \sum_{\tau=0}^{T-1}(1-\eta_c\lambda_i)^{2(T-\tau-1)} \mathbb{E}[\abs{\tilde{\textbf{w}}_{\tau,i}}^2]\leq -\tilde{O}(L_{max}^{3/4}).
\end{equation*}
% The proof of \eqref{finalb} is provided in Appendix \ref{ap:descent}. 
% This completes the derivation of the upper bound of $K_1$.
\subsection{Bounding $K_2$ and $K_3$}
We define the event $C_T$ as, $C_T=\left\{\forall t\leq T, \norm{\tilde{\boldsymbol{\kappa}}}\leq \tilde{O}\left(L_{max}^{3/8}\log\frac{1}{L_{max}}\right),\norm{\boldsymbol{\kappa}}\leq \tilde{O}(L_{max}^{3/4})\right\}$. From Lemma \ref{lemma:xt_x0} and Lemma \ref{lemma:xt_tildext} in Appendix \ref{ap:lemma1} and \ref{ap:lemma2} respectively, we know that with probability $\mathbb{P}(C_T)\geq 1-\tilde{O}\left(L_{max}^{7/2}\right)$, the term $\norm{\tilde{\boldsymbol{\kappa}}}$ can be bounded by $ \tilde{O}\left(L_{max}^{3/8}\log\frac{1}{L_{max}}\right) $ and $\norm{\boldsymbol{\kappa}}$ can be bounded by $\tilde{O}(L_{max}^{3/4})$, $\forall t\leq T=O\left(L_{max}^{-1/4}\right)$.

Now, to complete the proof of Theorem 2, we need to show that the term $K_1$ dominates both $K_2$ and $K_3$. Hence, we obtain the bound for the term $K_2$ as,

    \begin{align*}
\mathbb{E}[\zeta \mathbf{1}_{C_T}]&=\mathbb{E}\left[\nabla f(\textbf{x}_0)^T\boldsymbol{\kappa} +\tilde{\boldsymbol{\kappa}}^T\textbf{H}(\textbf{x}_0)\boldsymbol{\kappa} +\frac{1}{2}\boldsymbol{\kappa}^T\textbf{H}(\textbf{x}_0)\boldsymbol{\kappa} +\frac{\rho}{6}\norm{\tilde{\boldsymbol{\kappa}}+\boldsymbol{\kappa}}^3\right] \mathbb{P}(C_T)\\
&\leq \tilde{O}\left(L_{max}^{3/8}\log\frac{1}{L_{max}}\right)\tilde{O}(L_{max}^{3/4}) \mathbb{P}(C_T)=\tilde{O}\left(L_{max}^{9/8}\log\frac{1}{L_{max}}\right) \mathbb{P}(C_T).
    \end{align*}

% The explanation for the bound in \eqref{k22bound} is given in Appendix \ref{ap:k2bound}.

Finally, we bound the term $K_3$ as follows.
\begin{equation*}
% \begin{aligned}
\mathbb{E}[(f(\textbf{x}_T)-f(\textbf{x}_0))\mathbf{1}_{\Bar{C}_T}]-\mathbb{E}[\tilde{\zeta}\mathbf{1}_{\Bar{C}_T}]\leq \tilde{O}(1)\mathbb{P}(\Bar{C}_T) \leq \tilde{O}\left(L_{max}^{7/2}\right), 
% \end{aligned}
\end{equation*}
where the inequality arises from the boundedness of the function. Comparing the bounds of the terms $K_1,$ $K_2$, and $K_3$, we find that $K_1$ dominates, which completes the proof.
\end{proof}
\section{Proof of Theorem \ref{th:3}}\label{ap:th3_proof}
\begin{theorem}\label{th:3restated}(Theorem \ref{th:3} restated)
     Consider $f$ satisfying the assumptions \hyperlink{hyper:a1}{\textbf{A1}}-\hyperlink{hyper:a6}{\textbf{A6}}. Let the initial iterate $\textbf{x}_0$ be $\delta$ close to a local minimum $\textbf{x}^{*}$ such that $\norm{\textbf{x}_0-\textbf{x}^{*}}\leq\tilde{O}(\sqrt{L_{max}})<\delta$. With probability at least $1-\xi$,  $\forall t\leq T$ where $T= \tilde{O}\left(\frac{1}{L_{max}^2}\log\frac{1}{\xi}\right)$,
    \begin{equation*}
        \norm{\textbf{x}_t-\textbf{x}^{*}}\leq \tilde{O}\left(\sqrt{L_{max}\log\frac{1}{L_{max}\xi}}\right)<\delta
    \end{equation*}
\end{theorem}
\begin{proof}
    This theorem handles the case when the iterate is close to the local minimum (case \hyperlink{hyper:b3}{\textbf{B3}}). We aim to show that the iterate does not leave the neighbourhood of the minimum for $t\leq \tilde{O}\left(\frac{1}{L_{max}^2}\log\frac{1}{\xi}\right)$. By assumption \hyperlink{hyper:a6}{\textbf{A6}}, if $\textbf{x}_t$ is $\delta$ close to the local minimum $\textbf{x}^{*}$, the function is locally $\alpha$- strongly convex. We define event $D_t=\{\forall \tau\leq t,\norm{\textbf{x}_{\tau}-\textbf{x}^{*}}\leq \mu\sqrt{L_{max}\log\frac{1}{L_{max}\xi}}<\delta\}$. Let $L_{max}<\frac{r}{\log\xi^{-1}}$ where $r<\log\xi^{-1}$. It can be seen that $D_{t-1}\subset D_t$. Conditioned on event $D_t$, and using $\alpha-$strong convexity of $f$,
        $(\nabla f(\textbf{x}_t)-\nabla f(\textbf{x}^{*}))^T(\textbf{x}_t-\textbf{x}^{*})\mathbf{1}_{D_t}\geq \alpha\norm{\textbf{x}_t-\textbf{x}^{*}}^2\mathbf{1}_{D_t}.$
    As $\nabla f(\textbf{x}^{*})=0$, it becomes, $
        \nabla f(\textbf{x}_t)^T(\textbf{x}_t-\textbf{x}^{*})\mathbf{1}_{D_t}\geq \alpha\norm{\textbf{x}_t-\textbf{x}^{*}}^2\mathbf{1}_{D_t}.$
        We define a filtration $S_t =s\{\textbf{w}_0,\hdots,\textbf{w}_{t-1}\}$ in order to construct a supermartingale and use the Azuma-Hoeffding inequality where $s\{.\}$ denotes a sigma-algebra field. Now, assuming $L_{max}<\frac{\alpha}{\beta^2}$,
% \vspace{-1mm}
\begin{subequations}
    \begin{align}
        &\mathbb{E}[\norm{\textbf{x}_t-\textbf{x}^{*}}^2\mathbf{1}_{D_{t-1}} | S_{t-1}]=\mathbb{E}[\norm{\textbf{x}_{t-1}-\eta_c\nabla f(\textbf{x}_{t-1})-\textbf{w}_{t-1}-\textbf{x}^{*}}^2|S_{t-1}]\mathbf{1}_{D_{t-1}}\nonumber\\
        &=\mathbb{E}[\norm{(\textbf{x}_{t-1}-\textbf{x}^{*})-\eta_c\nabla f(\textbf{x}_{t-1})-\textbf{w}_{t-1}}^2|S_{t-1}]\mathbf{1}_{D_{t-1}}\nonumber\\
        % &=\mathbb{E}[((\textbf{x}_{t-1}-\textbf{x}^{*})-\eta_c\nabla f(\textbf{x}_{t-1})-\tilde{\textbf{w}}_{t-1})^T\nonumber\\
        % &((\textbf{x}_{t-1}-\textbf{x}^{*})-\eta_c\nabla f(\textbf{x}_{t-1})-\tilde{\textbf{w}}_{t-1})|S_{t-1}]\mathbf{1}_{D_{t-1}}\\
        &=[\norm{\textbf{x}_{t-1}-\textbf{x}^{*}}^2 -2\eta_c(\textbf{x}_{t-1}-\textbf{x}^{*})^T\nabla f(\textbf{x}_{t-1})+\eta_c^2\norm{\nabla f(\textbf{x}_{t-1})}^2+\mathbb{E}[\norm{\textbf{w}_{t-1}}^2]]\mathbf{1}_{D_{t-1}}\label{th3_zeromean}\\
        &\leq [\norm{\textbf{x}_{t-1}-\textbf{x}^{*}}^2 -2\eta_c\alpha\norm{\textbf{x}_{t-1}-\textbf{x}^{*}}^2+\eta_c^2\beta^2\norm{\textbf{x}_{t-1}-\textbf{x}^{*}}^2+\mathbb{E}[\norm{\textbf{w}_{t-1}}^2]]\mathbf{1}_{D_{t-1}}\label{th3_1}
    \end{align}
\end{subequations}
We use $\mathbb{E}[\textbf{w}_t] = 0$ in \eqref{th3_zeromean}. We use the $\beta$-smoothness and $\alpha-$convexity assumptions of $f$ in \eqref{th3_1}. 
Now, using $\textbf{w}_{t-1}=\eta_cg(\textbf{x}_{t-1})-\eta_c\nabla f(\textbf{x}_{t-1})+u_tg(\textbf{x}_{t-1})$, we compute $\mathbb{E}[\norm{\textbf{w}_{t-1}}^2]$ as,
\begin{equation}
    \begin{aligned}
        &\mathbb{E}[\norm{\textbf{w}_{t-1}}^2]\\      &=\mathbb{E}\Big[\eta_c^2\norm{g(\textbf{x}_{t-1})-\nabla f(\textbf{x}_{t-1})}^2+2\eta_cu_t\Big(g(\textbf{x}_{t-1})-\nabla f(\textbf{x}_{t-1})\Big)^Tg(\textbf{x}_{t-1})+u_t^2\norm{g(\textbf{x}_{t-1})}^2\Big]\\
        &\leq \eta_c^2\sigma^2+\mathbb{E}[u_t^2]\mathbb{E}[\norm{g(\textbf{x}_{t-1})}^2]\leq \eta_c^2\sigma^2+\mathbb{E}[u_t^2](\sigma^2+\norm{\nabla f(\textbf{x}_{t-1})}^2)\\
        &\leq \eta_c^2\sigma^2+\mathbb{E}[u_t^2]\sigma^2+\mathbb{E}[u_t^2]\beta^2 \norm{\textbf{x}_{t-1}-\textbf{x}^{*}}^2\\
        &\leq\sigma^2\left(\eta_c^2+\frac{2L_{max}^2}{3}-\frac{2L_{max}\eta_c}{3}\right)+\beta^2\norm{\textbf{x}_{t-1}-\textbf{x}^{*}}^2\left(\frac{2L_{max}^2}{3}-\frac{2L_{max}\eta_c}{3}\right) \label{th3_2}.
    \end{aligned}
\end{equation}
As $\eta_c=\frac{L_{min}+L_{max}}{2}$, $L_{min}=2\eta_c-L_{max}$. Hence, we write $\mathbb{E}[u_t^2]=\frac{(L_{max}-L_{min})^2}{12}=\frac{4(L_{max}-\eta_c)^2}{12}=\frac{L^2_{max}+\eta_c^2-2L_{max}\eta_c}{3}<\frac{2L_{max}^2}{3}-\frac{2L_{max}\eta_c}{3}$ in \eqref{th3_2}. Using \eqref{th3_2} in \eqref{th3_1},
\begin{equation*}
    \begin{aligned}
       \mathbb{E}[\norm{\textbf{x}_t-\textbf{x}^{*}}^2\mathbf{1}_{D_{t-1}} | S_{t-1}]&\leq  \left[\norm{\textbf{x}_{t-1}-\textbf{x}^{*}}^2\left(1-2\eta_c\alpha+\eta_c^2\beta^2+\frac{2L_{max}^2\beta^2}{3}-\frac{2L_{max}\eta_c\beta^2}{3}\right)\right.\nonumber\\
       &\left.+\sigma^2\left(\eta_c^2+\frac{2L_{max}^2}{3}-\frac{2L_{max}\eta_c}{3}\right)    \right]\mathbf{1}_{D_{t-1}}\\
      &\hspace{-15mm}\leq\left[\norm{\textbf{x}_{t-1}-\textbf{x}^{*}}^2\left(1+\eta_c\alpha+\frac{2L_{max}\alpha}{3}\right)+\sigma^2\left(L_{max}^2+\frac{2L_{max}^2}{3}\right)    \right]\mathbf{1}_{D_{t-1}}\\
       &\hspace{-15mm}\leq \left[\norm{\textbf{x}_{t-1}-\textbf{x}^{*}}^2\left(1+L_{max}\alpha+\frac{2L_{max}\alpha}{3}\right)+\sigma^2\left(L_{max}^2+\frac{2L_{max}^2}{3}\right)    \right]\mathbf{1}_{D_{t-1}}\\      &\hspace{-15mm}=\left[\norm{\textbf{x}_{t-1}-\textbf{x}^{*}}^2\left(1+\frac{5L_{max}\alpha}{3}\right)+\frac{5L_{max}^2\sigma^2}{3}    \right]\mathbf{1}_{D_{t-1}}.
    \end{aligned}
\end{equation*}
We use $L_{max}<\frac{\alpha}{\beta^2}$. Let $J_t=\left(1+\frac{5\alpha L_{max}}{3}\right)^{-t}\left(\norm{\textbf{x}_{t}-\textbf{x}^{*}}^2+\frac{L_{max}\sigma^2}{\alpha}\right)$. We prove $J_t\mathbf{1}_{D_{t-1}}$ is a supermartingale process as follows.\\
\begin{equation*}
    \begin{aligned}
        &\mathbb{E}\left[\left(1+\frac{5\alpha L_{max}}{3}\right)^{-t}\right.\left.\left(\norm{\textbf{x}_{t}-\textbf{x}^{*}}^2+\frac{L_{max}\sigma^2}{\alpha}\right)\bigg|S_{t-1}\right]\mathbf{1}_{D_{t-1}}\leq \nonumber\\
        &\left(1+\frac{5\alpha L_{max}}{3}\right)^{-t}\left[\norm{\textbf{x}_{t-1}-\textbf{x}^{*}}^2\left(1+\frac{5L_{max}\alpha}{3}\right)+\frac{5L_{max}^2\sigma^2}{3}   +\frac{L_{max}\sigma^2}{\alpha} \right]\mathbf{1}_{D_{t-1}}\\
        &=\left(1+\frac{5\alpha L_{max}}{3}\right)^{-(t-1)}\left[\norm{\textbf{x}_{t-1}-\textbf{x}^{*}}^2+\frac{L_{max}\sigma^2}{\alpha}\right]\mathbf{1}_{D_{t-1}}=J_{t-1}\mathbf{1}_{D_{t-1}}\leq J_{t-1}\mathbf{1}_{D_{t-2}}.
    \end{aligned}
\end{equation*}
Hence $J_t\mathbf{1}_{D_{t-1}}$ is a supermartingale. In order to use the Azuma-Hoeffding inequality, we bound $|J_t\mathbf{1}_{D_{t-1}}-\mathbb{E}[J_t\mathbf{1}_{D_{t-1}}|S_{t-1}]|$ as,\\
\begin{equation}
    \begin{aligned}
        &|J_t\mathbf{1}_{D_{t-1}}-\mathbb{E}[J_t\mathbf{1}_{D_{t-1}}|S_{t-1}]|= \left(1+\frac{5\alpha L_{max}}{3}\right)^{-t}\left[  \norm{\textbf{x}_{t}-\textbf{x}^{*}}^2-\mathbb{E}[\norm{\textbf{x}_{t}-\textbf{x}^{*}}^2|S_{t-1}]\right]\mathbf{1}_{D_{t-1}}\\
        &\leq \left(1+\frac{5\alpha L_{max}}{3}\right)^{-t} \bigg[ 2\norm{\textbf{x}_{t-1}-\eta_c\nabla f(\textbf{x}_{t-1})-\textbf{x}^{*}}\norm{\textbf{w}_{t-1}} +\norm{\textbf{w}_{t-1}}^2+ \\
        &   \sigma^2\left(\eta_c^2+\frac{2L_{max}^2}{3}-\frac{2L_{max}\eta_c}{3}\right)+\beta^2\norm{\textbf{x}_{t-1}-\textbf{x}^{*}}^2 \left(\frac{2L_{max}^2}{3}-\frac{2L_{max}\eta_c}{3}\right)  \bigg]\mathbf{1}_{D_{t-1}}\label{th3_4},
    \end{aligned}
\end{equation}
where we use \eqref{th3_2} in \eqref{th3_4} for the term $\mathbb{E}[\norm{\textbf{w}_{t-1}}^2]$. Now, we compute $\norm{\textbf{w}_{t-1}}$ using assumption \hyperlink{hyper:a3}{\textbf{A3}} as follows.
\begin{equation}
    \begin{aligned}
        \norm{\textbf{w}_{t-1}}&=\norm{\eta_cg(\textbf{x}_{t-1})-\eta_c\nabla f(\textbf{x}_{t-1})+u_tg(\textbf{x}_{t-1})}\\
        % &\leq \eta_cQ +|u_t|\norm{g(\textbf{x}_{t-1})} \\
        &\leq \eta_cQ +|u_t|(Q+\norm{\nabla f(\textbf{x}_{t-1})})
        % &=Q(\eta_c+|u_t|)+|u_t|\norm{\nabla f(\textbf{x}_{t-1})}\\
        \leq Q(\eta_c+|u_t|)+|u_t|\beta\norm{\textbf{x}_{t-1}-\textbf{x}^{*}} \label{th3_5}.
    \end{aligned}
\end{equation}
Using \eqref{th3_5} in \eqref{th3_4} and the bound of the event $D_{t-1}$,
\begin{equation*}
    \begin{aligned}
        &|J_t\mathbf{1}_{D_{t-1}}-\mathbb{E}[J_t\mathbf{1}_{D_{t-1}}|S_{t-1}]| \nonumber\\
        &\leq \left(1+\frac{5\alpha L_{max}}{3}\right)^{-t}\bigg[ 2\norm{\textbf{x}_{t-1}-\textbf{x}^{*}} \left(  Q(\eta_c+|u_t|)+|u_t|\beta\norm{\textbf{x}_{t-1}-\textbf{x}^{*}}\right)\nonumber\\
        &+\left(Q(\eta_c+|u_t|)+|u_t|\beta\norm{\textbf{x}_{t-1}-\textbf{x}^{*}}\right)^2+\sigma^2\left(\eta_c^2+\frac{2L_{max}^2}{3}-\frac{2L_{max}\eta_c}{3}\right)\nonumber\\
        &+\beta^2\norm{\textbf{x}_{t-1}-\textbf{x}^{*}}^2\left(\frac{2L_{max}^2}{3}-\frac{2L_{max}\eta_c}{3}\right)\bigg]\mathbf{1}_{D_{t-1}}\\
        &=\left(1+\frac{5\alpha L_{max}}{3}\right)^{-t}\bigg[\tilde{O}\left(\mu L_{max}^{1.5}\log^{0.5}\frac{1}{L_{max}\xi}\right)+\tilde{O}\left(\mu^2L_{max}^2\log\frac{1}{L_{max}\xi}\right)+2\tilde{O}(L_{max}^2)\nonumber\\
        &+\tilde{O}\left( \mu L_{max}^{2.5}\log^{0.5}\frac{1}{L_{max}\xi} \right)+2\tilde{O}\left( \mu^2L_{max}^3\log\frac{1}{L_{max}\xi} \right)\bigg]\\
        &\leq \left(1+\frac{5\alpha L_{max}}{3}\right)^{-t}\tilde{O}\left(\mu L_{max}^{1.5}\log^{0.5}\frac{1}{L_{max}\xi}\right)=d_t
    \end{aligned}
\end{equation*}
We denote the bound of $|J_t\mathbf{1}_{D_{t-1}}-\mathbb{E}[J_t\mathbf{1}_{D_{t-1}}|S_{t-1}]|$ as $d_t$.\\
Let $b_t=\sqrt{\sum_{\tau=1}^t d_{\tau}^2}=\sqrt{\sum_{\tau=1}^t  \left(1+\frac{5\alpha L_{max}}{3}\right)^{-2\tau}}\tilde{O}\left(\mu L_{max}^{1.5}\log^{0.5}\frac{1}{L_{max}\xi}\right)$. Now,
\begin{equation*}
    \begin{aligned}
        \sqrt{\sum_{\tau=1}^t  \left(1+\frac{5\alpha L_{max}}{3}\right)^{-2\tau}}&\tilde{O}\left(\mu L_{max}^{1.5}\log^{0.5}\frac{1}{L_{max}\xi}\right)\\
        &\hspace{-35mm}\leq \sqrt{\frac{1}{1-\left(1+\frac{5\alpha L_{max}}{3}\right)^{-2}}}\tilde{O}\left(\mu L_{max}^{1.5}\log^{0.5}\frac{1}{L_{max}\xi}\right)\\
        % &=\sqrt{\frac{\left(1+\frac{5\alpha L_{max}}{3}\right)^2}{\left(1+\frac{5\alpha L_{max}}{3}\right)^2-1}}\tilde{O}\left(\mu L_{max}^{1.5}\log^{0.5}\frac{1}{L_{max}\zeta}\right)\\
        &\hspace{-35mm}=\sqrt{\frac{\tilde{O}(1)}{\tilde{O}(L_{max})}}\tilde{O}\left(\mu L_{max}^{1.5}\log^{0.5}\frac{1}{L_{max}\xi}\right)=\tilde{O}\left(\mu L_{max}\log^{0.5}\frac{1}{L_{max}\xi}\right).
    \end{aligned}
\end{equation*}
Hence $b_t$ is of the order $\tilde{O}\left(\mu L_{max}\log^{0.5}\frac{1}{L_{max}\xi}\right).$
By the Azuma Hoeffding inequality,
\begin{equation*}
\begin{aligned}
   \mathbb{P}\left( J_t\mathbf{1}_{D_{t-1}} -J_0 \geq b_t\log^{0.5}\frac{1}{L_{max}\xi}\right)\leq  \exp\left(-\tilde{\Omega}\left(\log\frac{1}{L_{max}\xi}\right)\right)\leq \tilde{O}(L_{max}^3\xi)  , 
\end{aligned}
\end{equation*}
which leads to,
\begin{equation*}
    \mathbb{P}\left( J_t\mathbf{1}_{D_{t-1}} -J_0 \geq \tilde{O}\left(\mu L_{max}\log\frac{1}{L_{max}\xi}\right)\right)\leq \tilde{O}(L_{max}^3\xi).
\end{equation*}
% ----------------------------------------------------------------------------------------------
% \comm{intermediary steps required here?} NO.
Hence we can write,
\begin{equation*}
    \mathbb{P}\left(D_{t-1}\cap \left\{\norm{\textbf{x}_t-\textbf{x}^{*}}^2\geq \tilde{O}\left(\mu L_{max}\log\frac{1}{L_{max}\xi}\right)\right\}\right)\leq \tilde{O}(L_{max}^3\xi)
\end{equation*}
For some constant $\tilde{b}$ independent of $L_{max}$ and $\xi$ we can write,
\begin{equation*}
    \mathbb{P}\left(D_{t-1}\cap \left\{\norm{\textbf{x}_t-\textbf{x}^{*}}^2\geq \tilde{b}\mu L_{max}\log\frac{1}{L_{max}\xi}\right\}\right)\leq \tilde{O}(L_{max}^3\xi)
\end{equation*}
By choosing $\mu<\tilde{b}$,
\begin{equation*}
    % \mathbb{P}\left(D_{t-1}\cap \left\{\norm{\textbf{x}_t-\textbf{x}^{*}}^2\geq \mu^2 L_{max}\log\frac{1}{L_{max}\zeta}\right\}\right)\leq \tilde{O}(L_{max}^3\zeta)\\
    \mathbb{P}\left(D_{t-1}\cap \left\{\norm{\textbf{x}_t-\textbf{x}^{*}}\geq \mu\sqrt{ L_{max}\log\frac{1}{L_{max}\xi}}\right\}\right)\leq \tilde{O}(L_{max}^3\xi)
\end{equation*}
% Now,
\begin{equation*}
    \begin{aligned}
        \mathbb{P}(\Bar{D_t})&=\mathbb{P}\left(D_{t-1}\cap \left\{\norm{\textbf{x}_t-\textbf{x}^{*}}\geq \mu\sqrt{ L_{max}\log\frac{1}{L_{max}\xi}}\right\}\right)+\mathbb{P}(\Bar{D}_{t-1})\\
        &\leq \tilde{O}(L_{max}^3\xi) +\mathbb{P}(\Bar{D}_{t-1})
    \end{aligned}
\end{equation*}
Iteratively unrolling the above equation, we obtain $\mathbb{P}(\Bar{D_t})\leq t \tilde{O}(L_{max}^3\xi)$. Choosing $t=\tilde{O}\left(\frac{1}{L_{max}^2}\log\frac{1}{\xi}\right)$, $\mathbb{P}(\Bar{D_t})\leq \tilde{O}\left(L_{max}\xi\log\frac{1}{\xi}\right)$. As $L_{max}<\tilde{O}\left(\frac{1}{\log\frac{1}{\xi}}\right)$, $\mathbb{P}(\Bar{D}_t)\leq \tilde{O}(\xi).$ 
\end{proof}
\section{Proof using induction}\label{ap:induction}
In the proof of Lemma \ref{lemma:xt_x0} in Appendix \ref{ap:lemma1}, we state that \eqref{generalt} can be proved by induction for $t\geq 2$. We restate the equation here and provide the corresponding proof by induction.
% \vspace{-.3cm}
\begin{equation}\label{generalt_restated}
        \text{Induction hypothesis:} \quad \norm{\nabla \tilde{f}(\tilde{\textbf{x}}_t)}\leq 10\tilde{Q} \sum_{\tau=0}^{\frac{t(t-1)}{2}} (1+\eta_c\gamma_o)^{\tau}.
\end{equation}
% We now proceed to prove this by induction $\forall t\geq 2$. 
Recollect from that \eqref{tildefeq} that
$
    \nabla\tilde{f}(\tilde{\textbf{x}}_t)  
    =(I-\eta_c\textbf{H}(\textbf{x}_0))\nabla\tilde{f}(\tilde{\textbf{x}}_{t-1}) -\textbf{H}(\textbf{x}_0)\tilde{\textbf{w}}_{t-1}.$ 
Taking matrix induced norm on both sides,
\allowdisplaybreaks
\begin{equation}
    \begin{aligned}\label{substildef}
        \norm{\nabla \tilde{f}(\tilde{\textbf{x}}_{t+1})}&\leq (1+\eta_c\gamma_o)\norm{\nabla \tilde{f}(\tilde{\textbf{x}}_{t})}+\beta \norm{\tilde{\textbf{w}}_t} \\
        % &= (1+\eta_c\gamma_o)\norm{\nabla \tilde{f}(\tilde{x}_{t})}+\beta \norm{\eta_c \tilde{g}(\tilde{x}_{t}) - \eta_c\nabla \tilde{f}(\tilde{x}_{t}) +  u_{t+1} \tilde{g}(\tilde{x}_{t})}\\
        % &\leq  (1+\eta_c\gamma_o)\norm{\nabla \tilde{f}(\tilde{x}_{t})}+\beta (\eta_c \tilde{Q}+  \norm{u_{t+1} \tilde{g}(\tilde{x}_{t})})\\
        % &= (1+\eta_c\gamma_o)\norm{\nabla \tilde{f}(\tilde{x}_{t})}+\beta \left(\eta_c \tilde{Q}+  \abs{u_{t+1}}\norm{ \tilde{g}(\tilde{x}_{t})-\nabla \tilde{f}(\tilde{x}_{t})+\nabla \tilde{f}(\tilde{x}_{t})}\right)\\
        % & \leq (1+\eta_c\gamma_o)\norm{\nabla \tilde{f}(\tilde{x}_{t})}+\beta \left(\eta_c \tilde{Q}+  \abs{u_{t+1}}\left(\tilde{Q}+\norm{\nabla \tilde{f}(\tilde{x}_{t})}\right)\right)\\
        % &=(1+\eta_c\gamma_o)\norm{\nabla \tilde{f}(\tilde{x}_{t})}+\beta\left(\tilde{Q}(\eta_c+\abs{u_{t+1}})+\abs{u_{t+1}}\norm{\nabla \tilde{f}(\tilde{x}_{t})}\right) \\
        % &=(1+\eta_c\gamma_o)\norm{\nabla \tilde{f}(\tilde{x}_{t})}+ \beta \tilde{Q}(\eta_c+\abs{u_{t+1}}) + \beta\abs{u_{t+1}}\norm{\nabla \tilde{f}(\tilde{x}_{t})} \\
        &= \left((1+\eta_c\gamma_o) + \beta\abs{u_{t+1}}\right)\norm{\nabla \tilde{f}(\tilde{x}_{t})} + \beta \tilde{Q}(\eta_c+\abs{u_{t+1}}),
    \end{aligned}
\end{equation}
since, $\norm{\tilde{g}(\tilde{\textbf{x}}_t)-\nabla \tilde{f}(\tilde{\textbf{x}}_t)}\leq \tilde{Q}$. 
Note that $\norm{\nabla \tilde{f}(\tilde{\textbf{x}}_t)}\leq \epsilon$, $\abs{u_t}\leq L_{max}$ and $\beta L_{max}<1$ hold for all $t$. Therefore, at $t=1$,
\begin{equation*}
    \begin{aligned}
        \norm{\nabla \tilde{f}(\tilde{\textbf{x}}_1)}        \leq \left((1+\eta_c\gamma_o) + \beta\abs{u_{1}}\right)\epsilon + \beta \tilde{Q}(\eta_c+\abs{u_{1}}) 
        % &\leq  \left((1+\eta_c\gamma_o) + \beta L_{max}\right)\epsilon + \beta \tilde{Q}(\eta_c+L_{max})\label{t1eq4} \\
        \leq (1+\eta_c\gamma_o)\epsilon+\epsilon+2\tilde{Q}.
    \end{aligned}
\end{equation*}
% where we use $\beta L_{max}<1$ in \eqref{t1eq5}.
Now, we prove the hypothesis in \eqref{generalt_restated} for $t=2$. From \eqref{substildef}, for an arbitrarily small $\epsilon$,
% \textcolor{red}{write the proof of the above expression as a lemma}
\begin{equation*}
    \begin{aligned}
        \norm{\nabla \tilde{f}(\tilde{\textbf{x}}_{2})}  &\leq \left((1+\eta_c\gamma_o) + \beta\abs{u_{2}}\right)\norm{\nabla \tilde{f}(\tilde{\textbf{x}}_{1})} + \beta \tilde{Q}(\eta_c+\abs{u_{2}}) \\ %&& \text{Note that }\beta |u_i| \leq 1 \nonumber \\
        % &\leq \left((1+\eta_c\gamma_o) + 1\right)((1+\eta_c \gamma_o)\epsilon + \epsilon  +2\tilde{Q}) + 2 \tilde{Q}\\
        &\leq (1+\eta_c\gamma_o)^2\epsilon+2(1+\eta_c\gamma_o)\epsilon+\epsilon+2\tilde{Q}(1+\eta_c\gamma_o)+4\tilde{Q} \\
        &\leq 2\epsilon\sum_{\tau=0}^2 (1+\eta_c\gamma_o)^{\tau} + 4\tilde{Q}\sum_{\tau=0}^1 (1+\eta_c\gamma_o)^{\tau} \leq 10\tilde{Q} \sum_{\tau=0}^{\frac{2(2-1)}{2}} (1+\eta_c\gamma_o)^{\tau}.
    \end{aligned}
\end{equation*} 
% assuming $\epsilon$ is small enough such that $2\epsilon\sum_{\tau=0}^2 (1+\eta_c\gamma_o)^{\tau}\leq 6\tilde{Q}\sum_{\tau=0}^{1} (1+\eta_c\gamma_o)^{\tau}$.
% \subsection{Assumption}
% For a general $t\geq 2$,
% \begin{equation}\label{generaltnew}
%     \norm{\nabla \tilde{f}(\tilde{x}_t)} \leq 10\tilde{Q} \sum_{\tau=0}^{\frac{t(t-1)}{2}} (1+\eta_c\gamma_o)^{\tau}
% \end{equation}
% \subsection{Proving for $t+1$}
We have shown that the induction hypothesis holds for $t=2$. Now, assuming that it holds for any $t$, we need to prove that it holds for $t+1$. 
% We need to prove that,
% \begin{equation}\label{toprove}
%     \norm{\nabla \tilde{f}(\tilde{x}_{t+1})}  \leq 10\tilde{Q} \sum_{\tau=0}^{\frac{t(t+1)}{2}} (1+\eta_c\gamma_o)^{\tau}.
% \end{equation}
We know from \eqref{substildef}, when the hypothesis is assumed to hold for $t$,
% \begin{equation}
%           \norm{\nabla \tilde{f}(\tilde{x}_{t+1})}  \leq \left((1+\eta_c\gamma_o) + \beta\abs{u_{t+1}}\right)\norm{\nabla \tilde{f}(\tilde{x}_{t})} + \beta \tilde{Q}(\eta_c+\abs{u_{t+1}})   
% \end{equation}
% Using \eqref{generalt_restated}, we can write
\begin{equation*}
    \begin{aligned}
        \norm{\nabla \tilde{f}(\tilde{\textbf{x}}_{t+1})}  &\leq \left((1+\eta_c\gamma_o) + \beta\abs{u_{t+1}}\right)10\tilde{Q} \sum_{\tau=0}^{\frac{t(t-1)}{2}} (1+\eta_c\gamma_o)^{\tau} + \beta \tilde{Q}(\eta_c+\abs{u_{t+1}})   \\
        % &\leq \left((1+\eta_c\gamma_o) + 1\right)10\tilde{Q} \sum_{\tau=0}^{\frac{t(t-1)}{2}} (1+\eta_c\gamma_o)^{\tau} + \beta \tilde{Q}(\eta_c+\abs{u_{t+1}})   \\
        &\leq (1+\eta_c\gamma_o)10\tilde{Q} \sum_{\tau=0}^{\frac{t(t-1)}{2}} (1+\eta_c\gamma_o)^{\tau}+ 10\tilde{Q} \sum_{\tau=0}^{\frac{t(t-1)}{2}} (1+\eta_c\gamma_o)^{\tau} + \beta \tilde{Q}(\eta_c+\abs{u_{t+1}})  \\
        % &\leq 10\tilde{Q} \sum_{\tau=1}^{\frac{t(t-1)}{2}+1} (1+\eta_c\gamma_o)^{\tau}+ 10\tilde{Q} \sum_{\tau=0}^{\frac{t(t-1)}{2}} (1+\eta_c\gamma_o)^{\tau} + 2 \tilde{Q} \\
        &\leq 20\tilde{Q} \sum_{\tau=0}^{\frac{t(t-1)}{2}+1} (1+\eta_c\gamma_o)^{\tau}
    \end{aligned}
\end{equation*}
If we prove $20\tilde{Q} \sum_{\tau=0}^{\frac{t(t-1)}{2}+1} (1+\eta_c\gamma_o)^{\tau}\leq 10\tilde{Q} \sum_{\tau=0}^{\frac{t(t+1)}{2}} (1+\eta_c\gamma_o)^{\tau}$, the induction proof is complete. Now, we need to prove
% -----------------------------------------------------------------------------------------------------------------
% \textcolor{green}{ New proof from here}
% We need to prove,
\begin{equation*}
    \begin{aligned}
        20\tilde{Q} \sum_{\tau=0}^{\frac{t^2-t}{2}+1} (1+\eta_c\gamma_o)^{\tau}&\leq 10\tilde{Q} \sum_{\tau=0}^{\frac{t^2+t}{2}} (1+\eta_c\gamma_o)^{\tau} \\
        &\leq 10\tilde{Q} \sum_{\tau=0}^{\frac{t^2-t}{2}+1} (1+\eta_c\gamma_o)^{\tau} + 10\tilde{Q} \sum_{\tau=\frac{t^2-t}{2}+2}^{\frac{t^2+t}{2}} (1+\eta_c\gamma_o)^{\tau}. 
    \end{aligned}
\end{equation*}
Therefore we need to show that, 
\begin{equation}\label{appe1}
 \underbrace{\sum_{\tau=0}^{\frac{t^2-t}{2}+1} (1+\eta_c\gamma_o)^{\tau}}_{S_1} \leq \underbrace{\sum_{\tau=\frac{t^2-t}{2}+2}^{\frac{t^2+t}{2}} (1+\eta_c\gamma_o)^{\tau}}_{S_2}.
\end{equation}
% \begin{gather}
% 20\tilde{Q} \sum_{\tau=0}^{\frac{t^2-t}{2}+1} (1+\eta_c\gamma_o)^{\tau}\leq 10\tilde{Q} \sum_{\tau=0}^{\frac{t^2+t}{2}} (1+\eta_c\gamma_o)^{\tau}\\
% \Rightarrow 20\tilde{Q} \sum_{\tau=0}^{\frac{t^2-t}{2}+1} (1+\eta_c\gamma_o)^{\tau}\leq 10\tilde{Q} \sum_{\tau=0}^{\frac{t^2-t}{2}+1} (1+\eta_c\gamma_o)^{\tau} + 10\tilde{Q} \sum_{\tau=\frac{t^2-t}{2}+2}^{\frac{t^2+t}{2}} (1+\eta_c\gamma_o)^{\tau}\\
% \Rightarrow 10\tilde{Q} \sum_{\tau=0}^{\frac{t^2-t}{2}+1} (1+\eta_c\gamma_o)^{\tau} \leq 10\tilde{Q} \sum_{\tau=\frac{t^2-t}{2}+2}^{\frac{t^2+t}{2}} (1+\eta_c\gamma_o)^{\tau}\\
% \Rightarrow  \underbrace{\sum_{\tau=0}^{\frac{t^2-t}{2}+1} (1+\eta_c\gamma_o)^{\tau}}_{S_1} \leq \underbrace{\sum_{\tau=\frac{t^2-t}{2}+2}^{\frac{t^2+t}{2}} (1+\eta_c\gamma_o)^{\tau}}_{S_2}\label{appe1}
% \end{gather}
Now, summing up the geometric series $S_1$, $\sum_{\tau=0}^{\frac{t^2-t}{2}+1} (1+\eta_c\gamma_o)^{\tau} = \frac{(1+\eta_c\gamma_o)^{\frac{t^2-t}{2}+2}-1}{\eta_c\gamma_o}.$
% \begin{equation}
%      \sum_{\tau=0}^{\frac{t^2-t}{2}+1} (1+\eta_c\gamma_o)^{\tau} = \frac{(1+\eta_c\gamma_o)^{\frac{t^2-t}{2}+2}-1}{\eta_c\gamma_o}.
% \end{equation}
Using change of variable in $S_2$ of \eqref{appe1} as  $m=\tau-\left(\frac{t^2-t}{2}+2\right)$,
\begin{equation*}
    \sum_{m=0}^{t-2} (1+\eta_c\gamma_o)^{\frac{t^2-t}{2}+m+2} = (1+\eta_c\gamma_o)^{\frac{t^2-t}{2}+2}\frac{(1+\eta_c\gamma_o)^{t-1}-1}{\eta_c\gamma_o}.
\end{equation*}
Therefore, we now need to prove,
\begin{equation}
    \begin{aligned}
        (1+\eta_c\gamma_o)^{\frac{t^2-t}{2}+2}-1&\leq (1+\eta_c\gamma_o)^{\frac{t^2-t}{2}+2}\left((1+\eta_c\gamma_o)^{t-1}-1\right)\\
    % \Rightarrow (1+\eta_c\gamma_o)^{\frac{t^2-t}{2}+2}\leq (1+\eta_c\gamma_o)^{\frac{t^2-t}{2}+2+t-1} - (1+\eta_c\gamma_o)^{\frac{t^2-t}{2}+2}+1\\
    \Rightarrow 2(1+\eta_c\gamma_o)^{\frac{t^2-t}{2}+2}&\leq (1+\eta_c\gamma_o)^{\frac{t^2-t}{2}+t+1}+1\label{appe2}
    \end{aligned}
\end{equation}
% \begin{gather}
%     (1+\eta_c\gamma_o)^{\frac{t^2-t}{2}+2}-1\leq (1+\eta_c\gamma_o)^{\frac{t^2-t}{2}+2}\left((1+\eta_c\gamma_o)^{t-1}-1\right)\\
%     % \Rightarrow (1+\eta_c\gamma_o)^{\frac{t^2-t}{2}+2}\leq (1+\eta_c\gamma_o)^{\frac{t^2-t}{2}+2+t-1} - (1+\eta_c\gamma_o)^{\frac{t^2-t}{2}+2}+1\\
%     \Rightarrow 2(1+\eta_c\gamma_o)^{\frac{t^2-t}{2}+2}\leq (1+\eta_c\gamma_o)^{\frac{t^2-t}{2}+t+1}+1\label{appe2}
% \end{gather}
We further prove \eqref{appe2} by induction as follows.
For $t=2$, $2(1+\eta_c\gamma_o)^3\leq (1+\eta_c\gamma_o)^4+1$. 
Let us assume the following expression holds for time step $t$.
\begin{equation}
    2(1+\eta_c\gamma_o)^{\frac{t^2-t}{2}+2}\leq (1+\eta_c\gamma_o)^{\frac{t^2-t}{2}+t+1}\label{appe3}
\end{equation}
Now, we prove for the time step $t+1$,
\begin{equation}
    \begin{aligned}
        2(1+\eta_c\gamma_o)^{\frac{t(t+1)}{2}+2}&= 2(1+\eta_c\gamma_o)^{\frac{t(t-1)}{2}+t+2}\leq (1+\eta_c\gamma_o)^{\frac{t^2-t}{2}+t+1+t}\label{appe5}\\
        &=(1+\eta_c\gamma_o)^{\frac{t(t+1)}{2}+t+1}\leq (1+\eta_c\gamma_o)^{\frac{t(t+1)}{2}+t+2},
    \end{aligned}
\end{equation}
where we use $\frac{t(t-1)}{2}+t=\frac{t(t+1)}{2}$ and apply our assumption \eqref{appe3} in \eqref{appe5}. We have proved $2(1+\eta_c\gamma_o)^{\frac{t^2-t}{2}+2}\leq (1+\eta_c\gamma_o)^{\frac{t^2-t}{2}+t+1}\leq (1+\eta_c\gamma_o)^{\frac{t^2-t}{2}+t+1}+1$. This concludes our proof of \eqref{generalt_restated}.
\section{Choice of parameters for other LR schedulers}\label{ap:parameters}
\begin{enumerate}
    \item Cosine annealing \citep{sgdr}: 
    % In the cosine annealing learning rate scheduler, 
    There are 3 parameters namely, initial restart interval, a multiplicative factor and minimum learning rate. 
% The initial restart interval is to determine the number of epochs after which the cosine cycle is repeated after the first cycle is completed. The multiplicative factor determines the factor by which the subsequent restart intervals increase. For example, if our initial restart interval was $1$ and we set our factor as $2$, the next restarts will take place at epochs $2,4,8,\hdots$ and so on. 
The authors propose an initial restart interval of $1$, a factor of $2$ for subsequent restarts, with a minimum learning rate of $1e-4$, which we use in our comparisons. 
\item Knee \citep{knee}: The total number of epochs is divided into those that correspond to the "explore" epochs and "exploit" epochs. During the explore epochs, the learning rate is kept at a constant high value, while from the beginning of the exploit epochs, it is linearly decayed. We use the suggested setting of $100$ initial explore epochs with a learning rate of $0.1$ followed by a linear decay for the rest of the epochs.
\item One cycle \citep{smith2019super}: We perform the learning rate range test for our networks as suggested by the authors. For the range test, the learning rate is gradually increased during which the training loss explodes. The learning rate at which it explodes is noted and the maximum learning rate (the learning rate at the middle of the triangular cycle) is fixed to be before that. We linearly increase the learning rate for the initial $45\%$ of the total epochs up to the maximum learning rate determined by the range test, followed by a linear decay for the next $45\%$ of the total epochs. We then decay it further up to a divisive factor of $10$ for the rest of the epochs, which is the suggested setting. Note that the one cycle LR scheduler relies heavily on regularization parameters like weight decay and momentum.
\item Constant: To compare with a constant learning rate, we choose $0.05$ for the VGG-16 architecture and $0.1$ for the remaining architectures as done in our other baselines\citep{cyclic,sgdr}.
\item Multi step:  For the multi-step decay scheduler, our choice of the decay rate and time is based on the standard repositories for the architectures.
\footnote{ResNet:https://github.com/akamaster/pytorch$\_$resnet$\_$cifar10,\\DenseNet:https://github.com/andreasveit/densenet-pytorch, \\VGG:https://github.com/chengyangfu/pytorch-vgg-cifar10, \\WRN:https://github.com/meliketoy/wide-resnet.pytorch}. 
Specifically, we decay the learning rate by a factor of $10$ at the the epochs $100$ and $150$ for ResNet-110 and ResNet-50. In the case of DenseNet-40-12, we decay by a factor of $10$ at the epochs $150$ and $225$. For VGG-16, we decay by a factor of $10$ every $30$ epochs. In the case of WRN, we fix a learning rate of $0.2$ for the initial $60$ epochs,  decay it by $0.2^{2}$ for the next $60$ epochs, and by $0.2^{3}$ for the rest of the epochs.
\end{enumerate}
% \section{Comparison with noisy SGD}
\section{Online tensor decomposition}\label{ap:tensor}

% \comm{While there are ample works which prove the convergence of noisy SGD, they cannot be ported into practice for deep neural networks. They require smoothness constants \citep{zhang2017hitting,jin2021,arjevani2023lower} or functional bounds on the norms of the function derivatives \citep{yiming2025efficiently} to be computed for the additive noise injection, which can not be obtained for the loss functions of neural networks or can only be approximated locally \citep{latorre2020lipschitz}. Further, the empirical convergence properties of noisy SGD are not demonstrated through examples in the majority of these analytical works which makes it hard to compare their convergence with PLRS. However, we compare our proposed PLRS against the noisy SGD mechanism proposed by Ge et al. \citep{ge2015escaping} providing convergence results on the online tensor decomposition problem using the code provided by the authors.}

% Here, we compare our work with Noisy SGD proposed in \citep{ge2015escaping}. Note that the Noisy SGD algorithm can only be applied to certain class of objective functions.
% To ensure fair experimental comparison between Noisy SGD algorithm and our proposed PLRS, we provide results on orthogonal tensor decomposition, whose objective function satisfies the conditions for the application of Noisy SGD. 
%We refrain from including the Noisy SGD algorithm in our benchmarks for neural networks as the objective function is not proven to be strict saddle. 
We follow the experimental setup in \citep{ge2015escaping}, where their proposed projected noisy gradient descent is applied to orthogonal tensor decomposition. A brief description of the online tensor decomposition problem is given below. 

Consider a tensor $T$ which has an orthogonal decomposition, 
\begin{equation}
    T = \sum_{i=1}^d a_i ^{\otimes 4} ,
\end{equation}
where $a_i$'s are orthonormal vectors. The goal of performing the tensor decomposition is to find the orthonormal components, given the tensor. The objective function is defined to reduce the correlation between the components:
\begin{equation}
    \min_{\forall i, \norm{u_i} = 1} \sum_{i \neq j} T(u_i,u_i, u_j, u_j)
\end{equation}

We plot the normalized reconstruction error, $ \norm{T -  \sum_{i=1}^d u_i ^{\otimes 4} }^2_F/\norm{T}^2_F$ in Figure \ref{fig:tensor}, where $\norm{.}_F$ denotes the Frobenius norm.
\begin{figure}
    \centering
    \includegraphics[width=0.5\linewidth]{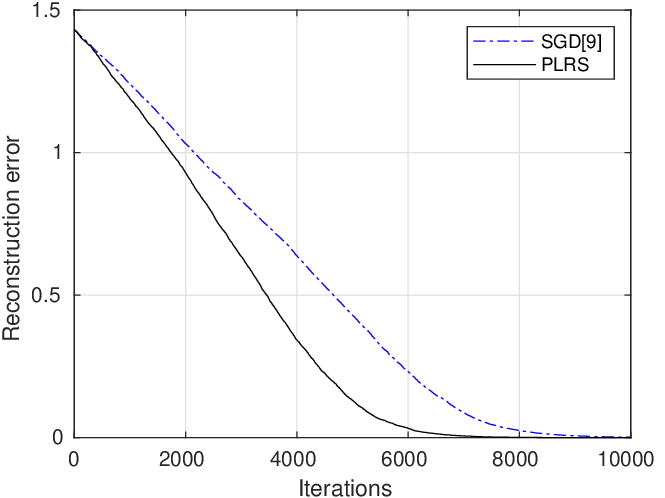}
    \caption{Reconstruction error for online tensor decomposition}
    \label{fig:tensor}
\end{figure}
We tune the learning rate parameters $L_{min}$ and $L_{max}$ to $0.007$ and $0.01$ respectively to obtain the convergence plot with PLRS. We compare against the plot in Figure 1.a of \citep{ge2015escaping}.
% , while we set a constant learning rate of $0.01$ to obtain the plot of \citep{ge2015escaping}.}
We note that the proposed Uniform LR produces faster and smoother convergence when compared to the unit sphere noise proposed in the Noisy SGD algorithm. As mentioned in \citep{ge2015escaping}, the plot may vary depending on the instance of initialization; however, it converges consistently across all runs.

Additionally, we implemented stochastic gradient descent with additive noise in the neural network setting. However, its performance was suboptimal even with extensive tuning of hyperparameters.
% \subsection{Result on WRN-28-10}\label{ap:wrn_noisysgd}
% To evaluate the convergence behavior of stochastic gradient descent with our proposed PLRS method in comparison to SGD with additive noise, as described in \citep{ge2015escaping}, we train a WRN-28-10 architecture on the CIFAR-10 dataset. For the unit-sphere noise addition procedure, samples are drawn from a normal distribution and normalized by their Euclidean norm. To address instances where the denominator becomes zero during updates, a small regularization term on the order of $10^{-9}$ is introduced. Following hyperparameter tuning, the learning rate is set to $0.5$. As illustrated in Fig.\ref{fig:wrn_rongge}, the additive noise approach exhibits poor performance, with the training loss diverging. Despite extensive tuning of the learning rate and modifications to the update mechanism, the method achieves a classification accuracy of only $39.71\%$.
% \begin{figure}
%     \centering
%     \includegraphics[width=0.5\linewidth]{Figures/wrn_rongge_combined.eps}
%     \caption{Training loss vs epochs for WRN-28-10 with CIFAR-10}
%     \label{fig:wrn_rongge}
% \end{figure}

\end{document}